\documentclass[pdflatex,sn-aps,12pt]{sn-jnl}

\usepackage{geometry}
\usepackage{amsfonts}
\usepackage[caption=false]{subfig}
\usepackage{supertabular,longtable}
\usepackage{multicol}
\usepackage{tabularx}
\usepackage{colortbl}
\usepackage{ulem}
\usepackage{graphicx}
\usepackage{booktabs}
\usepackage{array}
\usepackage[figuresright]{rotating}
\usepackage{float}
\usepackage{delarray,easybmat,eqnarray,exscale,mathenv}
\usepackage{mathptmx,yhmath,youngtab}
\usepackage{marvosym}
\usepackage{multirow}
\usepackage{bm}
\usepackage{algorithm}
\usepackage{algorithmicx}
\usepackage{amssymb}
\usepackage{mdframed}
\usepackage[font=footnotesize,labelfont=bf]{caption}
\usepackage{fancyhdr}
\usepackage{setspace}

\begin{document}
\title[PINN Swarm]{Physics-Informed Modeling and Control of Emergent Behaviors in Robot Swarms}


\author[1]{\fnm{Zixuan} \sur{Jin}}

\author[1]{\fnm{Wenzhuo} \sur{Zhang}}

\author[2]{\fnm{Shuxian} \sur{Quan}}

\author[1]{\fnm{Zirui} \sur{Dong}}

\author[1]{\fnm{Fangwen} \sur{Ye}}

\author[1]{\fnm{Yuchen} \sur{Shi}}

\author*[1,2]{\fnm{Cheng} \sur{Xu}\textsuperscript{\Letter}}\email{xucheng@ustb.edu.cn}

\affil*[1]{\orgdiv{School of Computer and Communication Engineering}, \orgname{University of Science and Technology Beijing}, \postcode{100083}, \state{Beijing}, \country{China}}

\affil*[2]{\orgdiv{Shunde Innovation School}, \orgname{University of Science and Technology Beijing}, \postcode{528399}, \state{Foshan Guangdong}, \country{China}}


\abstract{
Robot swarms can exhibit coherent collective behaviors through local perception, limited communication and decentralized decision-making, yet modeling and controlling such emergence remains challenging when behaviors unfold over multiple phases. 
Here we introduce PhySwarm, a physics-informed micro--macro framework that represents multi-stage swarm emergence as physically constrained density-field evolution coupled to executable robot motion. 
At the macroscopic level, a multi-phase advection--diffusion--reaction model (Macro-ADR) describes phase-dependent swarm-density evolution through directed transport, diffusion-based spatial regulation and behavioral phase transitions. 
At the microscopic level, an equivalent deterministic motion model (Micro-EDM) realizes these mechanisms through potential-field advection, density-gradient compensation and rate- or event-gated phase switching. 
A neural-physics controller (NPC) maps local observations and temporal memory to bounded physical parameters, and is trained with a reinforcement learning--PINN objective that combines task rewards with macro-scale density residuals and micro-scale motion-consistency constraints. 
In several proof-of-concept swarm missions---including trail-guided foraging, formation-reconfigurable navigation and role-adaptive search and rescue---we demonstrate that PhySwarm can generate distinct multi-stage emergent behaviors within a unified physics-informed modeling framework.
The learned density fields and physical parameters provide interpretable evidence of how advection, diffusion and reaction jointly regulate multi-stage swarm organization. These results establish a physics-informed route for learning, interpreting and controlling emergent behaviors in robot swarms.
}

\maketitle

\baselineskip14pt
\addtocontents{toc}{\protect\setcounter{tocdepth}{-1}}
\section{Introduction}
\label{sec:introduction}

Emergent behavior~\cite{oliveri2021continuous,benzion2023morphological,couzin2005effective} is one of the most characteristic and scientifically important collective phenomena in robot swarms. In such systems, large numbers of autonomous agents interact through local perception, limited communication and simple decision rules, yet can give rise to coherent macroscopic behaviors without centralized control~\cite{benzion2023morphological}. This property underpins the robustness, scalability, flexibility and fault tolerance of swarm robotic systems, making them attractive for environmental monitoring~\cite{Woxbots2026}, disaster response~\cite{dorigo2021swarm}, extraterrestrial exploration~\cite{exploration2025}, distributed transport~\cite{rubenstein2014programmable} and large-scale area coverage~\cite{werfel2014designing}. From an engineering perspective, this micro--macro organization enables swarms to maintain collective functionality under failures, communication constraints and environmental uncertainty. From a scientific perspective, it motivates a central modelling question: how do local perception, interaction and decentralized decision-making give rise to stable macroscopic structures, and how can these structures be continuously regulated across changing behavioral phases~\cite{vicsek1995novel,couzin2005effective,vicsek2012collective}?

Understanding and regulating this micro--macro transition is a central challenge in swarm intelligence~\cite{martinoli2004modeling, Quan2026Macroscopic, berman2009optimized, elamvazhuthi2019bilinear}. In realistic tasks, emergent behavior is rarely a single static pattern. It often appears as a sequence of functional behavioral phases that evolve continuously with environmental constraints and task demands~\cite{debie2023multask}. For example, swarm foraging~\cite{payton2001pheromone, bayindir2016review} may involve dispersed exploration, resource discovery, local aggregation, homing, trail formation and renewed exploration. Swarm navigation~\cite{olfati2006flocking, khatib1986realtime, rimon1992navigation} may require a swarm to maintain a formation in open space, compress into a line-like structure in a narrow corridor and recover its original formation after passing through the constraint. Search-and-rescue tasks~\cite{bayindir2016review, dorigo2021swarm} further involve distributed search, target localization, role differentiation, communication relay and coordinated withdrawal. These examples are not defined by the appearance of one fixed collective pattern, but by transitions among multiple functional phases. The key problem is therefore not only to reproduce a particular macroscopic collective state, but also to establish a dynamical framework that can represent and regulate multi-stage behavioral transitions in a unified manner. Such a framework should explain how different behavioral phases are represented in a common state space, how transitions are triggered by local observations and swarm interactions, and how these transitions remain consistent with macroscopic physical laws~\cite{martinoli2004modeling,lerman2001macroscopic,berman2009optimized}.

\subsection{Related Work}

Existing studies on swarm emergence can be broadly grouped into 1) rule-based methods, 2) data-driven learning methods, and 3) physical modelling approaches. Rule-based and bio-inspired methods include the Boids model~\cite{reynolds1987flocks}, particle swarm optimization~\cite{Gbenga2016Understanding}, artificial potential fields~\cite{khatib1986realtime, rimon1992navigation}, social-force models~\cite{Lu2021Swarm} and virtual pheromone mechanisms~\cite{dorigo1997ant,payton2001pheromone}, etc. These approaches are intuitive and interpretable, and have been successfully used to generate aggregation, flocking, obstacle avoidance, formation maintenance and area coverage~\cite{cortes2004coverage,olfati2006flocking}. However, they are usually task specific. Each behavior, or even each behavioral stage within a task, often requires a separately designed set of rules \cite{francesca2014automode,francesca2016automatic}. This makes it difficult for a single model to scale to multi-stage tasks or to adapt to unknown environmental constraints. From the perspective of behavioral-phase representation, rule-based methods effectively treat multi-stage emergence as a concatenation of discrete control logic, rather than as state evolution in a continuous dynamical system~\cite{bayindir2016review, dorigo2021swarm}. They therefore provide useful behavioral primitives, but offer limited support for a unified dynamical description of heterogeneous and multi-stage emergent behaviors.

A second line of work is driven by multi-agent learning. Recent multi-agent reinforcement learning methods, including MAPPO~\cite{lowe2017multi}, QMIX~\cite{rashid2018qmix}, MADDPG~\cite{yu2022surprising} and Transformer-based coordination architectures~\cite{wen2022multi}, have enabled robot swarms and multi-agent systems to learn complex cooperative policies from high-dimensional interaction data. Compared with manually designed rules, these methods offer stronger adaptability in unstructured environments and can optimize task performance through data-driven exploration~\cite{huttenrauch2019deep,sartoretti2019primal}. However, purely data-driven policies typically lack explicit physical structure. A learned policy may perform well in training environments, but its stability and generalization are not guaranteed under changes in environmental topology, swarm density or out-of-distribution conditions~\cite{gronauer2022multiagent, zhang2021decentralized}. More importantly, black-box policies provide limited explanation of why a collective phase emerges, how a behavioral transition occurs, or whether the macroscopic evolution satisfies physical consistency~\cite{karniadakis2021physics}. Learning methods can produce effective strategies, but they do not necessarily reveal whether those strategies correspond to interpretable physical phases, nor whether phase transitions follow continuous, conservative and stable macroscopic dynamics~\cite{gronauer2022multi}.

A third line of work seeks to bridge physical modelling and machine learning. Continuum models~\cite{berman2009optimized}, partial differential equations~\cite{elamvazhuthi2018pde} and active-matter theory~\cite{elamvazhuthi2019bilinear} abstract large-scale swarm systems into density, velocity or information fields, thereby providing compact macroscopic descriptions of collective dynamics~\cite{marchetti2013hydrodynamics,shaebani2020computational}. At the same time, physics-informed neural networks (PINNs) and neural differential equations~\cite{raissi2019physics,karniadakis2021physics,brunton2024promising} provide effective tools for embedding prior physical constraints into learnable models. Recent studies on non-reciprocal field theories~\cite{lama2025nonreciprocal,fruchart2021non} and the learning of hydrodynamic equations for active matter~\cite{supekar2023learning} further suggest that partial-differential-equation-based representations can capture complex collective patterns and macroscopic evolution laws in active systems. Nevertheless, most existing physical models still focus on specific steady patterns, fixed interaction mechanisms or globally constant parameters~\cite{marchetti2013hydrodynamics, shaebani2020computational, elamvazhuthi2019bilinear, supekar2023learning}. They do not yet fully explain how a swarm dynamically switches among multiple behavioral phases as environmental geometry, local observations, swarm density and task context change~\cite{supekar2023learning,lerman2006analysis}. A key missing capability is therefore an explicit formulation of behavioral switching in multi-stage emergence as phase transitions over learnable physical fields.

This limitation becomes particularly pronounced in realistic multi-stage swarm tasks. Traditional partial differential equation models and data-driven learning methods each address only part of this problem. Partial differential equations provide interpretable macroscopic structures for swarm evolution, but static parameterizations limit their responsiveness to dynamic environments and complex task stages. Multi-agent learning methods can acquire adaptive policies from interaction data, but they usually lack explicit constraints on transport consistency, density smoothing and phase evolution. The central challenge addressed in this paper is therefore how to represent spatial migration, density regulation and behavioral phase transitions within a single continuous dynamical framework, while allowing this framework to adapt to complex environments through data and still preserve physical consistency and mechanistic interpretability~\cite{rudy2017data,brunton2016discovering,brunton2024promising}.

\subsection{Contributions}

Here we propose PhySwarm, a physics-informed modelling and control framework for multi-stage emergent behaviors in robot swarms. The core contribution is a coupled micro--macro physical representation that links continuum-level swarm density dynamics with executable robot-level motion. At the macroscopic level, PhySwarm introduces a multi-phase advection--diffusion--reaction model, termed \textbf{Macro-ADR}, to describe how swarm densities evolve through three interpretable mechanisms: advection for task-directed transport, diffusion for density regulation and spatial redistribution, and reaction for transitions among behavioral phases or functional roles. At the microscopic level, PhySwarm introduces an equivalent deterministic motion model, termed \textbf{Micro-EDM}, which realizes the same ADR mechanisms through individual robot motion: advection is implemented by potential-field-based deterministic transport, diffusion by density-gradient compensation, and reaction by event- or rate-driven phase switching. Thus, Macro-ADR provides the field-level description of emergent collective dynamics, whereas Micro-EDM provides the robot-level executable realization.

To adapt this physical representation to unknown and dynamic environments, we further develop a neural-physics controller (\textbf{NPC}) that learns the physical parameters coupling Macro-ADR and Micro-EDM. Instead of directly generating unconstrained velocity commands, the NPC maps local observations and temporal memory to bounded ADR parameters \(P(t)=\{\omega(t),D(t),\lambda(t)\}\). The advection weights \(\omega\) determine how physical field bases are combined, the diffusion coefficient \(D\) regulates density-gradient-driven spreading and local pressure release, and the reaction rates \(\lambda\) govern behavioral phase transitions. These parameters are projected onto a physically feasible manifold and optimized using a RL--PINN objective, where reinforcement learning provides task adaptability and physics-informed residuals constrain the learned parameters to remain consistent with both Macro-ADR density evolution and Micro-EDM execution~\cite{schulman2017ppo, yu2022mappo, raissi2019pinn, karniadakis2021physics, brunton2024promising}.

By coupling Macro-ADR, Micro-EDM and NPC, PhySwarm moves beyond task-specific rule composition and purely black-box swarm policies. It provides a unified mechanism for representing, learning and controlling multi-stage emergent behaviors as physically constrained field evolution coupled to local robot execution. This formulation makes swarm emergence more interpretable, controllable and transferable: different collective behaviors can be generated by changing the learned physical parameter fields and task-specific field dictionaries, while retaining the same underlying ADR structure. More broadly, it offers a physics-informed learning perspective for understanding collective dynamics in distributed systems~\cite{zhang2025physics}. Although this study focuses on robot swarms, the same modelling principle may also inform traffic-flow control~\cite{duan2025attention}, crowd evacuation~\cite{helbing2000simulating}, biological collective motion~\cite{sayin2025behavioral,toner1998flocks} and distributed decision-making in complex systems~\cite{barfuss2025collective}.

\section{Results}

To evaluate PhySwarm, we conducted cross-validation experiments in both simulation and real-world settings using the E-puck swarm robotic platform~\cite{millard2017pi}. PhySwarm uses a swarm emergent behavior model and neural-physics controller to learn the physical parameters, which modulate advection, diffusion and reaction mechanisms in a unified multi-stage density-field representation.

The experiments include three complementary swarm tasks: \textit{Trail-Guided Swarm Foraging}, \textit{Formation-Reconfigurable Swarm Navigation} and \textit{Role-Adaptive Swarm Search and Rescue}. These tasks respectively test information-guided exploration--exploitation, geometry-constrained formation reconfiguration and task-driven role differentiation. We first evaluated behavior evolution in simulation under different swarm sizes, environmental disturbances and individual failure conditions, and then verified key behavioral patterns using real E-puck robots. The theoretical analyses supporting physical interpretability, boundedness, conditional controllability and NPC convergence are provided in the supplementary materials.

\subsection{Overview of PhySwarm}
\label{sec:model_overview}

\begin{figure}[thbp]
    \centering
    \includegraphics[width=\textwidth]{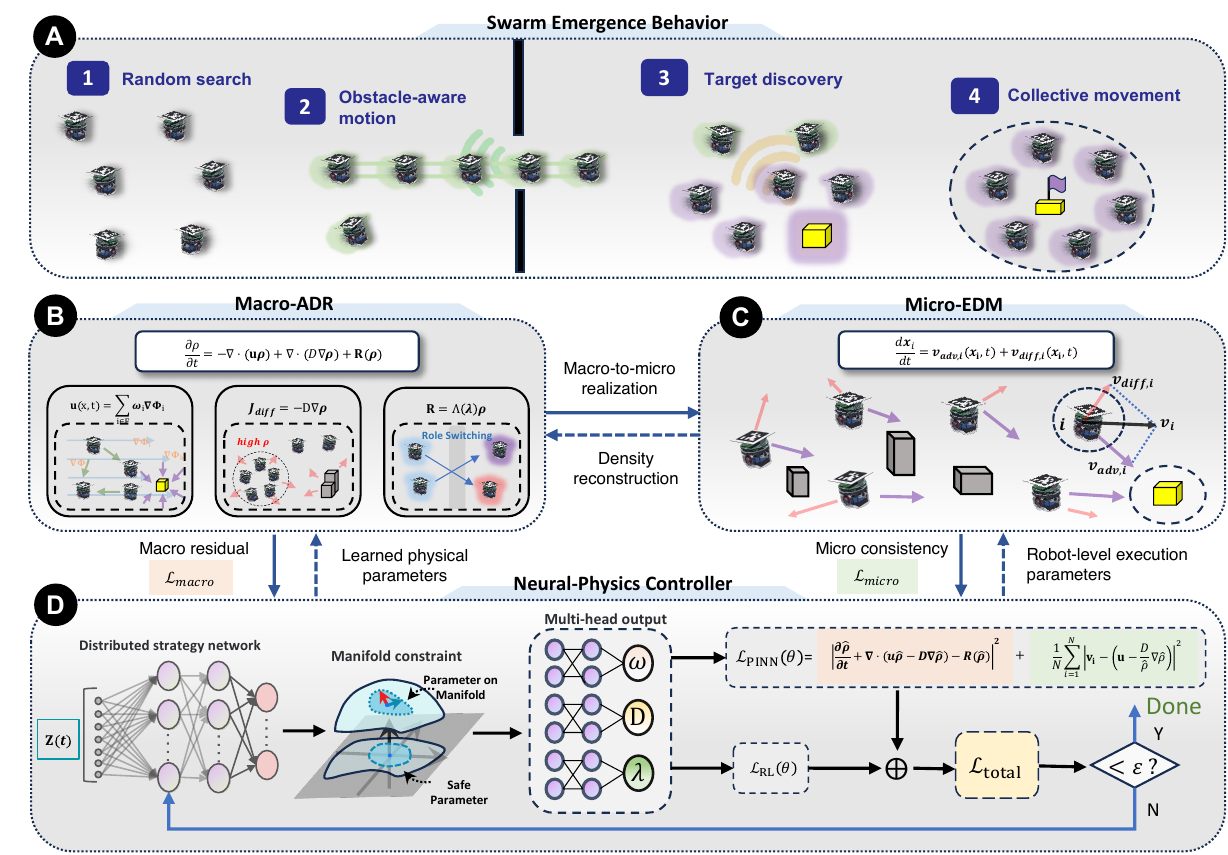} 
    \caption{
    \textbf{PhySwarm couples macroscopic density dynamics with microscopic robot motion through learned physical parameters.}
    \textbf{A,} Schematic of a multi-stage swarm-emergence process, including random search, obstacle-aware motion, target discovery and collective movement.
    \textbf{B,} The macroscopic advection--diffusion--reaction model (Macro-ADR) represents swarm behavior as phase-dependent density evolution governed by advection, diffusion and reaction mechanisms.
    \textbf{C,} The microscopic equivalent deterministic motion model (Micro-EDM) realizes the Macro-ADR mechanisms as executable robot-level motion and phase switching. The bidirectional connection between Macro-ADR and Micro-EDM indicates macro-to-micro realization through phase-conditioned physical fields and micro-to-macro density reconstruction from robot trajectories and behavioral phases.
    \textbf{D,} The neural-physics controller (NPC) maps distributed observations and temporal memory to bounded physical parameters \(P(t)=\{\omega(t),D(t),\lambda(t)\}\). These parameters are projected onto a physically feasible manifold and used to couple Macro-ADR and Micro-EDM. The controller is optimized by combining task-driven reinforcement learning with physics-informed residuals from both the macroscopic density dynamics \(\mathcal{L}_{macro}\) and the microscopic motion consistency \(\mathcal{L}_{micro}\).
    }
    \label{fig:framework}
\end{figure}

PhySwarm is a physics-informed framework for modeling and controlling multi-stage emergent behaviors in robot swarms through a coupled micro--macro representation (see Fig.~\ref{fig:framework}). The motivation is that swarm emergence is neither solely an individual-level control problem nor only a macroscopic pattern-generation problem. Individual robots execute local motion, respond to neighbouring robots and environmental cues, and switch behavioral states; collectively, these local processes induce density transport, spatial redistribution and phase transitions at the swarm level. PhySwarm therefore describes swarm emergence from a field-theoretic perspective: macroscopic collective behaviors are represented as density-field dynamics, while microscopic robot interactions provide an executable realization of those dynamics.

PhySwarm consists of two coupled components: the swarm emergence behavior model and the neural-physics controller. PhySwarm provides the physical representation of swarm emergence at two levels. At the macroscopic level, the Macro-ADR describes phase-dependent density evolution through advection, diffusion and reaction mechanisms (Fig.~\ref{fig:framework}B). At the microscopic level, the Micro-EDM converts the learned physical fields into executable robot motion and phase switching (Fig.~\ref{fig:framework}C). The NPC learns the time-varying physical parameters that couple these two levels and modulate the resulting swarm behavior (Fig.~\ref{fig:framework}D).

At the macroscopic level, we use an advection--diffusion--reaction formulation because multi-stage swarm behaviors naturally involve three classes of collective processes: directed transport, spatial redistribution and behavioral-state conversion (Detailed physical explanations and theoretical derivations are referred to the supplementary materials). Let \(\Omega\subset\mathbb{R}^{d}\) denote the task domain, \(x\in\Omega\) the spatial position and \(t\in[0,T]\) time. Let \(\mathcal{S}=\{1,\ldots,M\}\) be the set of behavioral phases, and let \(\rho_{\sigma}(x,t)\geq0\) denote the density of robots in phase \(\sigma\in\mathcal{S}\). The multi-phase density field is written as \(\rho(x,t)= [\rho_1(x,t),\ldots,\rho_M(x,t)]^{\top}\).

Macro-ADR models (Fig.~\ref{fig:framework}B) the evolution of each phase density as
\begin{equation}
    \frac{\partial \rho_{\sigma}}{\partial t}
    =
    \underbrace{
    -\nabla\cdot
    \left(
    u_{\sigma}\rho_{\sigma}
    \right)
    }_{\text{advection term}}
    +
    \underbrace{
    \nabla\cdot
    \left(
    D_{\sigma}\nabla\rho_{\sigma}
    \right)
    }_{\text{diffusion term}}
    +
    \underbrace{
    R_{\sigma}(\rho)
    }_{\text{reaction term}},
    \qquad
    \sigma=1,\ldots,M .
    \label{eq:results_overview_adr}
\end{equation}
Here, \(u_{\sigma}(x,t)\in\mathbb{R}^{d}\) is the advection velocity field, \(D_{\sigma}(x,t)\geq0\) is the diffusion coefficient and \(R_{\sigma}(\rho)\) is the reaction term describing transitions among behavioral phases. The advection term captures task-directed migration induced by goals, environmental cues and inter-robot interactions. The diffusion term captures density-gradient-driven spatial regulation, including coverage, local pressure release and collision avoidance. The reaction term captures redistribution of robots among behavioral phases or functional roles. This ADR structure follows from a local conservation principle: phase density changes arise from spatial fluxes and conservative phase transitions~\cite{cosner2014reaction,sun2021lie}. The derivation, physical interpretation and mass-conservation property are provided in Sec.~S1.1 of the supplementary materials .

At the microscopic level, Micro-EDM (Fig.~\ref{fig:framework}C) provides a Lagrangian flux-matching realization of Macro-ADR. This level is necessary because physical robots do not execute density equations directly; they execute velocity commands and state transitions. For behavioral phase \(\sigma\), the spatial flux associated with Eq.~\eqref{eq:results_overview_adr} is
\begin{equation}
    J_{\sigma}
    =
    u_{\sigma}\rho_{\sigma}
    -
    D_{\sigma}\nabla\rho_{\sigma}.
    \label{eq:overview_total_flux}
\end{equation}
If the density field is represented by a population of moving robots, an equivalent deterministic particle velocity can be obtained by matching the particle flux \(\rho_{\sigma}v_{\sigma}\) to the Macro-ADR flux \(J_{\sigma}\), namely \(\rho_{\sigma}v_{\sigma} = J_{\sigma}\). This gives
\begin{equation}
    v_{\sigma}
    =
    \frac{J_{\sigma}}{\rho_{\sigma}}
    =
    u_{\sigma}
    -
    \frac{D_{\sigma}}{\rho_{\sigma}}
    \nabla\rho_{\sigma},
    \label{eq:overview_flux_matching_velocity}
\end{equation}
when \(\rho_{\sigma}>0\). This flux-matching construction follows the standard relation between continuity equations, diffusion fluxes and Lagrangian particle transport~\cite{chandrasekhar1943stochastic,philibert2006one}. Thus, Micro-EDM is not an independent heuristic controller, but an executable deterministic motion model whose empirical density is consistent with the Macro-ADR spatial flux in the continuum limit. In implementation, the density \(\rho_{\sigma}\) is replaced by an estimated density \(\hat{\rho}\), and the denominator is regularized by a small constant \(\varepsilon>0\).

For robot \(i\), let \(x_i(t)\in\Omega\) denote its position and \(s_i(t)\in\mathcal{S}\) its current behavioral phase. Its Micro-EDM motion is written as
\begin{equation}
    \dot{x}_i(t)
    =
    v_{\mathrm{adv},i}(t)
    +
    v_{\mathrm{diff},i}(t),
    \label{eq:micro_edm_overview}
\end{equation}
where \(v_{\mathrm{adv},i}(t)\) and \(v_{\mathrm{diff},i}(t)\) are the microscopic advection and diffusion-compensation velocities, respectively. The microscopic advection velocity realizes the first term in Eq.~\eqref{eq:overview_flux_matching_velocity} and is generated by a weighted combination of physical field bases:
\begin{equation}
    v_{\mathrm{adv},i}(t)
    =
    u_{s_i}(x_i,t)
    =
    \sum_{r=1}^{K_b}
    \omega_{s_i,r}(x_i,t)b_r(x_i,t),
    \qquad
    b_r(x,t)=-\nabla\Phi_r(x,t).
    \label{eq:micro_adv_overview}
\end{equation}
Here, \(K_b\) is the number of field bases, \(\Phi_r(x,t)\) is the \(r\)-th potential field, \(b_r(x,t)\) is the velocity basis induced by its negative gradient and \(\omega_{s_i,r}(x_i,t)\) is the corresponding field weight in the current behavioral phase. These fields encode reusable physical cues, such as target attraction, obstacle and boundary avoidance, trail or information memory, morphology regulation, local cohesion and connectivity-aware organization. Thus, the controller learns how to combine physically meaningful fields rather than producing arbitrary low-level motion commands. Typical field bases and their task semantics are summarized in the modeling guidelines in Sec.~S2.2 of the supplementary materials .

The microscopic diffusion-compensation velocity realizes the second term in Eq.~\eqref{eq:overview_flux_matching_velocity}. Given a local density estimate \(\hat{\rho}(x,t)\), which can be phase-conditioned or defined as a generalized total density depending on the task, we define
\begin{equation}
    v_{\mathrm{diff},i}(t)
    =
    -
    \frac{
    D_{s_i}(x_i,t)
    }{
    \hat{\rho}(x_i,t)+\varepsilon
    }
    \nabla \hat{\rho}(x_i,t),
    \qquad
    \varepsilon>0 .
    \label{eq:micro_diff_overview}
\end{equation}
Here, \(\varepsilon\) avoids numerical singularities in low-density regions. This term makes robots respond to local density gradients and induces, at the continuum level, the diffusion flux \(J_{\mathrm{diff},\sigma}  =  - D_{\sigma}\nabla\rho_{\sigma}\). 
Therefore, diffusion in PhySwarm is not treated as unstructured random perturbation, but as an active density-regulation mechanism for coverage, congestion suppression and safety-margin maintenance. The correspondence between Eq.~\eqref{eq:micro_diff_overview} and the diffusion flux in Eq.~\eqref{eq:results_overview_adr} is derived in Sec.~S1.1 of the supplementary materials .

The reaction mechanism is realized through phase switching. For a transition from phase \(m\) to phase \(n\), the microscopic switching probability over a small interval \(\Delta t\) is modeled as
\begin{equation}
    \mathbb{P}
    \left(
    s_i(t+\Delta t)=n
    \mid
    s_i(t)=m
    \right)
    =
    \lambda_{mn}(x_i,t)\chi_{mn}(x_i,t)\Delta t
    +
    o(\Delta t),
    \label{eq:micro_switch_overview}
\end{equation}
where \(\lambda_{mn}(x_i,t)\geq0\) is the learned transition rate and \(\chi_{mn}(x_i,t)\in[0,1]\) is a task-dependent activation function encoding triggering conditions such as target detection, communication degradation, obstacle constraints, role demand or task completion. This microscopic switching process induces the macroscopic reaction term \(R_{\sigma}(\rho)\), which redistributes density among behavioral phases without creating or removing robots. The conservative reaction structure and phase-transition graphs are detailed in Secs.~S1.1 and S2 of the supplementary materials .

The NPC (Fig.~\ref{fig:framework}C) learns the physical parameters that couple Macro-ADR and Micro-EDM. Given the local observation \(O_i(t)\) and temporal memory \(h_{i,t-1}\), the NPC outputs
\begin{equation}
    P_i(t)
    =
    F_{\theta}
    \left(
    O_i(t),h_{i,t-1}
    \right)
    =
    \{\omega_i(t),D_i(t),\lambda_i(t)\},
    \label{eq:results_overview_parameters}
\end{equation}
where \(F_{\theta}\) denotes the neural controller with trainable parameters \(\theta\). The vector \(\omega_i(t)\) determines the composition of advection field bases, \(D_i(t)\) controls density regulation and \(\lambda_i(t)\) controls behavioral phase transitions. The outputs are mapped to a bounded physical parameter manifold: advection weights are non-negative and normalized, diffusion coefficients are positive and bounded, and reaction rates are non-negative and finite. This hard physical projection prevents unbounded velocity amplification, degenerate diffusion and non-physical switching rates. The precise manifold definition and boundedness proof for the induced ADR terms are provided in Sec.~S1.2 of the supplementary materials .

The controller is trained using a RL--PINN objective:
\begin{equation}
\begin{aligned}   
    \mathcal{L}_{\mathrm{total}}
    & =
    \mathcal{L}_{\mathrm{RL}}
    +
    \eta
    \mathcal{L}_{\mathrm{PINN}},\\
    \mathcal{L}_{\mathrm{PINN}}
    & =
    \mathcal{L}_{\mathrm{macro}}
    +
    \beta
    \mathcal{L}_{\mathrm{micro}} .
\end{aligned}
    \label{eq:results_overview_loss}
\end{equation}
Here, \(\mathcal{L}_{\mathrm{RL}}\) is the reinforcement-learning loss (MAPPO~\cite{yu2022mappo} is adopted in this study), \(\mathcal{L}_{\mathrm{macro}}\) penalizes deviations from the Macro-ADR density dynamics, \(\mathcal{L}_{\mathrm{micro}}\) penalizes inconsistencies between the executed robot motion and the Micro-EDM realization, \(\eta>0\) is the physics-regularization weight and \(\mathcal{L}_{\mathrm{PINN}}\) is the total physics-informed residual, where \(\beta>0\) balances the macro- and micro-scale residuals. Through this objective, the learned parameter trajectories \(P(t)=\{\omega(t),D(t),\lambda(t)\}\) are optimized for task performance while remaining consistent with the physical structure of PhySwarm. Detailed network implementation, density estimation, residual construction and scenario-specific parameter settings are provided in Methods and the supplementary materials.

In summary, PhySwarm establishes a closed modeling-and-control loop: local observations are encoded by the NPC; the NPC predicts physically constrained parameters; these parameters define the advection, diffusion and reaction mechanisms in Macro-ADR; and Micro-EDM converts the learned fields into individual robot motion. This loop enables the swarm to modulate its physical mechanisms continuously across behavioral stages. In the following experiments, despite differences in task semantics, they are generated by PhySwarm and controlled through the same learned physical parameter set \(P(t)=\{\omega(t),D(t),\lambda(t)\}\). Detailed model instantiations, field-function choices, network implementation and parameter settings are provided in Methods and the supplementary materials.

\subsection{Emergent behavior scenarios}
\label{sec:emergent_behaviors}

To evaluate whether PhySwarm provides a unified representation of multi-stage swarm emergence, we designed three proof-of-concept scenarios that capture complementary forms of collective organization. These scenarios are not intended to exhaust the full space of swarm-robot tasks. Instead, they test whether the same physical modeling framework can represent distinct types of behavioral transition driven by different mechanisms. We focus on three representative transition mechanisms: information-guided task-state evolution, geometry-constrained spatial reconfiguration and task-driven functional role differentiation.

These mechanisms are instantiated as three benchmark scenarios in this work: trail-guided swarm foraging, formation-reconfigurable swarm navigation and role-adaptive swarm search and rescue. Together, they test whether PhySwarm can represent not only isolated collective patterns, but also continuous transitions among behavioral phases, spatial structures and functional roles. In all scenarios, we use the same modeling pipeline. Behavioral phases are represented by phase densities, task objectives and environmental constraints are encoded as physical fields, density regulation is captured by diffusion, and phase or role transitions are described by conservative reaction terms. The neural-physics controller then learns the time-varying physical parameters \(P(t)=\{\omega(t),D(t),\lambda(t)\}\) that modulate these mechanisms. Detailed phase definitions, field dictionaries, transition graphs, reward terms and implementation settings are provided in the corresponding experimental sections and in Sec.~2.3 of the supplementary materials.

\subsection{Trail-Guided Swarm Foraging}
\label{sec:Foraging}

\begin{figure}[thbp]
    \centering
    \includegraphics[width=\textwidth]{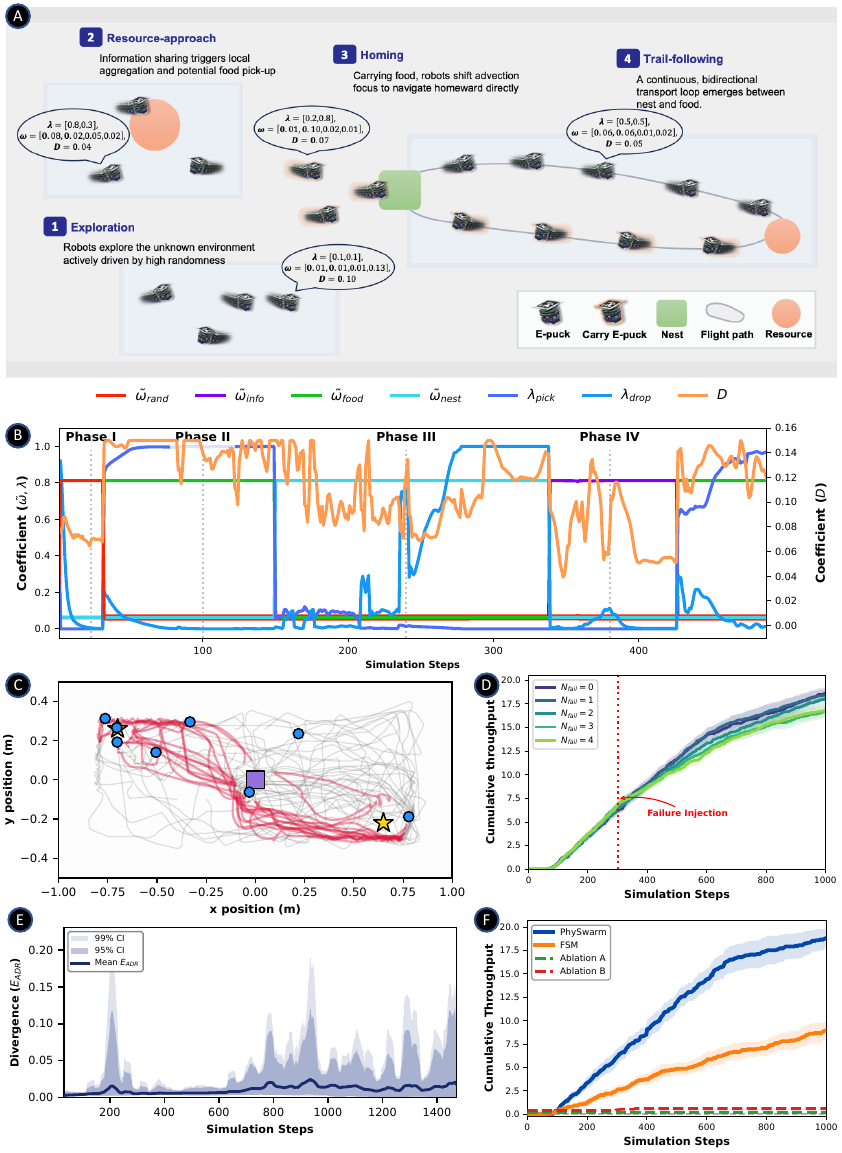} 
    \caption{
\textbf{Trail-guided swarm foraging.}
\textbf{A,} Four-stage foraging process generated by PhySwarm, including dispersed exploration, information-guided resource approach, resource transport to the nest and the formation of a stable bidirectional transport loop. Insets show representative learned physical parameters \(\omega\), \(D\) and \(\lambda\) for each stage.
\textbf{B,} Learned ADR parameters over time, showing coordinated changes in advection weights, reaction rates and diffusion coefficients across exploration, resource-approach, homing and trail-following phases.
\textbf{C,} Robot trajectories in the arena, with highlighted paths indicating the emergent transport route between the food source and the nest.
\textbf{D,} Cumulative macroscopic throughput under robot failures. The dashed line marks the failure-injection time.
\textbf{E,} ADR manifold divergence over simulation steps, with confidence intervals quantifying deviations from the learned physical manifold.
\textbf{F,} Throughput comparison between PhySwarm, a finite-state-machine baseline and fixed-parameter ablation variants. Shaded regions indicate confidence intervals where shown.
The scenario-specific metrics and ablation methods presented in \textbf{D} and \textbf{F} are detailed in Sec.~2.6 and 2.7 of the supplementary materials. The simulation and real-robot validation results are shown in supplementary Movies~S1, S4, S7 and S10.
}
    \label{fig:Foraging}
\end{figure}

The first proof-of-concept scenario evaluates whether PhySwarm can generate an information-guided exploration--exploitation cycle in a swarm foraging task. Robots are initialized without prior knowledge of the food location and must explore the workspace, transport food items to the nest, and reuse accumulated trail information to improve subsequent foraging (Fig.~\ref{fig:Foraging}A). This scenario tests a central requirement of the proposed framework: whether dispersed exploration, resource approach, homing and trail following can be represented as phase-conditioned density evolution under a shared dynamical model, rather than as separately hand-coded behaviors.

The learned physical parameters reveal how the NPC drives this multi-stage transition (see Fig.~\ref{fig:Foraging}B). During early exploration, the controller maintains relatively high diffusion and weakly directed advection, allowing robots to cover the arena and search for food-related cues. Once food information becomes available, the advection weights shift towards food- and information-related fields, and the corresponding reaction rates promote transitions from exploration to resource approach. After pick-up, the dominant advection component changes towards the nest field, guiding carrying robots back to the nest. As repeated transport occurs, trail-related fields become increasingly important, enabling the swarm to reuse an emergent transport route while preserving a residual exploratory component (see Fig.~\ref{fig:Foraging}B,C). Thus, the observed trajectory structure is not imposed as a pre-defined path; it arises from continuous modulation of \(\omega\), \(D\) and \(\lambda\) across behavioral phases.

PhySwarm also maintains the foraging process under perturbation and remains consistent with the learned physical manifold. After robot failures are injected, the cumulative throughput continues to increase, although the slope decreases as the number of active robots is reduced (Fig.~\ref{fig:Foraging}D). This suggests that the transport route is not tied to a fixed subset of robots, but is maintained by the collective density-field dynamics. The ADR manifold divergence remains bounded during long-horizon execution, with transient increases mainly corresponding to phase switching, route reorganization and failure-induced redistribution (see Fig.~\ref{fig:Foraging}E). Since this metric measures the discrepancy between empirical robot density and the ADR-induced reference density, the bounded divergence indicates that the learned behavior remains close to the physical structure specified by PhySwarm. 

Finally, PhySwarm achieves higher cumulative throughput than both the finite-state-machine baseline and the fixed-parameter ablation variants (see Fig.~\ref{fig:Foraging}F). The ablation variants use representative fixed values of the ADR parameters rather than learning time-varying physical parameter trajectories. The performance gap indicates that adaptive modulation of advection, diffusion and reaction is important for sustaining the full foraging cycle, especially when exploration, transport and trail reuse require different physical regimes. Overall, this scenario demonstrates that trail-guided swarm foraging can be represented as a continuous, physically constrained transition among behavioral phases within the unified PhySwarm framework.

\subsection{Formation-Reconfigurable Swarm Navigation}
\label{sec:formation_reconfigurable_navigation}

\begin{figure}[thbp]
    \centering
    \includegraphics[width=\textwidth]{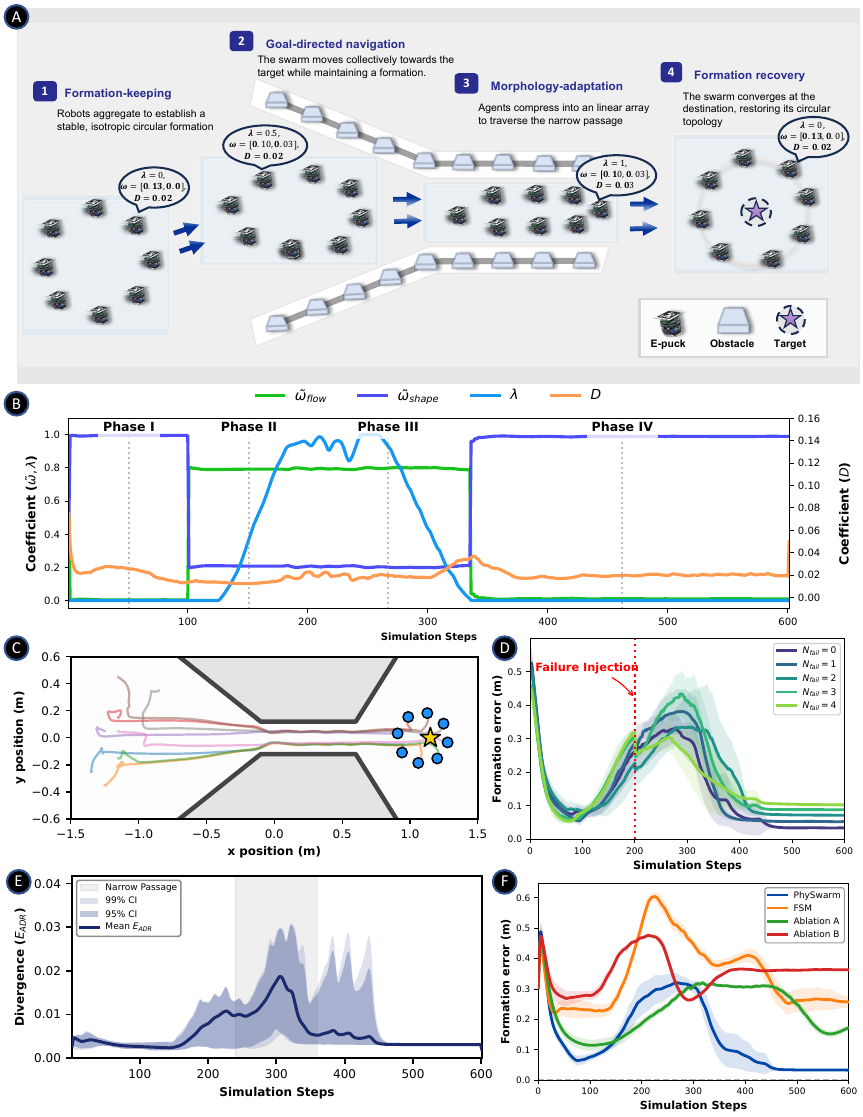} 
    \caption{
\textbf{Formation-reconfigurable swarm navigation.}
\textbf{A,} Schematic of the multi-stage navigation task. The swarm first aggregates into a stable isotropic circular formation, then senses environmental constraints and enters a topological morphing phase. To traverse the narrow corridor, robots compress into an anisotropic linear arrangement and subsequently recover the isotropic circular topology after reaching the target region. Representative physical parameters \(\omega\), \(D\) and \(\lambda\) are shown for each stage.
\textbf{B,} Learned ADR parameters during a representative trial. The normalized advection weights, reaction extent and diffusion coefficient vary consistently across four phases: formation-keeping, goal-directed navigation, morphology-adaptation and formation-recovery. The shift from isotropic to shape-adaptive advection weights indicates activation of morphology reconfiguration, whereas the reaction signal captures the transient switching process.
\textbf{C,} Robot trajectories during corridor traversal and target formation-recovery. The trajectories show that the swarm compresses through the constrained region and reforms near the target.
\textbf{D,} Circular-formation error under different numbers of failed robots. The vertical dashed line indicates failure injection. Although failures transiently increase formation error, the swarm subsequently recovers towards the target morphology.
\textbf{E,} ADR divergence over simulation steps. The mean curve and confidence intervals quantify the deviation of the learned dynamics from the ADR physical representation during formation-keeping, morphology-adaptation and formation-recovery.
\textbf{F,} Comparison of circular-formation error among the proposed method, a potential-field baseline and fixed-parameter ablation variants. The proposed method achieves more stable morphology recovery after corridor traversal. Shaded regions indicate confidence intervals where shown.
The scenario-specific metrics and ablation methods presented in \textbf{D} and \textbf{F} are detailed in Sec.~2.6 and 2.7 of the supplementary materials. The simulation and real-robot validation results are shown in supplementary Movies~S2, S5, S8 and S11.
}
    \label{fig:navigation}
\end{figure}

The second proof-of-concept scenario evaluates whether PhySwarm can generate geometry-constrained spatial reconfiguration during collective navigation. In this task, robots must move towards a target region while adapting their collective morphology to a narrow corridor (Fig.~\ref{fig:navigation}A). This scenario tests whether formation keeping, morphology adaptation, corridor traversal and formation recovery can be represented as continuous phase-conditioned density evolution under the same PhySwarm dynamics, rather than as a sequence of manually designed formation controllers.

The learned physical parameters show how the NPC modulates the swarm morphology in response to environmental geometry (see Fig.~\ref{fig:navigation}B). During formation keeping, the controller assigns a relatively high contribution to the isotropic formation field, allowing the swarm to aggregate into a compact circular topology. As the swarm approaches the constrained corridor, the advection weights shift towards shape-adaptive and corridor-oriented fields, while the reaction signal activates the transition towards morphology adaptation. During traversal, the anisotropic component becomes dominant, enabling the swarm to compress along the corridor direction while limiting lateral dispersion. After leaving the constrained region, the isotropic formation component recovers and the reaction signal decreases, driving the swarm back towards the target circular morphology (see Fig.~\ref{fig:navigation}B,C). Thus, the observed transition from circular aggregation to corridor-compatible compression and subsequent recovery is generated by continuous modulation of \(\omega\), \(D\) and \(\lambda\), rather than by a hard switch between predefined formation templates.

PhySwarm remains robust to robot failures and maintains physical consistency during the reconfiguration process (see Fig.~\ref{fig:navigation}D). When robot failures are injected, the circular-formation error increases transiently, with larger perturbations under more severe failures. Nevertheless, the error decreases again as the remaining robots reorganize and enter the recovery phase. This suggests that the formation is not maintained by a fixed set of robots, but by collective density-field regulation through the learned physical parameters. The ADR manifold divergence remains bounded throughout the episode, with temporary increases during corridor compression and formation recovery (see Fig.~\ref{fig:navigation}E). These transients correspond to the most demanding part of the task, where the swarm must simultaneously satisfy environmental constraints and morphology objectives. 

Finally, PhySwarm achieves lower circular-formation error than the potential-field baseline and the fixed-parameter ablation variants (see Fig.~\ref{fig:navigation}F). The ablation variants use representative fixed ADR parameters rather than learning time-varying physical parameter trajectories. The performance gap indicates that adaptive modulation of advection, diffusion and reaction is important for coordinating formation maintenance, corridor compression and morphology recovery within a single task. Overall, this scenario demonstrates that geometry-constrained swarm navigation can be represented as a continuous, physically constrained reconfiguration process within the unified PhySwarm framework.

\subsection{Role-Adaptive Swarm Search and Rescue}
\label{sec:role_adaptive_search_rescue}

\begin{figure}[thbp]
    \centering
    \includegraphics[width=\textwidth]{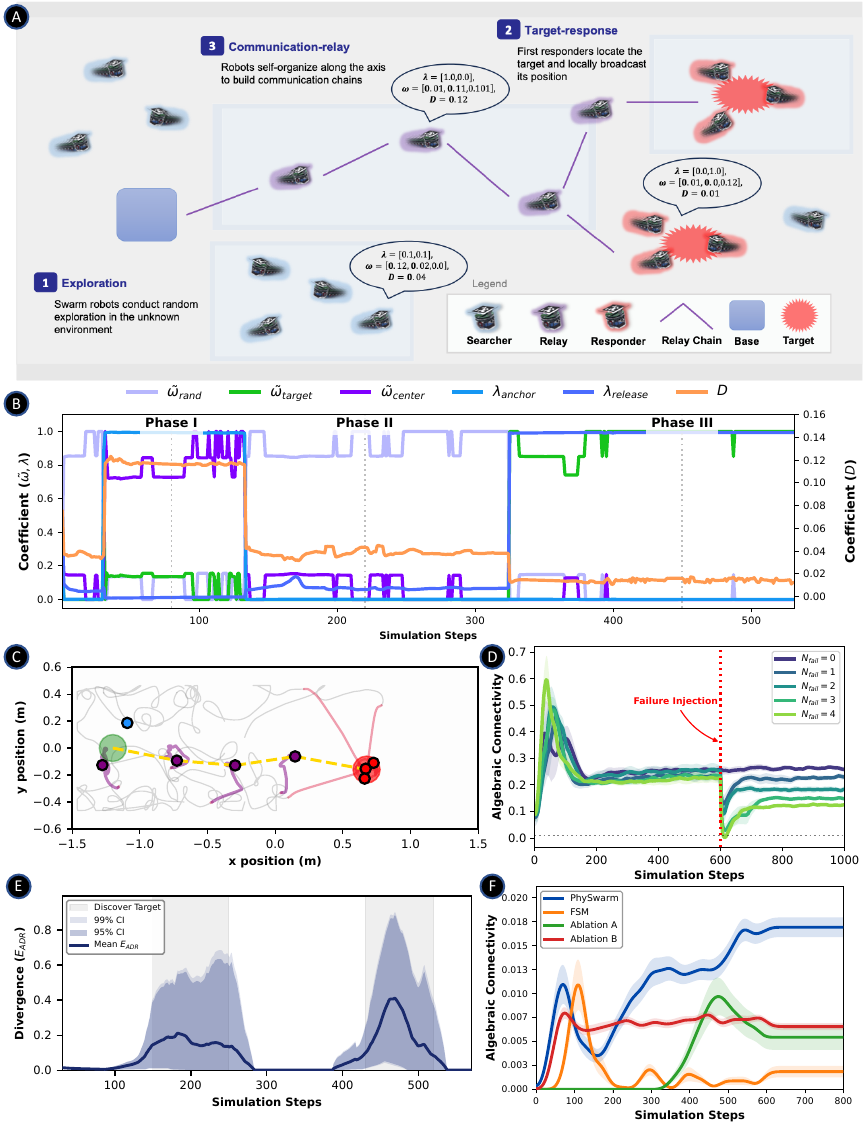} 
    \caption{
\textbf{Role-adaptive swarm search and rescue.}
\textbf{A,} Multi-stage search-and-rescue process generated by PhySwarm, including distributed search, target response and relay-chain formation. Representative learned parameters \(\omega\), \(D\) and \(\lambda\) are shown for each stage.
\textbf{B,} Learned ADR parameters over time, showing coordinated modulation of advection weights, reaction rates and diffusion coefficients across distributed-search, target-response and communication-relay phases.
\textbf{C,} Robot trajectories and role distribution, in which searchers explore the arena, responders gather near the target and relay robots maintain communication between the base and target region.
\textbf{D,} Algebraic connectivity under robot failures. The dashed line marks failure injection.
\textbf{E,} Topological manifold divergence over time, with transient increases during role redistribution and relay reorganization.
\textbf{F,} Connectivity comparison with a finite-state-machine baseline and fixed-parameter ablation variants. Shaded regions indicate confidence intervals where shown.
The scenario-specific metrics and ablation methods presented in \textbf{D} and \textbf{F} are detailed in Sec.~2.6 and 2.7 of the supplementary materials. The simulation and real-robot validation results are shown in supplementary Movies~S3, S6, S9 and S12.
}
    \label{fig:SAR}
\end{figure}

The third proof-of-concept scenario evaluates whether PhySwarm can generate task-driven role differentiation in a decentralized search-and-rescue task. Robots must explore an unknown environment, detect a target and maintain a communication pathway between the base and the target region (see Fig.~\ref{fig:SAR}A). This scenario tests whether distributed search, target response and communication relay can be represented as phase-conditioned density evolution under the same dynamics, rather than as predefined role assignments or separately designed task-allocation rules.

The learned physical parameters show how the NPC induces role differentiation from an initially homogeneous swarm (see Fig.~\ref{fig:SAR}B). During distributed search, exploration-related advection and diffusion support broad spatial coverage. Once target information is detected and propagated locally, reaction rates promote transitions from the search phase to the response phase, while target-directed advection guides responders towards the target region. As communication demand increases, relay-related advection fields and role-transition rates organize a subset of robots along the base--target direction, forming a communication backbone while the remaining robots continue search or response behaviors (see Fig.~\ref{fig:SAR}B,C). Thus, the observed searcher--responder--relay organization is generated by continuous modulation of \(\omega\), \(D\) and \(\lambda\), rather than by assigning fixed roles to specific robots.

PhySwarm also preserves relay connectivity under perturbation and remains consistent with the learned physical organization (see Fig.~\ref{fig:SAR}D). After robot failures are injected, algebraic connectivity decreases, with larger drops under more severe failures, but the learned policy maintains non-zero connectivity across the tested conditions. This indicates that the relay structure is not tied to a fixed set of robots; instead, the remaining robots can reorganize through density-field regulation and reaction-driven role redistribution. The topological manifold divergence remains bounded over the episode, with transient increases during target response and relay reconfiguration (see Fig.~\ref{fig:SAR}E). These transients correspond to periods in which the swarm changes its functional topology, and their subsequent reduction indicates that the learned role-switching dynamics remain close to the physical structure specified by PhySwarm. 

Finally, PhySwarm maintains stronger relay connectivity than the finite-state-machine baseline and the fixed-parameter ablation variants (see Fig.~\ref{fig:SAR}F). The ablation variants use representative fixed ADR parameters rather than learning time-varying physical parameter trajectories. The performance gap indicates that adaptive coupling among search-related advection, target attraction, relay alignment, density regulation and reaction-based role transitions is important for sustaining a functional communication chain. Overall, this scenario demonstrates that role-adaptive search and rescue can be represented as a continuous, physically constrained transition from distributed exploration to target response and relay coordination within the unified PhySwarm framework.

\subsection{Density-field evolution and physical interpretability}
\label{sec:density_field_evolution}
\begin{figure}[htbp]
    \centering
    \includegraphics[width=\textwidth]{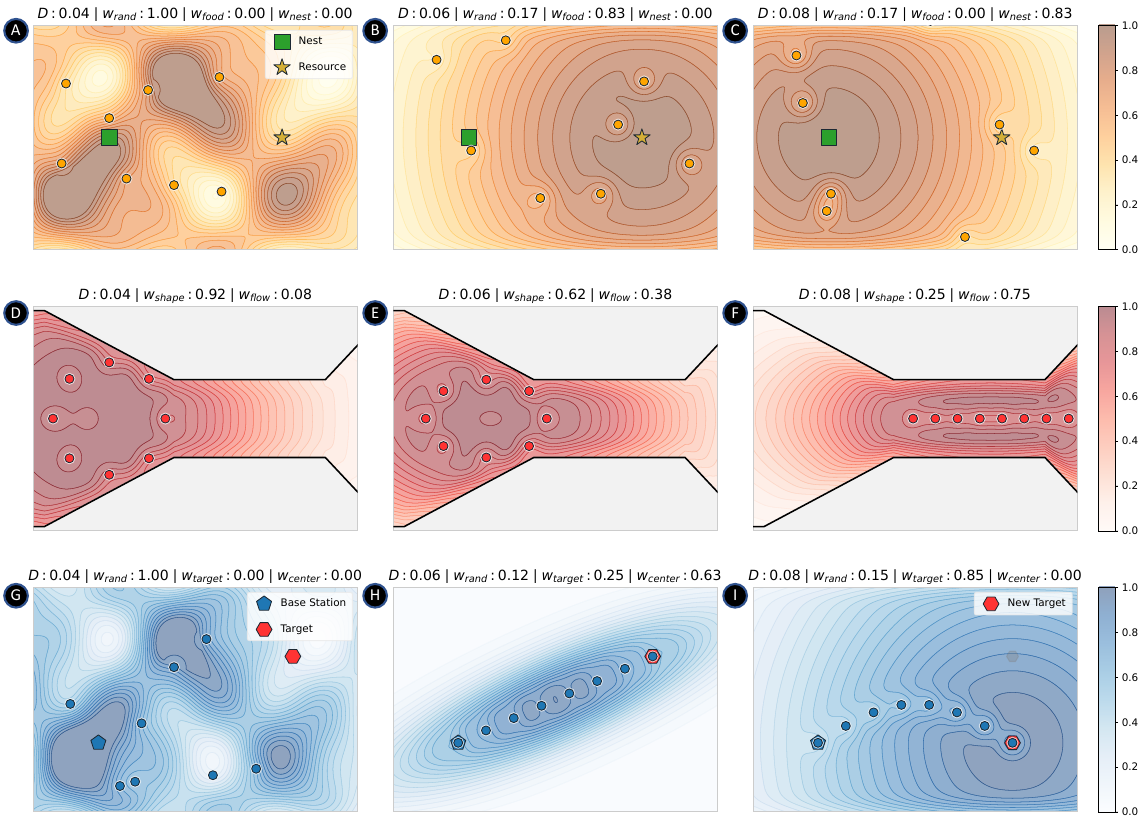}
    \caption{
    \textbf{Density-field evolution and physical visualization across three multi-stage swarm tasks.}
    \textbf{A--C,} \textit{Trail-Guided Swarm Foraging}. The learned field regulation shifts from exploration-dominated search (\textbf{A}) to food-directed aggregation (\textbf{B}) and nest-directed homing (\textbf{C}). 
    \textbf{D--F,} \textit{Formation-Reconfigurable Swarm Navigation}. The density field evolves from an isotropic circular formation in open space (\textbf{D}), through a topological morphing phase near the corridor entrance (\textbf{E}), to an anisotropic linear arrangement inside the narrow corridor (\textbf{F}).
    \textbf{G--I,} \textit{Role-Adaptive Swarm Search and Rescue}. The swarm first performs distributed exploration (\textbf{G}), then forms a relay-oriented structure between the base and target (\textbf{H}), and finally reorients the density field towards a new target region (\textbf{I}). 
    In all panels, coloured contours denote the normalized density-field landscape, robot markers show the corresponding microscopic realization, and the annotated values indicate representative learned diffusion coefficients and advection-field weights. These visualizations show how the neural-physics controller modulates the Macro-ADR mechanisms and realizes them through Micro-EDM across task stages (see Movie~S1--S3 in the supplementary materials). The derivation and physical interpretation of each Macro-ADR term, together with additional potential-field visualizations, are provided in the supplementary materials. The physical visualization results are shown in supplementary Movies~S1--S3.  
    }
    \label{fig:density_field_evolution}
\end{figure}

Beyond task-level performance, we examined whether the learned physical parameters produce interpretable density-field evolution across different swarm tasks. This analysis addresses a central question of PhySwarm: whether the neural-physics controller (NPC) merely learns task-specific actions, or instead learns to regulate the physical mechanisms encoded by the Macro-ADR representation. Fig.~\ref{fig:density_field_evolution} shows representative density-field snapshots from the three benchmark scenarios. Across all tasks, the learned parameters reshape the swarm distribution through the same physical channels: advection weights \(\omega\) determine the dominant transport direction, diffusion coefficients \(D\) regulate spatial spreading and local pressure release, and reaction rates \(\lambda\) determine when robots switch between phase-conditioned field responses.

A common pattern emerges from the three scenarios. During exploratory or weakly constrained phases, the learned fields produce broader density distributions, allowing the swarm to cover uncertain regions  (see Fig.~\ref{fig:density_field_evolution}A,D,G). When task-relevant information becomes available, such as a food location, a narrow corridor or a rescue target, the density field becomes more structured and aligns with the corresponding physical field (see Fig.~\ref{fig:density_field_evolution}B,E,H). In phases requiring transport, corridor or relay formation, the density distribution is further reorganized into task-specific spatial manifolds, such as a homing flow, a compressed formation or a base--target communication corridor (see Fig.~\ref{fig:density_field_evolution}C,F,I). These transitions show that multi-stage swarm behavior is realized through continuous modulation of physical fields, rather than through abrupt switching among unrelated hand-designed controllers.

In \textit{Trail-Guided Swarm Foraging}, the density field evolves from exploration-dominated dispersion to resource-directed aggregation and then to nest-directed homing (see Fig.~\ref{fig:density_field_evolution}A--C). This progression reflects a shift in the dominant advection component from exploration to food attraction and finally to nest attraction, while diffusion regulates local congestion around resource and nest regions. In \textit{Formation-Reconfigurable Swarm Navigation},  the density field changes from an isotropic circular formation in open space to an anisotropic linear arrangement inside the narrow corridor (see Fig.~\ref{fig:density_field_evolution}D--F). This indicates that formation reconfiguration is achieved by continuously rebalancing morphology preservation and forward navigation. In \textit{Role-Adaptive Swarm Search and Rescue}, the density field first supports distributed search, then organizes robots along a relay-oriented corridor between the base and target, and finally redirects the response towards the target region (see Fig.~\ref{fig:density_field_evolution}G--I). This demonstrates that role differentiation is expressed as phase-conditioned spatial organization at the density-field level.

Together, these visualizations provide physical evidence for the micro--macro coupling in PhySwarm. Macro-ADR describes the evolution of phase-conditioned density fields, Micro-EDM realizes these fields through executable robot motion and phase switching, and the NPC couples the two levels by learning \(P(t)=\{\omega(t),D(t),\lambda(t)\}\). The same physical parameter interface therefore supports qualitatively different emergent behaviors across tasks, while preserving interpretability in terms of advection, diffusion and phase transition.

\section{Discussion}
\label{sec:discussion}

This study investigates whether multi-stage emergent behaviors in robot swarms can be represented, learned and controlled as transitions over physical fields, rather than as predefined behavioral scripts, task-specific rules or purely black-box policies. The results across trail-guided foraging, formation-reconfigurable navigation and role-adaptive search and rescue support this perspective. Although these scenarios differ in task objectives, environmental constraints and organizational forms, they can all be represented through the same coupled micro--macro structure, in which density evolution, directed transport, diffusion-based spatial regulation and reaction-driven phase transitions jointly shape collective behaviors. This suggests that PhySwarm captures a shared field-level mechanism underlying different swarm emergent behaviors, instead of fitting isolated trajectories or manually enumerating task stages.

The central contribution of PhySwarm is therefore not merely the reproduction of several collective behaviors, but the formulation of swarm emergence as a physics-informed micro--macro modelling and learning problem. At the macroscopic level, Macro-ADR describes how swarm densities evolve under three interpretable mechanisms. The advection term captures directed migration induced by task goals, potential fields, environmental cues and collective information; the diffusion term captures density regulation, coverage, local pressure release and safety-related spatial redistribution; and the reaction term captures transitions among behavioral phases or functional roles. At the microscopic level, Micro-EDM converts these field-level mechanisms into executable robot motion and phase switching. This coupling links continuum-level density dynamics with individual robot execution, allowing emergent behavior to be analysed as both a macroscopic field process and a decentralized control process.

A key feature of this framework is that the physical fields are not switched by manually designed rules. Instead, the neural-physics controller (NPC) learns the time-varying physical parameters \(P(t)=\{\omega(t),D(t),\lambda(t)\}\). The advection weights \(\omega\) determine the contribution of different potential-field bases, the diffusion coefficient \(D\) regulates density-gradient-driven spreading and local smoothing, and the reaction rates \(\lambda\) control transitions among behavioral phases. This parameterization gives the learning policy a physically meaningful action space: the controller does not directly output arbitrary low-level robot commands, but learns how to modulate interpretable physical mechanisms. The physical manifold projection further constrains these parameters to a bounded feasible manifold, preventing unbounded advection, degenerate diffusion and non-physical phase-switching rates. The corresponding boundedness analysis is provided in Sec.~S1.2 of the supplementary materials.

The physics-informed training objective also plays an important role. Without the PINN residual constraint, a multi-agent reinforcement-learning policy may still discover reward-improving behaviors, but the resulting policy can be difficult to interpret, analyse or transfer. By incorporating the Macro-ADR and Micro-EDM residuals into the MAPPO training objective, the learned density evolution and microscopic execution are encouraged to remain consistent with the prescribed physical mechanisms. In this sense, the physical model is not used only as a post-hoc explanation of learned behavior; it actively regularizes the learning process. This is particularly important for multi-stage swarm behaviors, where the same robots must continuously redistribute themselves across spatial patterns and functional roles while satisfying environmental and safety constraints.

The three benchmark scenarios provide complementary evidence for this formulation. Trail-guided swarm foraging demonstrates information-guided emergence: dispersed exploration, resource approach, homing and trail reuse arise from the coupling of resource, nest and trail-information fields with diffusion and phase transitions. Formation-reconfigurable swarm navigation demonstrates geometry-constrained emergence: the swarm maintains a compact formation in open space, adapts its morphology when encountering a narrow corridor and recovers its formation after the constraint is removed. Role-adaptive search and rescue demonstrates task-driven functional differentiation: robots transition among search, response and relay roles according to target confidence, communication quality and local role demand. These behaviors are not generated by separate task-specific controllers. Instead, they are different instantiations of PhySwarm, with scenario-specific field dictionaries, transition graphs and learned parameter trajectories. Additional implementation details, evaluation metrics and extended experimental results are provided in the supplementary materials.

PhySwarm differs from several existing paradigms. Rule-based swarm-control methods are often interpretable, but they require task-specific engineering and become difficult to scale to long-horizon multi-stage behaviors. Classical formation-control and potential-field methods provide useful geometric primitives, but they usually do not describe how behavioral phases emerge, switch or interact at the level of swarm density. Black-box multi-agent reinforcement learning can learn complex policies, but often provides limited mechanistic insight and weak physical guarantees. PhySwarm combines complementary strengths of these approaches: Macro-ADR provides an interpretable field-level representation, Micro-EDM connects this representation to executable robot motion, potential-field bases encode reusable physical primitives, NPC provides adaptive parameter learning, and PINN residuals constrain the learned policy toward physically consistent swarm evolution.

Several limitations remain. First, PhySwarm relies on a continuum-field approximation of the swarm distribution. This approximation is well suited to the Macro-ADR formulation when the swarm provides sufficient spatial samples for density reconstruction. However, for very small swarms, highly sparse deployments or severely fragmented configurations, kernel-based density estimates and the associated PDE residuals may become less reliable. This limitation reflects the finite-sample nature of mapping discrete robot trajectories to continuous density fields, rather than an intrinsic failure of the ADR formulation. It is also consistent with the trends observed in the scalability and fault-tolerance analyses, where reduced population size or robot failures can weaken collective-field regulation and degrade task performance. Second, the ADR formulation provides an interpretable parameterization of swarm dynamics, but it is not unique. The choice of behavioral phases, potential-field bases and reaction structures can affect the compactness, expressiveness and identifiability of the learned model. Third, although the present experiments include both simulation and \textit{e-puck} robot validation, broader evaluation is still needed in heterogeneous swarms, larger workspaces, three-dimensional environments and long-duration field deployments.

Future work may address these limitations along several directions. More robust density estimation and residual-evaluation methods would improve the applicability of PhySwarm to sparse, fragmented or partially observable swarms. Adaptive field dictionaries, neural operators and equation-discovery methods could reduce the reliance on predefined potential fields and phase-transition graphs, allowing parts of the physical structure to be inferred from data. Another important direction is to extend the current homogeneous-swarm formulation to heterogeneous robots with different sensing, mobility and communication capabilities. These extensions would further test whether physics-informed micro--macro learning can serve as a general modeling principle for distributed systems in which local interactions generate large-scale collective dynamics.

\section{Methods}
\label{sec:methods}

PhySwarm consists of the swarm emergent behavior model and the neural-physics controller (NPC). The model overview in Sec.~\ref{sec:model_overview} introduces the physical rationale and the coupled formulation. Here we describe how the framework is instantiated, trained and evaluated in the experiments, while detailed field definitions, theoretical derivations and task-specific configurations are provided in Sec.~S2 of the supplementary materials.

We consider \(N\) robots moving in a bounded two-dimensional task domain \(\Omega\subset\mathbb{R}^{2}\). For robot \(i\in\{1,\ldots,N\}\), \(x_i(t)\in\Omega\) denotes its position at time \(t\), \(v_i(t)\in\mathbb{R}^{2}\) its executed planar velocity, \(O_i(t)\) its local observation, \(h_{i,t-1}\) its recurrent memory state and \(\sigma_i(t)\in\mathcal{S}=\{1,\ldots,M\}\) its behavioral phase. The phase variable represents the current functional state of the robot, such as searching, carrying, morphing, responding or relaying, depending on the task.

At each control step, the NPC predicts a physically constrained parameter set
\begin{equation}
    P_i(t)
    =
    \{\omega_i(t),D_i(t),\lambda_i(t)\},
    \label{eq:methods_parameter_set}
\end{equation}
where \(\omega_i(t)\) controls the composition of advection field bases, \(D_i(t)\) controls diffusion-based density regulation and \(\lambda_i(t)\) controls behavioral phase transitions. These parameters are then used by Micro-EDM to generate executable robot motion and phase switching.

\subsection{Macro-ADR construction}
\label{sec:methods_macro_adr}

The Macro-ADR model requires a continuous approximation of the swarm distribution and task-dependent definitions of the advection, diffusion and reaction components. This subsection describes how these components are constructed in the experiments. Their physical interpretation, field-function examples and theoretical properties are provided in Sec.~S1 of the supplementary materials.

For each behavioral phase \(\sigma\in\mathcal{S}\), we estimate the phase density from robot positions using kernel density estimation~\cite{hamann2008framework}:
\begin{equation}
    \hat{\rho}_{\sigma}(x,t)
    =
    \frac{1}{N h^{d}}
    \sum_{i\in\mathcal{I}_{\sigma}(t)}
    K
    \left(
    \frac{x-x_i(t)}{h}
    \right),
    \label{eq:methods_density_estimation}
\end{equation}
where \(\mathcal{I}_{\sigma}(t)=\{i:\sigma_i(t)=\sigma\}\) is the set of robots currently assigned to phase \(\sigma\), \(K(\cdot)\) is a smooth kernel, \(h>0\) is the kernel bandwidth and \(d\) is the spatial dimension (\(d=2\) in this paper). The normalization by \(N\) makes \(\int_\Omega \hat{\rho}_{\sigma}(x,t)\,dx\) proportional to the fraction of robots in phase \(\sigma\). If an absolute robot-number density is required, the estimate can be rescaled by \(N\). The estimated total density is
\begin{equation}
    \hat{\rho}_{\mathrm{tot}}(x,t)
    =
    \sum_{\sigma=1}^{M}
    \hat{\rho}_{\sigma}(x,t).
    \label{eq:methods_total_density}
\end{equation}
The estimated phase densities \(\hat{\rho}_{\sigma}\) are used to evaluate the Macro-ADR residual, whereas \(\hat{\rho}_{\mathrm{tot}}\) or a task-specific generalized density field is used for local density-gradient compensation in Micro-EDM.

\textbf{The advection component} is constructed from a finite dictionary of physical field bases. For the \(r\)-th basis, we define
\begin{equation}
    b_r(x,t)
    =
    -\nabla \Phi_r(x,t),
    \qquad
    r=1,\ldots,K_b,
    \label{eq:methods_field_basis}
\end{equation}
where \(\Phi_r(x,t)\) is a scalar potential field, \(b_r(x,t)\in\mathbb{R}^{2}\) is the velocity basis induced by its negative gradient and \(K_b\) is the number of field bases. The phase-conditioned advection field is
\begin{equation}
    u_{\sigma}(x,t)
    =
    \sum_{r=1}^{K_b}
    \omega_{\sigma,r}(x,t)b_r(x,t),
    \label{eq:methods_advection_field}
\end{equation}
where \(u_{\sigma}(x,t)\in\mathbb{R}^{2}\) is the advection velocity field for phase \(\sigma\), and \(\omega_{\sigma,r}(x,t)\) is the learned contribution of field basis \(r\). The field dictionary is task dependent and may include target, obstacle, boundary, trail-information, morphology, cohesion, separation, alignment and communication-quality fields. These fields encode reusable physical cues, whereas the learned weights determine their relative contributions.

\textbf{The diffusion component} is constructed from the learned diffusion coefficient \(D_{\sigma}(x,t)\) and the estimated phase density. At the continuum level, the diffusion flux~\cite{philibert2006one} for phase \(\sigma\) is
\begin{equation}
    J_{\mathrm{diff},\sigma}(x,t)
    =
    -
    D_{\sigma}(x,t)
    \nabla \hat{\rho}_{\sigma}(x,t),
    \label{eq:methods_diffusion_flux}
\end{equation}
where \(J_{\mathrm{diff},\sigma}(x,t)\in\mathbb{R}^{2}\) denotes the diffusion flux and \(D_{\sigma}(x,t)\geq0\) denotes the phase-dependent diffusion coefficient. This term is used in the Macro-ADR residual to represent density-gradient-driven spatial regulation. In robot-level execution, the same diffusion mechanism is implemented through the density-gradient compensation term in Micro-EDM. Depending on the scenario, the density field used for local compensation may be the robot density alone or a generalized density field that also includes obstacle, boundary or risk constraints.

\textbf{The reaction component} is constructed from the learned transition rates \(\lambda_{mn}(x,t)\) and a task-specific transition mask. Here, \(\lambda_{mn}(x,t)\geq0\) denotes the local transition rate from phase \(m\) to phase \(n\). We use a conservative reaction matrix \(\Lambda(\lambda)\in\mathbb{R}^{M\times M}\), whose off-diagonal entries encode inflow from other phases and whose diagonal entries encode the corresponding outflow:
\begin{equation}
    \Lambda_{\sigma\eta}(\lambda)
    =
    \left\{
    \begin{array}{ll}
        \lambda_{\eta\sigma}, 
        & \sigma \neq \eta, \\[4pt]
        -\displaystyle\sum_{\ell \neq \sigma}\lambda_{\sigma\ell},
        & \sigma = \eta .
    \end{array}
    \right.
    \label{eq:methods_reaction_matrix}
\end{equation}
Let \( \hat{\rho}(x,t) = [\hat{\rho}_1(x,t),\ldots,\hat{\rho}_M(x,t)]^{\top} \) denote the estimated multi-phase density vector. The macroscopic reaction term for phase \(\sigma\) is
\begin{equation}
    R_{\sigma}(\hat{\rho})
    =
    \left[
    \Lambda(\lambda)\hat{\rho}
    \right]_{\sigma}.
    \label{eq:methods_reaction_term}
\end{equation}
This structure redistributes density among behavioral phases without changing the total robot mass. The derivation of the conservative reaction structure and its mass-conservation property are provided in Sec.~S1 of the supplementary materials.

Together, Eqs.~\eqref{eq:methods_advection_field}, \eqref{eq:methods_diffusion_flux} and \eqref{eq:methods_reaction_term} define the task-specific Macro-ADR instance used in each experiment. Complete field definitions, diffusion-field choices, phase-transition graphs and activation functions are provided in Sec.~S2 of the supplementary materials.

\subsection{Micro-EDM execution}
\label{sec:methods_micro_edm}

During execution, Micro-EDM converts the learned physical fields into robot-level velocity commands. For robot \(i\), the desired planar velocity~\cite{degond1990deterministic} is
\begin{equation}
    v_i^{\mathrm{des}}(t)
    =
    u_{s_i}(x_i,t)
    -
    \frac{
    D_{s_i}(x_i,t)
    }{
    \hat{\rho}_{\mathrm{loc}}(x_i,t)+\varepsilon
    }
    \nabla\hat{\rho}_{\mathrm{loc}}(x_i,t),
    \qquad
    \varepsilon>0.
    \label{eq:methods_micro_edm_velocity}
\end{equation}
Here, \(v_i^{\mathrm{des}}(t)\in\mathbb{R}^{2}\) is the desired planar velocity, \(u_{s_i}(x_i,t)\) is the advection velocity associated with the current behavioral phase \(s_i(t)\), \(\hat{\rho}_{\mathrm{loc}}(x_i,t)\) is the density estimate used for local density-gradient compensation and \(\varepsilon\) avoids numerical singularities in low-density regions. Depending on the scenario, \(\hat{\rho}_{\mathrm{loc}}\) can be a phase-conditioned robot density or a generalized density field that also includes obstacle or boundary constraints.

For differential-drive E-puck robots, the desired planar velocity is converted into translational and angular velocity commands through heading tracking with velocity saturation:
\begin{equation}
\begin{aligned}
    \nu_i^{\mathrm{cmd}}
    & =
    \min
    \left(
    \|v_i^{\mathrm{des}}\|_2,
    \nu_{\max}
    \right),\\
    \varpi_i^{\mathrm{cmd}}
    & =
    \operatorname{clip}
    \left(
    k_{\psi}\,
    \operatorname{wrap}
    \left(
    \angle v_i^{\mathrm{des}}-\psi_i
    \right),
    -\varpi_{\max},
    \varpi_{\max}
    \right),
\end{aligned}
    \label{eq:methods_differential_drive}
\end{equation}
where \(\nu_i^{\mathrm{cmd}}\) is the commanded translational velocity, \(\varpi_i^{\mathrm{cmd}}\) is the commanded angular velocity, \(\psi_i\) is the heading of robot \(i\), \(\nu_{\max}\) and \(\varpi_{\max}\) are actuator limits, \(k_{\psi}>0\) is the heading-control gain, \(\angle v_i^{\mathrm{des}}\) denotes the direction of \(v_i^{\mathrm{des}}\), and \(\operatorname{wrap}(\cdot)\) maps an angle to \((-\pi,\pi]\). The same Micro-EDM formulation is used in simulation and real-robot experiments; platform-specific differences arise from actuator limits, sensing noise and communication constraints.

Behavioral phase transitions are governed by learned transition rates and task-dependent activation functions. A transition from phase \(m\) to phase \(n\) follows
\begin{equation}
    \mathbb{P}
    \left(
    s_i(t+\Delta t)=n
    \mid
    s_i(t)=m
    \right)
    =
    \lambda_{mn}(x_i,t)\chi_{mn}(x_i,t)\Delta t
    +
    o(\Delta t),
    \label{eq:methods_phase_transition}
\end{equation}
where \(\Delta t\) is a small control interval, \(\chi_{mn}(x_i,t)\in[0,1]\) is a task-specific activation function and \(o(\Delta t)\) denotes a higher-order term satisfying \(o(\Delta t)/\Delta t\rightarrow0\) as \(\Delta t\rightarrow0\). The activation function encodes triggering conditions such as food detection, nest arrival, corridor detection, target confidence, communication degradation or task completion.

\subsection{Neural-Physics Controller}
\label{sec:methods_npc}

The NPC maps local observations and temporal memory to physical parameters:
\begin{equation}
    P_i(t)
    =
    F_{\theta}
    \left(
    O_i(t),h_{i,t-1}
    \right) =
    \{\omega_i(t),D_i(t),\lambda_i(t)\},
    \label{eq:methods_npc_mapping}
\end{equation}
where \(F_{\theta}\) denotes the neural controller with trainable parameters \(\theta\). The observation \(O_i(t)\) contains task-relevant local information available to robot \(i\), including neighboring robots, obstacle cues, density estimates, target or food detections, trail information, formation error and communication-quality measurements, depending on the scenario. The network uses a shared observation encoder, a recurrent module for temporal memory and three physical output heads for \(\omega\), \(D\) and \(\lambda\).

To ensure physical feasibility, raw neural outputs are mapped to a bounded parameter manifold. Let \(z_{\omega,i,r}\), \(z_{D,i}\) and \(z_{\lambda,mn,i}\) denote the unconstrained outputs of the advection, diffusion and reaction heads, respectively. We use
\begin{equation}
    \begin{aligned}
        \omega_{i,r}
        & =
        \frac{\exp(z_{\omega,i,r})}
        {\sum_{q=1}^{K_b}\exp(z_{\omega,i,q})},
        \qquad
        r=1,\ldots,K_b,\\
        D_i
        & =
        D_{\min}
        +
        (D_{\max}-D_{\min})
        \operatorname{sigmoid}
        \left(
        z_{D,i}
        \right),\\
        \lambda_{mn,i}
        & =
        A_{mn}
        \lambda_{\max}
        \operatorname{sigmoid}
        \left(
        z_{\lambda,mn,i}
        \right).
    \end{aligned}
\label{eq:methods_reaction_matrix}
\end{equation}
where \(A_{mn}\in\{0,1\}\) is the task-specific phase-transition mask, \(D_{\min}>0\) and \(D_{\max}<\infty\) are the diffusion bounds, and \(\lambda_{\max}<\infty\) is the maximum transition rate. Thus, the advection weights form a non-negative normalized combination of field bases, diffusion coefficients remain positive and bounded, and reaction rates remain non-negative and finite. The boundedness analysis of the induced advection, diffusion and reaction terms is provided in Sec.~S1 of the supplementary materials.

The NPC is trained using a RL--PINN objective. The reinforcement-learning component optimizes task performance, whereas the physics-informed component penalizes deviations from PhySwarm:
\begin{equation}
    \mathcal{L}_{\mathrm{total}}(\theta)
    =
    \mathcal{L}_{\mathrm{RL}}(\theta)
    +
    \eta
    \mathcal{L}_{\mathrm{PINN}}(\theta),
    \label{eq:methods_total_loss}
\end{equation}
where \(\mathcal{L}_{\mathrm{total}}\) is the total training loss, \(\mathcal{L}_{\mathrm{RL}}\) is the MAPPO~\cite{yu2022mappo} loss, \(\mathcal{L}_{\mathrm{PINN}}\) is the physics-informed residual and \(\eta>0\) is the physics-regularization weight. In our implementation, \(\mathcal{L}_{\mathrm{RL}}\) includes the clipped policy loss, value-function loss and entropy regularization following the MAPPO training paradigm~\cite{yu2022mappo}.

The physics-informed loss is
\begin{equation}
    \mathcal{L}_{\mathrm{PINN}}
    =
    \mathcal{L}_{\mathrm{macro}}
    +
    \beta
    \mathcal{L}_{\mathrm{micro}},
    \label{eq:methods_pinn_loss}
\end{equation}
where \(\mathcal{L}_{\mathrm{macro}}\) is the macroscopic Macro-ADR residual, \(\mathcal{L}_{\mathrm{micro}}\) is the microscopic Micro-EDM consistency residual and \(\beta>0\) balances the two residuals. The macroscopic residual measures the deviation of estimated phase-density evolution from the Macro-ADR dynamics:
\begin{equation}
    \mathcal{L}_{\mathrm{macro}}
    =
    \sum_{\sigma=1}^{M}
    \left\|
    \frac{\partial \hat{\rho}_{\sigma}}{\partial t}
    +
    \nabla\cdot
    \left(
    u_{\theta,\sigma}\hat{\rho}_{\sigma}
    \right)
    -
    \nabla\cdot
    \left(
    D_{\theta,\sigma}\nabla \hat{\rho}_{\sigma}
    \right)
    -
    \left[
    \Lambda(\lambda_{\theta})\hat{\rho}
    \right]_{\sigma}
    \right\|_{L^{2}(\Omega\times[0,T])}^{2}.
    \label{eq:methods_macro_loss}
\end{equation}
Here, \(u_{\theta,\sigma}\), \(D_{\theta,\sigma}\) and \(\lambda_{\theta}\) are the advection field, diffusion coefficient and reaction rates induced by the current NPC parameters, respectively. The matrix \(\Lambda(\lambda_{\theta})\) is the conservative reaction matrix induced by the learned transition rates, and \(\|\cdot\|_{L^{2}(\Omega\times[0,T])}\) denotes the space--time \(L^2\) norm over the task domain and training horizon. The microscopic residual aligns the executed or predicted robot velocity with the Micro-EDM velocity:
\begin{equation}
    \mathcal{L}_{\mathrm{micro}}
    =
    \frac{1}{N}
    \sum_{i=1}^{N}
    \left\|
    v_i(t)
    -
    v_i^{\mathrm{des}}(t)
    \right\|_2^2 .
    \label{eq:methods_micro_loss}
\end{equation}
This objective constrains policy search to physically meaningful parameter trajectories while allowing reinforcement learning to optimize task performance. The conditional convergence analysis of the RL--PINN controller is provided in Sec.~S1 of the supplementary materials.

\subsection{Scenario-specific instantiation}
\label{sec:methods_scenarios}

The same PhySwarm implementation is used across all three benchmark tasks in the Results section. In each task, the NPC outputs the same physical parameter set \(P(t)=\{\omega(t),D(t),\lambda(t)\}\), and the learned parameters are executed through the same micro--macro coupling. Thus, the scenarios differ not in the underlying model architecture, but in how PhySwarm is instantiated for each task.

For each scenario, we specify five task-dependent components: 1) the behavioral phase set \(\mathcal{S}\), 2) the physical field dictionary used to construct advection, 3) the phase-transition graph and activation functions used by the reaction term, 4) the density or constraint fields used for diffusion regulation, and 5) the reward function used by MAPPO. These components define the task-specific semantics, whereas the parameter projection, Micro-EDM execution, Macro-ADR residual and RL--PINN training objective remain shared. Detailed field equations, transition masks, activation functions, reward terms, hyperparameters, simulation settings and real-robot implementation details are provided in Sec.~S2 of the supplementary materials.

\subsection{Theoretical guarantees}
\label{sec:methods_theory}

The theoretical analysis supporting PhySwarm is provided in Sec.~S1 of the supplementary materials. We first derive the multi-stage swarm representation from a local conservation principle, formulating swarm behavior as phase-dependent ADR density dynamics. In this formulation, advection, diffusion and reaction respectively represent task-directed migration, density-gradient-driven spatial regulation and transitions among behavioral phases.

We then prove that the physical parameters predicted by the NPC, \(P(t)=\{\omega(t),D(t),\lambda(t)\}\), remain within a bounded feasible parameter manifold and that the induced advection field, diffusion term and reaction term are bounded under standard regularity conditions. We further establish conditional controllability of the Macro-ADR behavior representation by constructing a target density path, matching phase-mass redistribution through reaction rates, deriving the desired advection field from the spatial flux and approximating it using a finite set of potential-field bases.

Finally, we analyze the conditional convergence of the RL--PINN controller. Under the existence of a physically feasible reference trajectory, bounded optimization error and controlled physics residuals, the learned parameter trajectory approximates the ideal trajectory and induces swarm density evolution close to the target emergent-behavior density. Complete definitions, assumptions, theorems and proofs are provided in Secs.~S1 of the supplementary materials.

\subsection{Analysis metric}
\label{sec:methods:analysis-metrics}

To evaluate whether the learned swarm evolution is consistent with the physical structure encoded by PhySwarm, we use the ADR manifold divergence as the main task-independent physical-consistency metric. It measures how closely the empirical swarm density follows the density manifold implied by the learned advection and diffusion fields. We also report one representative task-related metric for each benchmark scenario: cumulative foraging throughput in Fig.~\ref{fig:Foraging}, formation error in Fig.~\ref{fig:navigation}, and the relay-chain connectivity in Fig.~\ref{fig:SAR}. Their detailed mathematical definitions, normalization procedures and additional scenario-specific analyses are provided in Sec.~S2.6 of the supplementary materials.

\textbf{ADR manifold divergence.}
For each behavioral phase \(\sigma\), the learned advection field defines an effective potential \(\Phi_{\sigma}^{\mathrm{eff}}(x,t)\), obtained from the phase-conditioned weighted combination of potential-field bases. Under a quasi-steady zero-flux approximation, the corresponding reference density can be written as a normalized Boltzmann-type distribution~\cite{romanczuk2012active}:
\begin{equation}
    \rho_{\sigma}^{\mathrm{ref}}(x,t)
    =
    \frac{
    \exp\!\left[
    -\Phi_{\sigma}^{\mathrm{eff}}(x,t)/
    D_{\sigma}^{\mathrm{eff}}(t)
    \right]
    \mathbb{I}(x\in\Omega_{\sigma})
    }{
    \int_{\Omega_{\sigma}}
    \exp\!\left[
    -\Phi_{\sigma}^{\mathrm{eff}}(y,t)/
    D_{\sigma}^{\mathrm{eff}}(t)
    \right]dy
    },
    \label{eq:methods_ref_density}
\end{equation}
where \(D_{\sigma}^{\mathrm{eff}}(t)>0\) denotes the effective diffusion strength, \(\Omega_{\sigma}\) is the spatial support associated with phase \(\sigma\), and \(\mathbb{I}(\cdot)\) is the indicator function. The empirical phase density \(\hat{\rho}_{\sigma}^{\mathrm{emp}}(x,t)\) is reconstructed from robot positions using kernel density estimation~\cite{chen2017tutorial}.

We define the ADR manifold divergence as the phase-mass-weighted \(L^2\) discrepancy between the empirical density and the reference density:
\begin{equation}
    E_{\mathrm{ADR}}(t)
    =
    \sum_{\sigma=1}^{M}
    \frac{N_{\sigma}(t)}{N}
    \left\|
    \hat{\rho}_{\sigma}^{\mathrm{emp}}(\cdot,t)
    -
    \rho_{\sigma}^{\mathrm{ref}}(\cdot,t)
    \right\|_{L^{2}(\Omega)}^{2},
    \label{eq:methods_adr_divergence}
\end{equation}
where \(N_{\sigma}(t)\) is the number of robots in phase \(\sigma\), \(N\) is the total number of robots, and \(M\) is the number of behavioral phases. A smaller \(E_{\mathrm{ADR}}(t)\) indicates that the observed swarm distribution is closer to the physical density manifold specified by the learned Macro-ADR fields.

\section{Conclusion}\label{section:5}

We introduced \textit{PhySwarm}, a physics-informed micro--macro framework for modeling and controlling multi-stage emergent behaviors in robot swarms. The framework couples macroscopic advection--diffusion--reaction density dynamics with microscopic equivalent deterministic motion, thereby linking field-level swarm evolution to executable robot actions. A neural-physics controller learns bounded physical parameters \(P(t)=\{\omega(t),D(t),\lambda(t)\}\) from local observations and temporal memory, while macro-scale and micro-scale physics residuals constrain the learned policy to remain consistent with the underlying physical representation.

Across three benchmark scenarios, the same PhySwarm framework generated distinct emergent behaviors without changing the underlying dynamical representation. Experiments showed that the learned advection weights, diffusion coefficients and reaction rates can reshape swarm-density fields, support behavioral phase transitions and produce executable robot motion under different task demands. These results suggest that multi-stage swarm emergence can be represented as physically constrained field evolution, rather than as disconnected task-specific rules or purely black-box policies.

More broadly, PhySwarm provides a route towards interpretable, learnable and controllable swarm intelligence. Future work will extend the framework to sparse and heterogeneous swarms, richer and adaptive field dictionaries, three-dimensional workspaces and long-duration deployments. Another important direction is to discover physical field structures and phase-transition mechanisms directly from data, further reducing reliance on predefined modeling choices while preserving physical consistency and interpretability.

\section*{Data Availability}

The data supporting the findings of this study are available from the corresponding author upon reasonable request during peer review. Upon formal acceptance of the manuscript, the data and codes will be made publicly available in an online repository.

\section*{Code Availability}

The data supporting the findings of this study are available from the corresponding author upon reasonable request during peer review. Upon formal acceptance of the manuscript, the data and codes will be made publicly available in an online repository.

\section*{Acknowledgements}

This work was supported in part by the National Natural Science Foundation of China (NSFC) under Grant 62101029, and in part by the National Science and Technology Major Project - Advanced Materials under Grant 2024ZD0608000.

\section*{Author Contributions}

C.X. conceived the study, proposed and refined the core methodology, supervised the project, and led the overall research. Z.J. contributed to the design of the proposed method and was mainly responsible for implementing the framework. W.Z. led the hardware implementation, experimental design, and experimental validation. S.Q. and Z.D. contributed to the implementation of the models and algorithms. F.Y. implemented the fault-tolerance experiments. Y.S. contributed to model implementation and validation. C.X. supervised the theoretical analysis, experimental validation, and manuscript preparation. Z.J., W.Z., S.Q., Z.D., F.Y., and Y.S. contributed to experimental analysis, result organization, and supplementary materials. J.Z. and C.X. drafted and revised the manuscript with input from all authors. All authors reviewed, discussed, and approved the final manuscript.

\section*{Competing Interests}

All authors declare no financial or non-financial competing interests.

\section*{Figures and Tables}

\noindent
{\bf Fig.1.} PhySwarm couples macroscopic density dynamics with microscopic robot motion through learned physical parameters. \\
{\bf Fig.2.} Trail-guided swarm foraging. \\
{\bf Fig.3.} Formation-reconfigurable swarm navigation. \\
{\bf Fig.4.} Role-adaptive swarm search and rescue. \\
{\bf Fig.5.} Density-field evolution and physical visualization across three multi-stage swarm tasks. 

\section*{Supplementary Materials}

Supplementary materials and movies are available from the project website: \url{https://physwarm.github.io/}.

\vspace{3mm}

\noindent {\bf Section S1: Supplementary Results} 

{\bf Section S1.1.} Physical and theoretical explanations 

{\bf Section S1.2.} Fault tolerance analysis

{\bf Section S1.3.} Scalability analysis 

{\bf Section S1.4.} Versatility analysis 

{\bf Section S1.5.} The boundedness 

{\bf Section S1.6.} Controllability of the model

{\bf Section S1.7.} Convergence of the controller

\noindent {\bf Section S2: Supplementary Methods}

{\bf Section S2.1.} General modeling guidelines 

{\bf Section S2.2.} Examples of typical fields

{\bf Section S2.3.} Scenario-specific model implementation

{\bf Section S2.4.} NPC network implementation  

{\bf Section S2.5.} Experimental arena setup  

{\bf Section S2.6.} Supplementary evaluation metrics

{\bf Section S2.7.} Baseline and ablation methods 

\noindent {\bf Section S3: Supplementary Note} 

{\bf Section S3.1.} Emergent Behavior v.s. Collective Behavior \\
{\it Figures S1--S18 and Tables S1--S11 are included in Secs.~S1-S3.}\\

\noindent {\it Movies S1--S12 provide supplementary demonstrations of the PhySwarm framework, including density-field visualizations, fault-tolerance and scalability tests, and real-robot validations. Each movie illustrates the corresponding swarm behavior together with its physical interpretation when applicable.}\\

\noindent {\bf Movie S1.} Density-field evolution and physical visualization of Trail-Guided Swarm Foraging.\\
{\bf Movie S2.} Density-field evolution and physical visualization of Formation-Reconfigurable Swarm Navigation.\\
{\bf Movie S3.} Density-field evolution and physical visualization of Role-Adaptive Swarm Search and Rescue.\\

\noindent {\bf Movie S4.} Fault-tolerance verification in simulation for Trail-Guided Swarm Foraging.\\
{\bf Movie S5.} Fault-tolerance verification in simulation for Formation-Reconfigurable Swarm Navigation.\\
{\bf Movie S6.} Fault-tolerance verification in simulation for Role-Adaptive Swarm Search and Rescue.\\

\noindent {\bf Movie S7.} Scalability verification in simulation for Trail-Guided Swarm Foraging.\\
{\bf Movie S8.} Scalability verification in simulation for Formation-Reconfigurable Swarm Navigation.\\
{\bf Movie S9.} Scalability verification in simulation for Role-Adaptive Swarm Search and Rescue.\\

\noindent {\bf Movie S10.} Real-robot validation of Trail-Guided Swarm Foraging.\\
{\bf Movie S11.} Real-robot validation of Formation-Reconfigurable Swarm Navigation.\\
{\bf Movie S12.} Real-robot validation of Role-Adaptive Swarm Search and Rescue.\\






\clearpage

\newtheorem{problem}{Problem}
\newtheorem{theorem}{Theorem}
\newtheorem{remark}{Remark}
\newtheorem{assumption}{Assumption}
\newtheorem{lemma}[theorem]{Lemma}
\newtheorem{definition}{Definition}
\newcommand{\cmnt}[1]{\textcolor{red}{\bf[#1]}}
\newcommand\topstrut[1][1.2ex]{\setlength\bigstrutjot{#1}{\bigstrut[t]}}
\newcommand\botstrut[1][0.9ex]{\setlength\bigstrutjot{#1}{\bigstrut[b]}}
\newcommand{\RN}[1]{%
  \textup{\uppercase\expandafter{\romannumeral#1}}%
}

\makeatletter

\setcounter{section}{0}
\setcounter{subsection}{0}
\setcounter{subsubsection}{0}
\setcounter{equation}{0}
\setcounter{figure}{0}
\setcounter{table}{0}

\renewcommand{\thesection}{S\arabic{section}}
\renewcommand{\thesubsection}{\thesection.\arabic{subsection}}
\renewcommand{\thesubsubsection}{\thesubsection.\arabic{subsubsection}}

\renewcommand{\thefigure}{S\@arabic\c@figure}
\renewcommand{\thetable}{S\@arabic\c@table}
\renewcommand{\theequation}{S\@arabic\c@equation}
\renewcommand{\theproblem}{S\@arabic\c@problem}
\renewcommand{\thetheorem}{S\@arabic\c@theorem}
\newtheorem{\thelemma}{S\@arabic\c@lemma}
\renewcommand{\theassumption}{S\@arabic\c@assumption}
\renewcommand{\theremark}{S\@arabic\c@remark}
\makeatother

\setcounter{MaxMatrixCols}{15}
\setcounter{figure}{0}
\setcounter{equation}{0}

\begin{titlepage}
\thispagestyle{empty}

\begin{center}

\vspace*{12mm}

{\Large Supplementary Materials for\par}

\vspace{6mm}

{\large\textsc{
\textbf{Physics-Informed Modeling and Control of\\
Emergent Behaviors in Robot Swarms}
}\par}

\vspace{12mm}

{\large
Zixuan Jin$^1$, Wenzhuo Zhang$^1$, Shuxian Quan$^2$, Zirui Dong$^1$, \\
Fangwen Ye$^1$, Yuchen Shi$^1$, Cheng Xu$^{1,2}$\textsuperscript{\Letter}
\par}

\vspace{8mm}

{\large 
$^1$School of Computer and Communication Engineering,\\
University of Science and Technology Beijing, Beijing, China
\par}

\vspace{3mm}

{\large
$^2$Shunde Innovation School,\\
University of Science and Technology Beijing, Guangdong, China
\par}

\vspace{3mm}

{\large
\textsuperscript{\Letter}Correspondence: \url{xucheng@ustb.edu.cn}
\par}

\end{center}

\end{titlepage}

\clearpage
\tableofcontents
\clearpage

\addtocontents{toc}{\protect\setcounter{tocdepth}{3}}
\section{Supplementary Results}

\subsection{Physical and theoretical explanations}

This section provides the physical interpretation and theoretical foundation of the proposed swarm behavior representation. We first introduce the notation, boundary condition and conservative reaction structure used to model multi-stage swarm behaviors as phase-dependent density fields in a bounded task domain. We then explain how the three terms in the ADR equation correspond to interpretable swarm mechanisms: advection represents directed migration induced by task goals, potential fields and local interactions; diffusion represents density regulation, coverage and safety-driven spreading; and reaction represents transitions among behavioral phases or functional roles. These analyses show that the proposed model is not only a descriptive field representation, but also a physically constrained and learnable dynamical framework for multi-stage emergent behaviors.

\subsubsection{Notations and assumptions}

To describe multi-stage emergent behaviors in robot swarms, we model the swarm as a continuous density field evolving in a bounded task domain. Let
\(\Omega\subset\mathbb{R}^{d}\) denote the task space, where \(d=2\) corresponds to planar robot motion. The variables \(x\in\Omega\) and \(t\in[0,T]\) denote spatial position and time, respectively. Unlike single-stage collective motion, multi-stage emergent behaviors~\cite{elamvazhuthi2019bilinear} usually involve continuous transitions among multiple behavioral phases. We therefore define a finite set of behavioral phases as \(\mathcal{S}=\{1,2,\ldots,M\}\), where each \(\sigma\in\mathcal{S}\) represents a macroscopic behavioral phase, such as exploration, aggregation, etc.

For each behavioral phase \(\sigma\), let \(\rho_{\sigma}(x,t)\geq 0\) denote the density of robots in phase \(\sigma\) at position \(x\) and time \(t\). The multi-phase density field~\cite{elamvazhuthi2019bilinear,zhu2007asymptotic,zhu2009strong} is written as
\begin{equation}
    \rho(x,t)=
    \left[
    \rho_{1}(x,t),
    \rho_{2}(x,t),
    \ldots,
    \rho_{M}(x,t)
    \right]^{T}.
\end{equation}
The total swarm density is
\begin{equation}
    \rho_{\mathrm{tot}}(x,t)
    =
    \sum_{\sigma=1}^{M}\rho_{\sigma}(x,t),
\end{equation}
and the total swarm mass is
\begin{equation}
    \mathcal{M}(t)
    =
    \sum_{\sigma=1}^{M}
    \int_{\Omega}\rho_{\sigma}(x,t)\,dx.
\end{equation}
When \(\rho\) is interpreted as a probability density, one typically has \(\mathcal{M}(t)=1\). When \(\rho\) is interpreted as a robot number density, \(\mathcal{M}(t)\) represents the normalized total number of robots.

This leads to the following multi-stage advection--diffusion--reaction equation (Eq.~(1) in the main paper):
\begin{equation}
    \frac{\partial \rho_{\sigma}}{\partial t}
    +
    \nabla\cdot
    \left(
    u_{\sigma}\rho_{\sigma}
    \right)
    =
    \nabla\cdot
    \left(
    D_{\sigma}\nabla\rho_{\sigma}
    \right)
    +
    R_{\sigma}(\rho),
    \qquad
    \sigma=1,\ldots,M.
    \label{eq:multi_stage_adr_cn}
\end{equation}
Here, \(u_{\sigma}(x,t)\) is the advection velocity field of phase \(\sigma\), \(D_{\sigma}(x,t)\geq 0\) is the diffusion coefficient, and \(R_{\sigma}(\rho)\) is the reaction term induced by transitions among behavioral phases.

Robot swarms usually operate in bounded task regions, such as closed simulation arenas, or workspaces constrained by obstacles and task boundaries (see Sec.~\ref{sec:arena} for descriptions of the arenas used in this work). Let \(\partial\Omega\) denote the boundary of the task domain, and let \(\mathbf{n}\) be the outward unit normal vector on \(\partial\Omega\). For behavioral phase \(\sigma\), the total spatial flux~\cite{keller1971model} is defined as
\begin{equation}
    J_{\sigma}
    =
    J_{\mathrm{adv},\sigma}
    +
    J_{\mathrm{diff},\sigma}
    =
    u_{\sigma}\rho_{\sigma}
    -
    D_{\sigma}\nabla\rho_{\sigma},
    \label{eq:total_flux_for_boundary_cn}
\end{equation}
where the advective flux \(J_{\mathrm{adv},\sigma}=u_{\sigma}\rho_{\sigma}\) describes deterministic migration driven by task objectives and potential fields, and the diffusive flux \(J_{\mathrm{diff},\sigma}=-D_{\sigma}\nabla\rho_{\sigma}\) describes density-gradient-driven spatial spreading and pressure release. The reaction term \(R_{\sigma}(\rho)\) describes density exchange among behavioral phases.

\begin{assumption}[Zero-flux boundary condition~\cite{stewart1980generation}]
\label{assump:zero_flux}
For each behavioral phase \(\sigma=1,\ldots,M\), the net flux across the task-domain boundary is zero:
\begin{equation}
    \mathbf{n}\cdot J_{\sigma}
    =
    \mathbf{n}\cdot
    \left(
    u_{\sigma}\rho_{\sigma}
    -
    D_{\sigma}\nabla\rho_{\sigma}
    \right)
    =
    0,
    \qquad
    x\in\partial\Omega.
    \label{eq:zero_flux_boundary_cn}
\end{equation}
\end{assumption}

\begin{remark}
This assumption states that robots may undergo directed migration and density-driven spreading inside the task domain, but they do not leave the workspace through its boundary. Therefore, the advection and diffusion terms only redistribute the swarm within \(\Omega\), and do not cause loss of total swarm mass through the boundary.
\end{remark}

\begin{assumption}[Conservative reaction structure~\cite{seabrook2023tutorial}]
\label{assump:cons_react}
The reaction term \(R_{\sigma}(\rho)\) describes transitions of robots among behavioral phases, rather than the creation or disappearance of robots. Specifically, we use the following conservative inflow--outflow structure:
\begin{equation}
    R_{\sigma}(\rho)
    =
    \sum_{\eta\neq\sigma}
    \lambda_{\eta\sigma}\rho_{\eta}
    -
    \sum_{\ell\neq\sigma}
    \lambda_{\sigma\ell}\rho_{\sigma},
    \qquad
    \lambda_{mn}\geq 0.
    \label{eq:conservative_reaction_structure_cn}
\end{equation}
Here, \(\lambda_{mn}\) denotes the transition rate from behavioral phase \(m\) to behavioral phase \(n\) per unit time~\cite{zhu2007asymptotic,zhu2009strong}. In Eq.~\eqref{eq:conservative_reaction_structure_cn}, the first term represents the density flowing into phase \(\sigma\) from other phases, and the second term represents the density flowing out of phase \(\sigma\) to other phases. Thus, \(R_{\sigma}(\rho)\) is the net inflow of phase \(\sigma\).

This structure further satisfies the global conservation condition~\cite{giovangigli2012multicomponent}
\begin{equation}
    \sum_{\sigma=1}^{M}R_{\sigma}(\rho)=0.
    \label{eq:reaction_sum_zero_cn}
\end{equation}
That is, the reaction term only changes the distribution of robots among behavioral phases, while preserving the total number of robots.
\end{assumption}

\begin{remark}
The conservative reaction structure is consistent with behavioral phase switching in robot swarms. Each transition from phase \(m\) to phase \(n\) appears as a loss term in the source phase and a gain term in the target phase. When all behavioral phases are summed, these paired gain and loss terms cancel out. This structure is equivalent to the generator-matrix form of a continuous-time Markov jump process~\cite{seabrook2023tutorial,zhu2007asymptotic,zhu2009strong}, where \(\lambda_{mn}\) controls the transition intensity between behavioral phases rather than the creation or removal of robots.
\end{remark}

\begin{theorem}[Local conservation relation and global swarm mass conservation~\cite{giovangigli2012multicomponent}]
\label{thm:local_and_global_mass_conservation_cn}
Consider a multi-stage swarm density field defined on a bounded task domain \(\Omega\subset\mathbb{R}^{d}\):
\[
    \rho(x,t)
    =
    [\rho_1(x,t),\rho_2(x,t),\ldots,\rho_M(x,t)]^{\top}.
\]
If each behavioral phase \(\sigma=1,\ldots,M\) satisfies a local conservation relation over any control volume \(V\subset\Omega\), then its local differential form is
\begin{equation}
    \frac{\partial \rho_{\sigma}}{\partial t}
    +
    \nabla\cdot J_{\sigma}
    =
    R_{\sigma}(\rho),
    \qquad
    \sigma=1,\ldots,M.
    \label{eq:local_conservation_cn}
\end{equation}
Moreover, if the spatial flux takes the ADR flux decomposition in Eq.~\eqref{eq:total_flux_for_boundary_cn}, then Eq.~\eqref{eq:local_conservation_cn} is equivalent to the multi-stage Macro-ADR system
\begin{equation}
    \frac{\partial \rho_{\sigma}}{\partial t}
    +
    \nabla\cdot
    \left(
    u_{\sigma}\rho_{\sigma}
    \right)
    =
    \nabla\cdot
    \left(
    D_{\sigma}\nabla\rho_{\sigma}
    \right)
    +
    R_{\sigma}(\rho).
    \label{eq:adr_for_mass_conservation_cn}
\end{equation}
If the system satisfies the zero-flux boundary condition in Assumption~\ref{assump:zero_flux} and the reaction term satisfies the conservative reaction structure in Assumption~\ref{assump:cons_react}, then the total swarm mass is conserved:
\begin{equation}
    \frac{d}{dt}
    \sum_{\sigma=1}^{M}
    \int_{\Omega}
    \rho_{\sigma}(x,t)\,dx
    =
    0.
    \label{eq:mass_conservation_prop_cn}
\end{equation}
\end{theorem}

\begin{proof}[\textbf{Proof of Theorem~\ref{thm:local_and_global_mass_conservation_cn}:}]
We first prove the \textit{local conservation relation}. For any control volume \(V\subset\Omega\), the local mass balance of behavioral phase \(\sigma\) can be written as~\cite{sun2021lie,giovangigli2012multicomponent}
\begin{equation}
    \frac{d}{dt}
    \int_V \rho_{\sigma}(x,t)\,dx
    =
    -
    \int_{\partial V}
    J_{\sigma}\cdot\mathbf{n}\,dS
    +
    \int_V
    R_{\sigma}(\rho)\,dx.
    \label{eq:integral_local_balance_cn}
\end{equation}
By the divergence theorem~\cite{chandrasekhar1943stochastic,cosner2014reaction},
\begin{equation}
    \int_{\partial V}
    J_{\sigma}\cdot\mathbf{n}\,dS
    =
    \int_V
    \nabla\cdot J_{\sigma}\,dx.
\end{equation}
Therefore,
\begin{equation}
    \int_V
    \left(
    \frac{\partial \rho_{\sigma}}{\partial t}
    +
    \nabla\cdot J_{\sigma}
    -
    R_{\sigma}(\rho)
    \right)dx
    =
    0.
\end{equation}
Since the control volume \(V\) is arbitrary, we obtain
\begin{equation}\label{eq:S15}
    \frac{\partial \rho_{\sigma}}{\partial t}
    +
    \nabla\cdot J_{\sigma}
    =
    R_{\sigma}(\rho),
\end{equation}
which is Eq.~\eqref{eq:local_conservation_cn}.

Substituting the ADR flux definition in Eq.~\eqref{eq:total_flux_for_boundary_cn} into Eq.~\eqref{eq:local_conservation_cn} gives
\begin{equation}
    \frac{\partial \rho_{\sigma}}{\partial t}
    +
    \nabla\cdot
    \left(
    u_{\sigma}\rho_{\sigma}
    -
    D_{\sigma}\nabla\rho_{\sigma}
    \right)
    =
    R_{\sigma}(\rho).
\end{equation}
Rearranging terms yields
\begin{equation}
    \frac{\partial \rho_{\sigma}}{\partial t}
    +
    \nabla\cdot
    \left(
    u_{\sigma}\rho_{\sigma}
    \right)
    =
    \nabla\cdot
    \left(
    D_{\sigma}\nabla\rho_{\sigma}
    \right)
    +
    R_{\sigma}(\rho),
\end{equation}
which is Eq.~\eqref{eq:adr_for_mass_conservation_cn}.

We next prove global swarm mass conservation. Define the total swarm mass as
\begin{equation}
    \mathcal{M}(t)
    =
    \sum_{\sigma=1}^{M}
    \int_{\Omega}
    \rho_{\sigma}(x,t)\,dx.
    \label{eq:total_mass_for_proof_cn}
\end{equation}
Taking the time derivative and using the local conservation form in Eq.~\eqref{eq:local_conservation_cn}, we have
\begin{equation}\label{eq:mass_derivative_step_cn}
    \frac{d\mathcal{M}}{dt}
    =
    \sum_{\sigma=1}^{M}
    \int_{\Omega}
    \frac{\partial \rho_{\sigma}}{\partial t}
    \,dx
    =
    -
    \sum_{\sigma=1}^{M}
    \int_{\Omega}
    \nabla\cdot J_{\sigma}\,dx
    +
    \sum_{\sigma=1}^{M}
    \int_{\Omega}
    R_{\sigma}(\rho)\,dx.
\end{equation}
By the divergence theorem and the zero-flux boundary condition in Assumption~\ref{assump:zero_flux}, for each behavioral phase \(\sigma\),
\begin{equation}
    \int_{\Omega}
    \nabla\cdot J_{\sigma}\,dx
    =
    \int_{\partial\Omega}
    \mathbf{n}\cdot J_{\sigma}\,dS
    =
    0.
    \label{eq:boundary_flux_zero_cn}
\end{equation}
On the other hand, by the conservative reaction structure in Assumption~\ref{assump:cons_react}, the reaction terms only redistribute mass among behavioral phases and cancel when summed over all phases, as stated in Eq.~\eqref{eq:reaction_sum_zero_cn}. Hence,
\begin{equation}
    \sum_{\sigma=1}^{M}
    \int_{\Omega}
    R_{\sigma}(\rho)\,dx
    =
    0.
    \label{eq:reaction_sum_zero_cn_jifen}
\end{equation}
Substituting Eqs.~\eqref{eq:boundary_flux_zero_cn} and \eqref{eq:reaction_sum_zero_cn_jifen} into Eq.~\eqref{eq:mass_derivative_step_cn} gives
\begin{equation}
    \frac{d\mathcal{M}}{dt}=0.
\end{equation}
Therefore,
\begin{equation}
    \mathcal{M}(t)=\mathcal{M}(0),
    \qquad
    t\in[0,T].
\end{equation}
\end{proof}

\begin{remark}[Physical interpretation]
Theorem~\ref{thm:local_and_global_mass_conservation_cn} provides the conservation basis of the multi-stage Macro-ADR model. The local conservation relation (Eq.~\eqref{eq:S15}) shows that the density change of each behavioral phase is jointly determined by the spatial flux \(J_{\sigma}\) and the phase-transition term \(R_{\sigma}(\rho)\). The zero-flux boundary condition rules out non-physical loss of robots through the task-domain boundary, while the conservative reaction structure ensures that behavioral phase transitions only redistribute robots among phases without changing the total number of robots. Therefore, the total swarm mass remains finite and conserved.
\end{remark}

\clearpage

\subsubsection{The physical explanation of the advection term}

The advection term describes deterministic robot migration induced by task objectives, environmental constraints and inter-agent interactions. In PhySwarm, robots do not directly execute unconstrained neural actions. Instead, their motion directions are generated from weighted combinations of physically interpretable field bases. For robot \(i\), the microscopic advection velocity is written as
\begin{equation}
    v_{\mathrm{adv},i}(t)
    =
    \sum_{r=1}^{K_b}
    \omega_{i,r}(t)b_r(x_i,t),
    \qquad
    b_r(x,t)=-\nabla\Phi_r(x,t),
    \label{eq:micro_advection_velocity}
\end{equation}
where \(\Phi_r(x,t)\) is the \(r\)-th potential field, \(b_r(x,t)\) is the velocity basis induced by its negative gradient, and \(K_b\) is the number of field bases. These field bases encode reusable physical cues, including target attraction, obstacle repulsion, boundary constraints, trail or information guidance, formation shaping, cohesion, separation, alignment and connectivity maintenance. The weights \(\omega_{i,r}(t)\) are predicted by the Neural-Physics Controller and dynamically regulate the contribution of each field basis to the current motion. Thus, advection in PhySwarm is not an arbitrary velocity output, but a phase- and context-dependent composition of physically meaningful transport directions. Detailed examples of the field bases are provided in Supplementary Materials Sec.~\ref{fields_example}.

At the macroscopic level, the advection velocity field associated with behavioral phase \(\sigma\) is denoted by \(u_{\sigma}(x,t)\). If the density of this phase is \(\rho_{\sigma}(x,t)\), the advective flux generated by deterministic migration is
\begin{equation}
    J_{\mathrm{adv},\sigma}
    =
    u_{\sigma}\rho_{\sigma}.
    \label{eq:advective_flux}
\end{equation}
This flux represents the number of phase-\(\sigma\) robots crossing a unit area per unit time due to deterministic motion. According to the local conservation relation in Theorem~\ref{thm:local_and_global_mass_conservation_cn}, the temporal change of density is given by the negative divergence of the outgoing flux:
\begin{equation}
    \frac{\partial \rho_{\sigma}}{\partial t}
    +
    \nabla\cdot J_{\mathrm{adv},\sigma}
    =
    0.
\end{equation}
Substituting Eq.~\eqref{eq:advective_flux} into this relation gives
\begin{equation}
    \frac{\partial \rho_{\sigma}}{\partial t}
    =
    -
    \nabla\cdot
    \left(
    u_{\sigma}\rho_{\sigma}
    \right).
    \label{eq:advection_term}
\end{equation}
Eq.~\eqref{eq:advection_term} gives the origin of the advection term in the Macro-ADR equation. Physically, the advection term describes deterministic transport of robots under task objectives, environmental constraints and inter-agent interactions. If more robots are transported out of a local region than into it, the density in that region decreases; conversely, if more robots enter the region than leave it, the local density increases. Therefore, the advection term mainly accounts for directed swarm transport, goal-driven migration, collective formation movement and other task-oriented collective motions.

\begin{figure}[htbp]
    \centering
    \includegraphics[width=\linewidth]{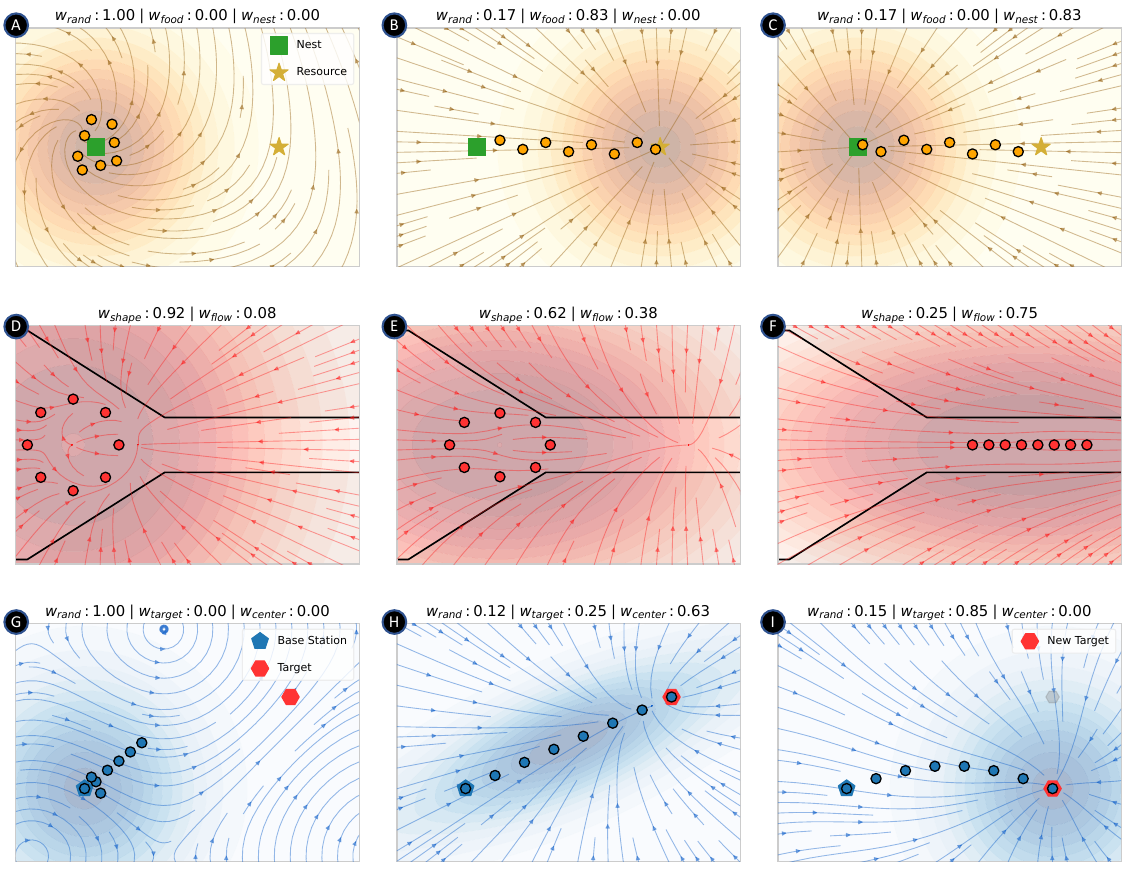} 
    \caption{
\textbf{Representative advection-field topologies across tasks and behavioral phases.}
\textbf{A--C,} Trail-Guided Swarm Foraging: the dominant advection mechanism shifts from random exploration (\textbf{A}) to food attraction (\textbf{B}) and nest-directed homing (\textbf{C}). 
\textbf{D--F,} Formation-Reconfigurable Swarm Navigation: the field changes from isotropic gathering (\textbf{D}) to corridor-induced squeezing (\textbf{E}) and anisotropic transit (\textbf{F}). 
\textbf{G--I,} Role-Adaptive Swarm Search and Rescue: the field changes from random locating (\textbf{G}) to base--target relay bridging (\textbf{H}) and target tracking after reconfiguration (\textbf{I}). 
Background contours indicate scalar potential landscapes generated by weighted physical field bases, and streamlines indicate the deterministic velocity fields induced by the negative potential gradients. The weights shown above each panel are representative examples of the advection parameters learned or selected by the Neural-Physics Controller; full field definitions and scenario-specific parameter settings are provided in Sec.~\ref{sec:trail_guided_swarm_foraging_cn} of the supplementary materials.
}
    \label{fig:adv}
\end{figure}

Fig.~\ref{fig:adv} illustrates the field structures of the advection term in three representative swarm tasks and how these structures are modulated across behavioral phases. The purpose of this figure is not to present a complete solution process of the Macro-ADR system, but to provide an intuitive view of how, in the PhySwarm, the advection term generates phase-dependent deterministic transport directions through weighted combinations of physical potential-field bases. The learned NPC weights \(\omega\) determine the relative contribution of these field bases at different behavioral stages.

For \textit{Trail-Guided Swarm Foraging}, the advection field is mainly composed of a random-exploration field, a food-attraction field and a nest-attraction field. In Phase I, a relatively high \(\omega_{\mathrm{rand}}\) produces weakly directional exploratory streamlines, supporting spatial search in an initially unknown area. In Phase II, \(\omega_{\mathrm{food}}\) increases, and the potential-field topology becomes oriented towards the food location. The streamlines gradually converge near the resource site, inducing aggregation around the food region and enabling subsequent pick-up events. In Phase III, \(\omega_{\mathrm{nest}}\) becomes the dominant advection component, reorienting the velocity field from the food region towards the nest and guiding food-carrying robots back to the nest. Thus, exploration, resource aggregation and homeward transport in the foraging process are represented by phase-dependent modulation of a shared advection-field dictionary, rather than by independent handcrafted state controllers.

For \textit{Formation-Reconfigurable Swarm Navigation}, the advection field is primarily shaped by a morphology-preserving field and a forward-flow field. In Phase I, a high \(\omega_{\mathrm{shape}}\) and a low \(\omega_{\mathrm{flow}}\) support an approximately isotropic aggregation or circular formation in open space. In Phase II, after the robots perceive the corridor constraint, the contribution of the forward-flow field increases, while the morphology field gradually shifts from circular formation maintenance to corridor-compatible deformation. The streamlines are consequently reorganized along the corridor direction. In Phase III, \(\omega_{\mathrm{flow}}\) is further strengthened, and the swarm is guided into an anisotropic arrangement suitable for narrow-corridor traversal. This process shows that formation reconfiguration can be interpreted as a continuous trade-off between morphology-preserving potentials and environment-guided flow fields. The swarm does not switch abruptly between fixed formations, but adapts its morphology through advection-field reconstruction in constrained space.

For \textit{Role-Adaptive Swarm Search and Rescue}, the advection field is composed of a random-exploration field, a target-attraction field and a relay or centre-constraining field. In Phase I, before the target is detected, \(\omega_{\mathrm{rand}}\) dominates and the swarm maintains a dispersed search pattern. In Phase II, after target information is locally sensed and propagated, \(\omega_{\mathrm{center}}\) and \(\omega_{\mathrm{target}}\) are activated. The velocity field begins to form a connective structure between the base and the target, guiding a subset of robots to organize along the relay axis and supporting the establishment of a communication or response chain. In Phase III, when the target location changes or dynamic tracking is required, the target-attraction field reorients the streamlines towards the new target region, allowing the swarm to adjust its spatial organization and re-establish response or relay structures under the updated task geometry. These results indicate that role adaptation does not rely on a predefined fixed role assignment. Instead, spatial differentiation among searchers, responders and relays is induced at the density-field level through dynamic weighting of the exploration, target and relay-related fields.

Overall, Fig.~\ref{fig:adv} explains the role of the advection term in PhySwarm from the perspective of potential-field topology. The advection term converts high-level task semantics, including task objectives, environmental constraints, swarm morphology and communication requirements, into executable velocity fields. Different tasks and behavioral phases correspond to different combinations of potential-field weights, which are learned by the NPC and jointly determine the multi-stage evolution of swarm density together with the diffusion and reaction terms. Detailed field-function definitions, weight constraints, diffusion constructions and reaction phase graphs are provided in the modeling guidelines and scenario-specific implementation details in Sec.~\ref{sec:general_modeling_procedure_cn} and \ref{sec:scenario_specific_details_cn} of the supplementary materials.

\clearpage

\subsubsection{The physical explanation of the diffusion term}

The diffusion term describes the tendency of the swarm to expand from high-density regions towards low-density regions. In the proposed framework, diffusion is not treated as aimless microscopic equivalent deterministic motion, but as an active response of robots to local density gradients. Specifically, a local density field \(\hat{\rho}(x)\) is estimated using kernel density estimation, and the microscopic density-gradient compensation velocity of robot \(i\) is defined as
\begin{equation}
    v_{\mathrm{diff},i}
    =
    -
    \frac{D}{\hat{\rho}(x_i)+\varepsilon}
    \nabla\hat{\rho}(x_i),
    \label{eq:micro_diffusion_velocity}
\end{equation}
where \(D\geq 0\) is the diffusion coefficient predicted by the Neural-Physics Controller, and \(\varepsilon>0\) is a small constant introduced to avoid numerical singularities in low-density regions.

When the estimated density \(\hat{\rho}\) sufficiently approximates the true continuum density \(\rho\), Eq.~\eqref{eq:micro_diffusion_velocity} induces the following macroscopic diffusive flux:
\begin{equation}
    J_{\mathrm{diff}}
    =
    \rho v_{\mathrm{diff}}
    =
    \rho
    \left(
    -
    \frac{D}{\rho+\varepsilon}
    \nabla\rho
    \right)
    \approx
    -D\nabla\rho.
\end{equation}
When \(\varepsilon\) is sufficiently small and \(\rho\) is bounded away from zero, this expression reduces to the standard Fickian diffusion law~\cite{philibert2006one}:
\begin{equation}
    J_{\mathrm{diff},\sigma}
    =
    -
    D_{\sigma}\nabla\rho_{\sigma}.
    \label{eq:fick_flux}
\end{equation}
According to the local conservation relation in Theorem~\ref{thm:local_and_global_mass_conservation_cn}, the diffusive transport satisfies
\(\frac{\partial\rho_{\sigma}}{\partial t}+\nabla\cdot J_{\mathrm{diff},\sigma}=0\).
Substituting Eq.~\eqref{eq:fick_flux} gives
\begin{equation}
    \frac{\partial \rho_{\sigma}}{\partial t}
    =
    \nabla\cdot
    \left(
    D_{\sigma}\nabla\rho_{\sigma}
    \right).
    \label{eq:diffusion_term_general}
\end{equation}
If \(D_{\sigma}\) is spatially uniform, this further reduces to
\begin{equation}
    \frac{\partial \rho_{\sigma}}{\partial t}
    =
    D_{\sigma}\Delta\rho_{\sigma}.
    \label{eq:diffusion_term_constant}
\end{equation}

Thus, the diffusion term in the Macro-ADR equation originates from a density-gradient-driven diffusive flux. Its primary physical role is to regulate the spatial distribution of the swarm. A larger \(D_{\sigma}\) promotes rapid coverage and exploration in open areas, whereas a smaller \(D_{\sigma}\) helps preserve compactness and topological stability in narrow or constrained regions. In addition, the diffusion term suppresses excessive local aggregation and thereby contributes to collision avoidance, pressure release and density equalization.

\begin{figure}[h!]
    \centering
    \includegraphics[width=\linewidth]{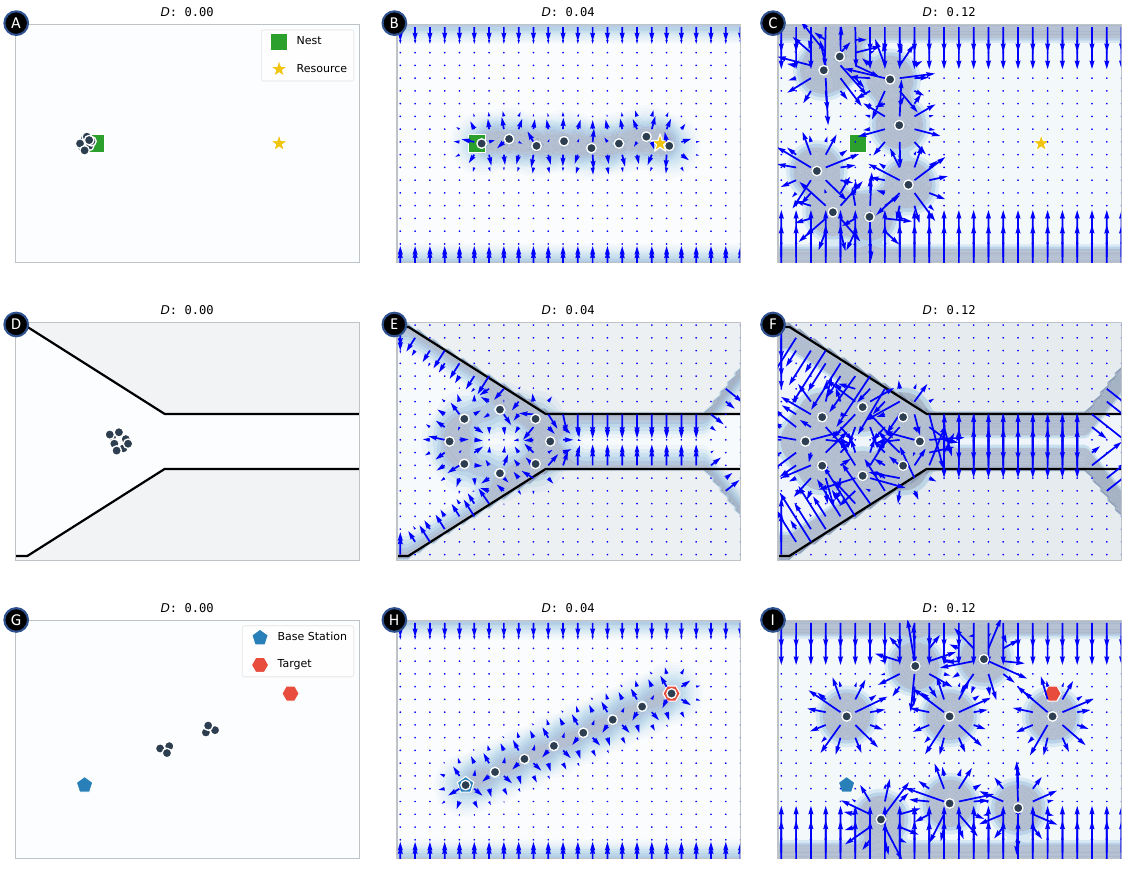} 
    \caption{
\textbf{Effect of diffusion strength on density-gradient compensation across task scenarios.}
\textbf{A--C,} Trail-Guided Swarm Foraging under low diffusion (\textbf{A}, \(D=0.00\)), balanced diffusion (\textbf{B}, \(D=0.04\)) and high diffusion (\textbf{C}, \(D=0.12\)). 
\textbf{D--F,} Formation-Reconfigurable Swarm Navigation under the same three diffusion levels. 
\textbf{G--I,} Role-Adaptive Swarm Search and Rescue under the same three diffusion levels. 
Blue arrows indicate the deterministic density-gradient compensation velocity induced by the diffusion term, and shaded regions indicate areas where the diffusion response is stronger. 
At low diffusion, local density regulation is nearly inactive, leading to over-aggregation and congestion around task-relevant regions such as the nest, base or local formation centre. 
At balanced diffusion, the swarm maintains more uniform spacing, improving density equalization, local collision avoidance and structural stability. 
At high diffusion, the stronger density-gradient response promotes rapid spatial coverage and exploration, but may also weaken formation compactness or relay-chain continuity in constrained environments. 
These examples illustrate the role of \(D\) in controlling the trade-off between spatial cohesion and exploratory dispersion in PhySwarm.
Full field definitions and scenario-specific parameter settings are provided in Sec.~\ref{sec:formation_reconfigurable_swarm_navigation_cn} of the supplementary materials.
}
    \label{fig:phydiff}
\end{figure}

Fig.~\ref{fig:phydiff} illustrates how the diffusion coefficient \(D\) affects swarm spatial distribution and density-gradient compensation velocity. The purpose of this figure is to show that the diffusion term does not directly encode task objectives. Instead, it regulates how robots respond to local density gradients and spatial constraints, thereby influencing dispersion, safety margins and topological stability.

Under low-diffusion conditions (Fig.~\ref{fig:phydiff}A,D,G), density-gradient compensation is largely suppressed, and the system lacks an effective mechanism for mitigating local overcrowding. As a result, in the foraging and search-and-rescue tasks, robots tend to form compact clusters near the nest, base or local information regions. In the formation-reconfiguration task, the swarm has limited ability to adjust its spatial distribution near the corridor entrance, making it difficult to maintain formation while preserving a safe distance from the boundaries. This regime corresponds to under-diffusion and reflects the congestion that may arise when the swarm relies mainly on advection-based attraction, target guidance or morphology constraints.

Under balanced-diffusion conditions (Fig.~\ref{fig:phydiff}B,E,H), the diffusion term provides moderate density equalization and spatial pressure release. More uniform local repulsive responses emerge among robots and between robots and environmental boundaries, allowing the swarm to avoid excessive crowding while preserving task-relevant structures. In the foraging task, this diffusion level helps maintain appropriate spacing along transport paths. In the formation-reconfiguration task, the swarm can remain compact under corridor constraints without becoming over-compressed. In the search-and-rescue task, diffusion compensation supports a more evenly distributed relay chain. These results indicate that moderate diffusion provides an effective trade-off between spatial cohesion and coverage or dispersion capability.

Under high-diffusion conditions (Fig.~\ref{fig:phydiff}C,F,I), the response to density gradients becomes substantially stronger, and the swarm exhibits a greater tendency to disperse. In open regions, enhanced diffusion can enlarge the search coverage. However, in narrow corridors, near boundaries or in tasks that require maintaining a communication chain, excessive diffusion may weaken the constraints imposed by the advection term on targets, formations or relay structures. This can lead to over-dispersion and degrade formation stability, corridor-traversal efficiency or communication-chain continuity. Therefore, stronger diffusion does not necessarily improve task performance, and \(D\) should be modulated according to the task stage and environmental constraints.

Overall, Fig.~\ref{fig:phydiff} shows that the diffusion term in PhySwarm acts as a mechanism for spatial regulation and safety constraint. A smaller \(D\) helps maintain cohesion but may lead to crowding; a larger \(D\) improves coverage and exploration but may disrupt formation or communication topology; and an intermediate \(D\) balances density equalization, local collision avoidance and task-structure preservation. By learning phase-dependent diffusion coefficients \(D(t)\), the NPC enables the swarm to dynamically adjust its spatial distribution across task stages such as exploration, corridor traversal, transport and relay formation. The physical origin of the diffusion term, its microscopic density-gradient compensation form and its correspondence to the ADR flux are provided in the physical interpretation and modeling guidelines in the Supplementary Materials.

\clearpage

\subsubsection{The physical explanation of the reaction term}

\begin{figure}[t]
    \centering
    \includegraphics[width=\linewidth]{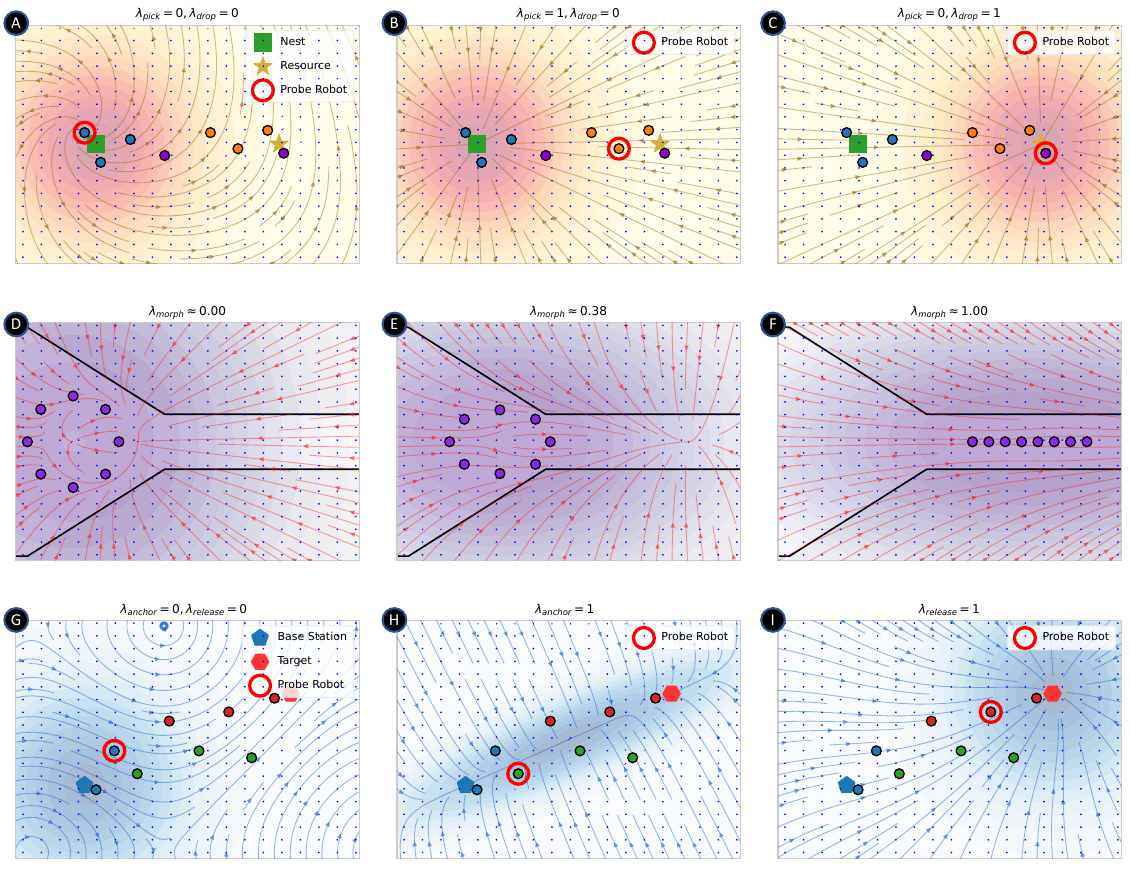} 
    \caption{
\textbf{Representative reaction-regulated state fields across tasks and behavioral phases.}
\textbf{A--C,} Trail-Guided Swarm Foraging. 
\textbf{A,} In the searcher state, neither pick-up nor drop-off transitions are activated, and robots mainly respond to exploration-related fields. 
\textbf{B,} When the pick-up transition is activated, robots enter the carrier state and the effective field redirects motion towards the nest. 
\textbf{C,} When the drop-off transition is activated, robots leave the carrier state and re-enter an approach or re-search state directed by food- or task-information-related fields. 
\textbf{D--F,} Formation-Reconfigurable Swarm Navigation. 
\textbf{D,} With \(\lambda_{\mathrm{morph}}\approx0\), the swarm maintains an approximately circular formation. 
\textbf{E,} As \(\lambda_{\mathrm{morph}}\) increases, the effective morphology field deforms towards an elliptical, corridor-adaptive shape. 
\textbf{F,} With \(\lambda_{\mathrm{morph}}\approx1\), the swarm adopts a line-like formation compatible with narrow-corridor traversal. 
\textbf{G--I,} Role-Adaptive Swarm Search and Rescue. 
\textbf{G,} In the searcher state, relay and release transitions are inactive and robots remain in a dispersed locating mode. 
\textbf{H,} Activation of \(\lambda_{\mathrm{anchor}}\) induces an anchored relay state, organizing robots along the base--target direction. 
\textbf{I,} Activation of \(\lambda_{\mathrm{release}}\) induces a responder phase, redirecting robots towards the target region. 
In all panels, the background contours display the effective phase-conditioned field perceived by the designated Probe Robot, and streamlines indicate the corresponding deterministic velocity directions. The reaction rates \(\lambda\) do not directly generate spatial potentials; instead, they regulate behavioral phase or role transitions, thereby switching robots among phase-conditioned field responses. Full reaction matrices, transition graphs and task-specific activation conditions are provided in Sec.~\ref{sec:formation_reconfigurable_swarm_navigation_cn} of the supplementary materials.
}
    \label{fig:phyrea}
\end{figure}

The reaction term describes transitions among behavioral phases. In the context of swarm emergent behavior modeling, it should be emphasized that the reaction term does not represent material generation or annihilation in the chemical sense. Instead, it represents mass transfer among different functional states or behavioral phases of the same robot population. In other words, the total number of robots is conserved, while the proportion of robots assigned to different phases may vary over time.

Let
\(\lambda_{mn}(x,t)\geq 0,\ m\neq n,\)
denote the local transition rate from behavioral phase \(m\) to behavioral phase \(n\) per unit time. At the microscopic level, the transition of robot \(i\) from phase \(\sigma_m\) to phase \(\sigma_n\) can be modeled as an environmentally triggered stochastic jump~\cite{seabrook2023tutorial,zhu2007asymptotic,zhu2009strong}:
\begin{equation}
    \mathbb{P}
    \left(
    \sigma_i(t+\Delta t)=\sigma_n
    \mid
    \sigma_i(t)=\sigma_m
    \right)
    =
    \lambda_{mn}(x_i,t)
    \chi_{mn}(x_i,t)
    \Delta t
    +
    o(\Delta t),
    \label{eq:micro_reaction_probability}
\end{equation}
where \(\chi_{mn}(x_i,t)\) is a task- or environment-dependent activation function.

At the macroscopic level, transitions among phases are represented by a reaction matrix \(\Lambda(\lambda)\). Using the column-vector convention,
\(\rho=[\rho_1,\rho_2,\ldots,\rho_M]^{T}\),
the reaction term can be written as
\begin{equation}
    R(\rho)
    =
    \Lambda(\lambda)\rho .
    \label{eq:reaction_vector}
\end{equation}
The entries of \(\Lambda(\lambda)\) are defined as
\begin{equation}
    \Lambda_{\sigma\eta}(\lambda)
    =
    \left\{
    \begin{array}{ll}
        \lambda_{\eta\sigma}, 
        & \sigma \neq \eta, \\[4pt]
        -\displaystyle\sum_{\ell \neq \sigma}\lambda_{\sigma\ell},
        & \sigma = \eta .
    \end{array}
    \right.
    \label{eq:reaction_matrix}
\end{equation}
Thus, the reaction term of phase \(\sigma\) is
\begin{equation}
    R_{\sigma}(\rho)
    =
    \left[
    \Lambda(\lambda)\rho
    \right]_{\sigma}.
\end{equation}
Expanding this component gives
\begin{equation}
    R_{\sigma}(\rho)
    =
    \sum_{\eta\neq\sigma}
    \lambda_{\eta\sigma}\rho_{\eta}
    -
    \sum_{\ell\neq\sigma}
    \lambda_{\sigma\ell}\rho_{\sigma}.
    \label{eq:reaction_component}
\end{equation}
Here,
\(\sum_{\eta\neq\sigma}\lambda_{\eta\sigma}\rho_{\eta}\)
denotes the density flowing into phase \(\sigma\) from other phases, whereas
\(\sum_{\ell\neq\sigma}\lambda_{\sigma\ell}\rho_{\sigma}\)
denotes the density flowing out of phase \(\sigma\) to other phases. Therefore, \(R_{\sigma}\) represents the net density change of phase \(\sigma\) induced by behavioral phase transitions. For a transition between two behavioral phases \(m\) and \(n\), the reaction dynamics can be written as
\begin{equation}
\begin{aligned}
\left[
\begin{array}{c}
R_m \\
R_n
\end{array}
\right]
&=
\left[
\begin{array}{cc}
-\lambda_{mn} & \lambda_{nm} \\
 \lambda_{mn} & -\lambda_{nm}
\end{array}
\right]
\left[
\begin{array}{c}
\rho_m \\
\rho_n
\end{array}
\right].
\end{aligned}
\label{eq:two_state_reaction}
\end{equation}
This expression shows that density lost from one phase is gained by the other phase with the same amount. Hence, the reaction term does not create or remove robots; it redistributes swarm density among behavioral phases. Its primary physical role is to encode multi-stage behavioral switching, such as transitions from exploration to aggregation, from formation maintenance to collective navigation, or from search to rescue response.

Fig.~\ref{fig:phyrea} illustrates how the reaction term regulates behavioral-phase and role-state transitions across different task stages. It is important to note that the reaction term \(R(\rho)\) does not directly generate a spatial potential field. Instead, it modulates transitions among behavioral phases or functional roles through the reaction rates \(\lambda\). As a result, robots switch to different phase-conditioned fields, which indirectly changes the advection, diffusion and local control mechanisms to which they respond.Because robots in the swarm simultaneously occupy different behavioral phases, they perceive different physical fields. For visual clarity, each panel visualizes only the scalar potential landscape perceived by the designated Probe Robot. Other robots are guided by their respective phase-conditioned 
fields, which are omitted from the background to avoid overlapping contour layers.

In \textit{Trail-Guided Swarm Foraging}, the reaction term mainly regulates transitions among searchers, carriers and re-searching or approaching robots through the pick-up rate \(\lambda_{\mathrm{pick}}\) and the drop-off rate \(\lambda_{\mathrm{drop}}\). In the search stage (Fig.~\ref{fig:phyrea}A), \(\lambda_{\mathrm{pick}}=\lambda_{\mathrm{drop}}=0\), and robots have not triggered pick-up or drop-off reactions; they therefore mainly respond to exploration-related fields. In the carrying stage (Fig.~\ref{fig:phyrea}B), once the pick-up condition is triggered, \(\lambda_{\mathrm{pick}}\) is activated and robots enter the carrier state. Their effective field is then reoriented towards the nest, supporting food-carrying return motion. In the post-drop-off stage (Fig.~\ref{fig:phyrea}C), \(\lambda_{\mathrm{drop}}\) is activated, robots leave the carrier state and re-enter a food-approach or re-search state. Thus, exploration, pick-up, homeward transport and re-search in the foraging task are represented as phase-conditioned field switching regulated by reaction rates, rather than by independent handcrafted state controllers.

In \textit{Formation-Reconfigurable Swarm Navigation}, the reaction term represents the continuous transition of formation topology from a circular arrangement to a line-like arrangement through the morphology-transition progress \(\lambda_{\mathrm{morph}}\). In the circular stage (Fig.~\ref{fig:phyrea}D), \(\lambda_{\mathrm{morph}}\approx0\), and the swarm mainly maintains an approximately isotropic circular or aggregated structure. As the swarm enters the corridor region, geometric constraints gradually activate morphology transition. In the elliptical transition stage (Fig.~\ref{fig:phyrea}E), \(\lambda_{\mathrm{morph}}\) increases, and the effective morphology field progressively shifts from circular formation maintenance to corridor-compatible deformation. In the line-like traversal stage (Fig.~\ref{fig:phyrea}F), \(\lambda_{\mathrm{morph}}\approx1\), and the formation field is further compressed into a line-like structure aligned with the corridor direction. This process shows that formation reconfiguration can be modeled as continuous phase-field deformation regulated by reaction progress, rather than as abrupt switching between discrete formation templates.

In \textit{Role-Adaptive Swarm Search and Rescue}, the reaction term regulates role transitions among searchers, relays and responders through \(\lambda_{\mathrm{anchor}}\) and \(\lambda_{\mathrm{release}}\). In the search stage (Fig.~\ref{fig:phyrea}G), \(\lambda_{\mathrm{anchor}}=\lambda_{\mathrm{release}}=0\), and robots remain in the searcher state, mainly responding to exploration- and base-related fields to expand the search region. In the relay-formation stage (Fig.~\ref{fig:phyrea}H), \(\lambda_{\mathrm{anchor}}\) is activated, and a subset of robots enters an anchored relay state. Their effective field forms a relay constraint along the Base--Target direction, guiding robots to organize into a communication chain. In the response stage (Fig.~\ref{fig:phyrea}I), \(\lambda_{\mathrm{release}}\) is activated, robots enter the responder state and respond more strongly to the target region, supporting aggregation and rescue response near the target. This result indicates that role differentiation in search and rescue can be interpreted as a reaction-regulated process triggered by task events and communication demands, rather than as a predefined fixed role allocation.

Overall, Fig.~\ref{fig:phyrea} shows the central role of the reaction term in PhySwarm. It converts task events, environmental triggers and role demands into transition rates among behavioral phases, and changes the physical mechanisms that robots respond to through phase-conditioned fields. In contrast to the advection term, which primarily governs directional spatial transport, and the diffusion term, which primarily regulates density distribution, the reaction term organizes multi-stage behavioral transitions along temporal and functional dimensions. Detailed definitions of the reaction matrix, conservative reaction structure, phase-transition graphs and scenario-specific triggering conditions are provided in the physical interpretation, theoretical analysis and scenario-specific implementation details in the Supplementary Materials.

\clearpage

\subsection{Fault tolerance analysis}
\label{sec:fault_tolerance}

We further evaluate the fault-tolerance capability of PhySwarm under partial robot failures. The purpose of this supplementary analysis is to examine whether the learned policy depends on a fixed set of individual robots, or whether the remaining active robots can redistribute task-relevant behavior after some agents become unavailable. Unlike the main experiments, which focus on nominal multi-stage behavior generation and task performance, the experiments in this section introduce robot failures during execution and observe whether the swarm can preserve the essential macroscopic organization required by each task.

The same trained policy is used without retraining or task-specific recovery rules. Failed robots are treated as unavailable agents and no longer contribute to motion, role execution or task completion. The remaining robots continue to infer the physical parameters \(P(t)=\{\omega(t),D(t),\lambda(t)\}\) from their local observations and execute the corresponding Micro-EDM dynamics. Therefore, robustness is expected to arise from the learned physical mechanisms themselves: advection fields redistribute task-directed motion, diffusion regulates local spacing after agent loss, and reaction terms reallocate behavioral phases or functional roles when the available swarm composition changes.

We test this setting in the three benchmark tasks. In \textit{Trail-Guided Swarm Foraging}, the failure perturbation reduces the number of robots available for exploration and transport. The evaluation focuses on whether the remaining agents can still maintain task-directed motion between the food and nest regions. In \textit{Formation-Reconfigurable Swarm Navigation}, failures perturb the formation membership during corridor traversal, and the evaluation focuses on whether the remaining agents can preserve a corridor-compatible spatial organization. In \textit{Role-Adaptive Swarm Search and Rescue}, failures remove part of the search, response or relay population, and the evaluation focuses on whether the swarm can reallocate roles and preserve a base--target communication or response structure.

These experiments are intended as representative perturbation tests rather than an exhaustive worst-case fault-tolerance guarantee. The results show that, across the three task types, the remaining active robots can reorganize their trajectories and maintain the main task-relevant collective structures after partial robot loss. This supports the interpretation that PhySwarm does not encode emergence as a fixed robot-level assignment, but as a density-field process regulated by learned advection, diffusion and reaction parameters.

\begin{figure}[htbp]
    \centering
    \includegraphics[width=\linewidth]{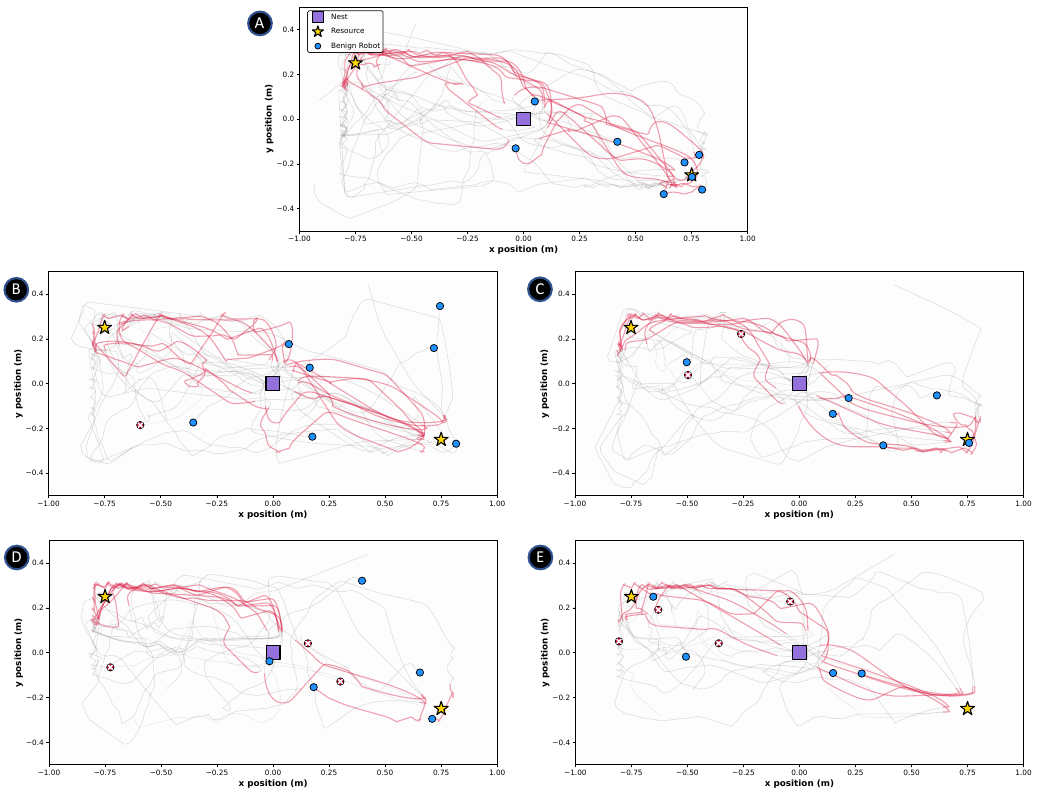} 
    \caption{
\textbf{Fault tolerance in Trail-Guided Swarm Foraging.}
\textbf{A--E,} Representative trajectories of the swarm under different robot-failure conditions in the foraging task.
\textbf{A,} Baseline trial without robot failure.
\textbf{B--E,} Trials with partial robot failures, where failed robots are marked by red crossed circles.
Grey curves show individual robot trajectories over the trial, whereas red curves highlight task-directed motion between the task-relevant sites.
Across the tested failure cases, the remaining benign robots adapt their trajectories and continue to maintain a coherent transport pattern, indicating that the learned PhySwarm policy does not depend on a fixed set of individual robots but can redistribute foraging behavior through the learned advection, diffusion and phase-transition mechanisms. (For more detailed experiment results, please refer to the supplementary Movie~S4.)
}
    \label{fig:forage_fails_0}
\end{figure}

\begin{figure}[htbp]
    \centering
    \includegraphics[width=\linewidth]{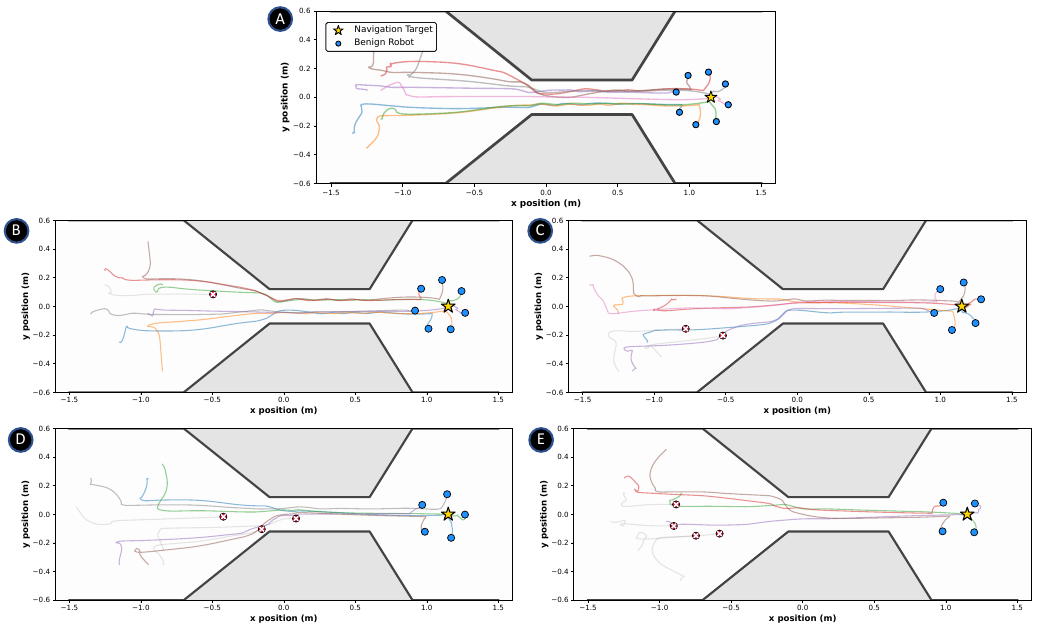} 
    \caption{
\textbf{Fault tolerance in Formation-Reconfigurable Swarm Navigation.}
\textbf{A--E,} Representative trajectories of the swarm passing through a constrained corridor under different robot-failure conditions.
\textbf{A,} Baseline trial without robot failure.
\textbf{B--E,} Trials with partial robot failures, marked by red crossed circles.
Grey regions indicate the geometric constraints of the corridor.
Colored trajectories show the motion of individual robots during formation reconfiguration and corridor traversal.
Despite the loss of several robots, the remaining robots preserve task-directed motion through the narrow corridor and maintain a corridor-compatible spatial organization.
These results provide supplementary evidence that the learned policy can reallocate the formation and transport structure under partial robot loss, rather than relying on a fixed formation membership. (For more detailed experiment results, please refer to the supplementary Movie~S5.)
}
    \label{fig:forage_fails_0}
\end{figure}

\begin{figure}[htbp]
    \centering
    \includegraphics[width=\linewidth]{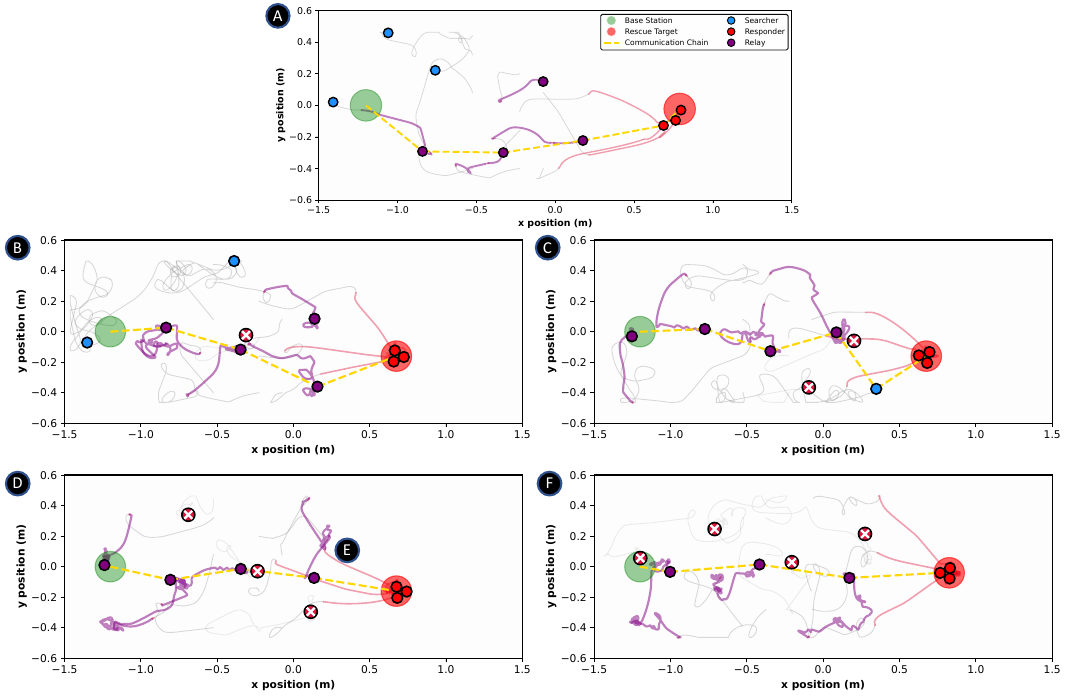} 
    \caption{
\textbf{Fault tolerance in Role-Adaptive Swarm Search and Rescue.}
\textbf{A--E,} Representative trajectories and role assignments under different robot-failure conditions in the search-and-rescue task.
\textbf{A,} Baseline trial without robot failure.
\textbf{B--E,} Trials with partial robot failures, where failed robots are marked by red crossed circles.
Grey curves show individual robot trajectories over the trial.
After robot failures occur, the remaining robots adapt their roles and spatial positions to preserve a base--target communication structure and support target response.
These supplementary results show that the role-adaptive policy can redistribute search, relay and response functions under partial robot loss, supporting the fault-tolerance analysis of PhySwarm. (For more detailed experiment results, please refer to the supplementary Movie~S6.)
}
    \label{fig:forage_fails_0}
\end{figure}

\clearpage

\subsection{Scalability analysis}
\label{sec:scalability}

We further evaluate the scalability of PhySwarm by increasing the number of robots in the three benchmark tasks. The purpose of this supplementary analysis is to examine whether the same PhySwarm formulation can preserve task-relevant macroscopic organization as the swarm population grows, rather than being restricted to a fixed number of agents. In contrast to the fault-tolerance tests, which remove robots during execution, the scalability experiments increase the number of active robots and assess whether the learned physical mechanisms can redistribute motion, density regulation and phase or role allocation over larger robot teams.

The scalability analysis focuses on the qualitative consistency of the emergent structures generated by the learned advection, diffusion and reaction parameters. In \textit{Trail-Guided Swarm Foraging}, increasing the swarm size introduces more robots for exploration and transport. The evaluation therefore focuses on whether the swarm can maintain coherent task-directed movement among food, nest and information-guided regions without collapsing into local congestion. In \textit{Formation-Reconfigurable Swarm Navigation}, larger populations impose stronger spatial constraints during corridor traversal. The evaluation focuses on whether the swarm can preserve a corridor-compatible formation and adjust its density distribution as the number of robots increases. In \textit{Role-Adaptive Swarm Search and Rescue}, larger teams provide more agents that can be allocated to search, response and relay functions. The evaluation focuses on whether the learned phase-transition and field-modulation mechanisms can support larger or multiple communication and response structures.

These experiments are intended to complement the main results by demonstrating population-level robustness of the learned physical representation. They do not constitute an asymptotic complexity guarantee, but they show that the same modeling interface can be instantiated across different swarm sizes without changing the underlying ADR formulation. Across the three tasks, the observed behaviors remain organized around the same task-relevant physical mechanisms: advection provides directed transport, diffusion regulates local density and spacing, and reaction terms redistribute behavioral phases or functional roles. This supports the interpretation that PhySwarm models swarm emergence at the density-field level, enabling the learned policy to generalize beyond a single fixed population size.

\begin{figure}[htbp]
    \centering
    \includegraphics[width=\linewidth]{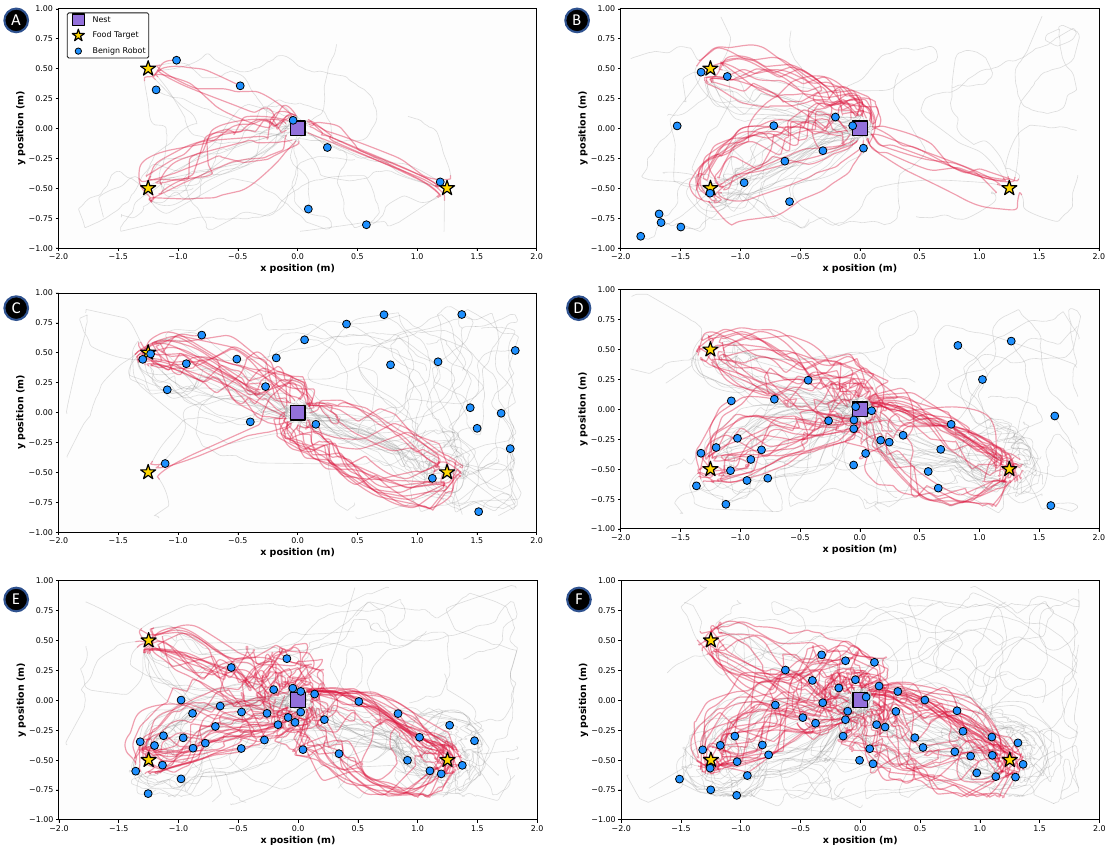} 
    \caption{
\textbf{Scalability in Trail-Guided Swarm Foraging.}
\textbf{A--F,} Representative trajectories of the foraging task under different swarm sizes.
Grey curves show individual robot trajectories over the trial, whereas red curves highlight the dominant task-directed transport paths formed by the swarm.
As the number of robots increases, the swarm maintains coherent exploration and transport patterns between task-relevant locations, indicating that the learned PhySwarm policy can scale the same advection--diffusion--reaction mechanisms to larger robot populations without requiring task-specific controller redesign. (For more detailed experiment results, please refer to the supplementary Movie~S7.)
}
    \label{fig:forage_expand_8}
\end{figure}

\begin{figure}[htbp]
    \centering
    \includegraphics[width=\linewidth]{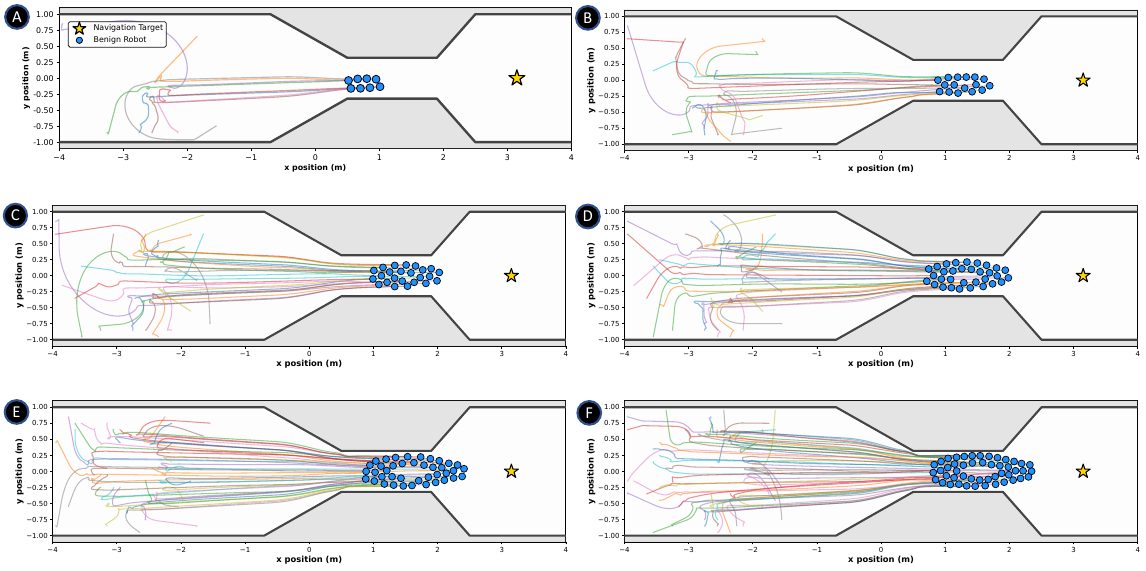} 
    \caption{
\textbf{Scalability in Formation-Reconfigurable Swarm Navigation.}
\textbf{A--F,} Representative trajectories of swarms with different population sizes navigating through a constrained corridor.
Grey regions indicate the geometric constraints of the corridor.
Colored curves show individual robot trajectories during formation reconfiguration and corridor traversal.
Across swarm sizes, the robots preserve a corridor-compatible spatial organization while moving through the narrow region, showing that the learned policy can extend the same morphology-regulation and density-adjustment mechanisms to larger formations.
These results supplement the main scalability analysis by demonstrating that the PhySwarm framework can modulate formation structure through physical parameters rather than relying on a fixed number of agents. (For more detailed experiment results, please refer to the supplementary Movie~S8.)
}
    \label{fig:forage_expand_8}
\end{figure}

\begin{figure}[htbp]
    \centering
    \includegraphics[width=\linewidth]{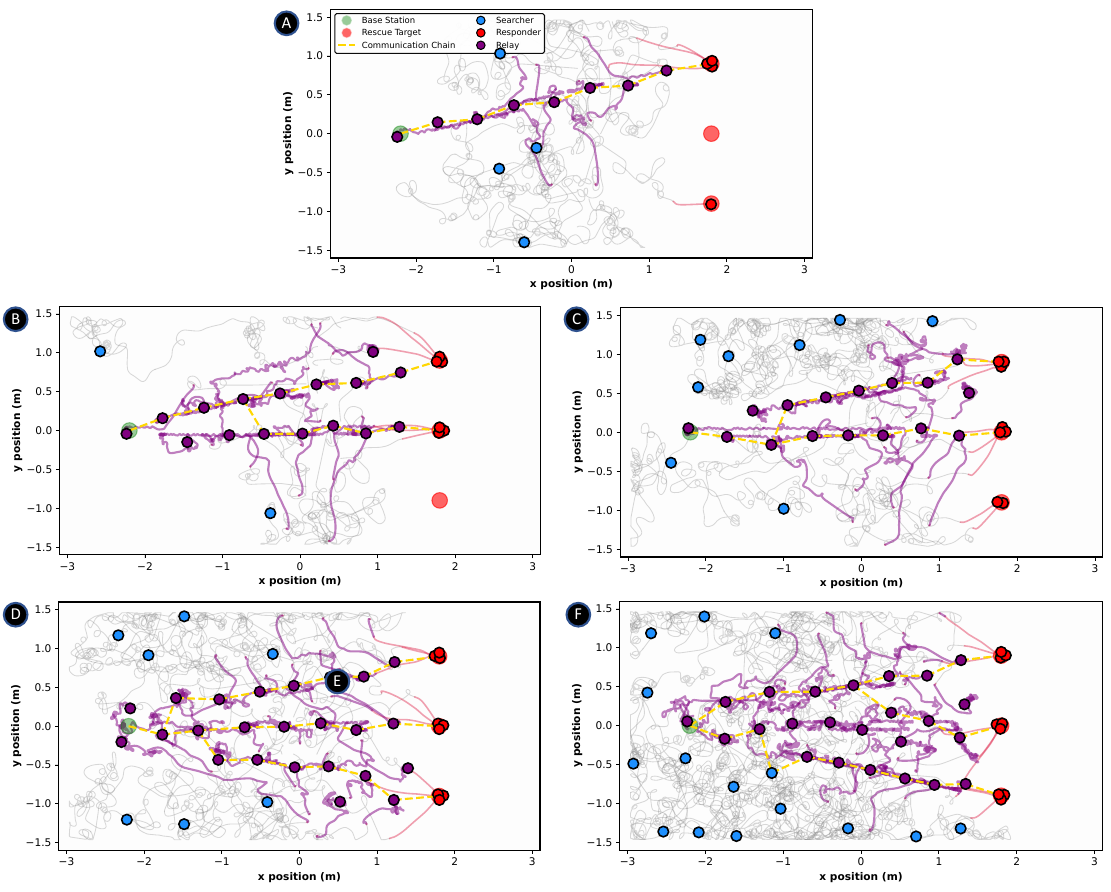} 
    \caption{
\textbf{Scalability in Role-Adaptive Swarm Search and Rescue.}
\textbf{A--E,} Representative trajectories and role allocations under different swarm sizes in the search-and-rescue task.
Grey curves show individual robot trajectories over the trial.
As the swarm size increases, additional robots are redistributed across search, relay and response roles, enabling the formation of larger or multiple base--target communication structures.
These results show that the learned role-adaptive policy can scale the same phase-transition and field-modulation mechanisms to larger robot teams while preserving task-directed search, relay and response organization. (For more detailed experiment results, please refer to the supplementary Movie~S9.)
}
    \label{fig:forage_expand_8}
\end{figure}

\clearpage

\subsection{Versatility analysis}
\label{sec:versatility}

The versatility analysis evaluates whether the same PhySwarm framework can maintain balanced performance across tasks with different physical requirements. The radar plots compare the proposed PhySwarm framework with a hand-crafted finite-state-machine (FSM) baseline and two fixed-parameter ablation variants under the same robot platform, sensing range, actuator limits and task-specific input cues. Each radar axis reports a normalized score, with larger values indicating better performance. The definitions and normalization procedures of the radar-plot metrics are provided in Supplementary Sec.~\ref{sec:versatility_metrics}, and the implementation details of the FSM baseline and fixed-parameter ablations are described in Supplementary Sec.~\ref{sec:baseline_ablation_methods}.

\textbf{Trail-Guided Swarm Foraging. }
Fig.~\ref{fig:versatility_forage} compares the four methods in the foraging task. The radar plot reports per-robot foraging efficiency \(\Phi_{\mathrm{ind}}\), transport economy \(E_{\mathrm{trans}}\), collision-safety score \(C_{\mathrm{rate}}\) and control-smoothness score \(J_{\mathrm{smooth}}\). PhySwarm achieves the most balanced profile, indicating that adaptive modulation of advection, diffusion and reaction is important for coordinating exploration, food approach, homing and trail following. The FSM baseline can exploit manually designed task heuristics, but it lacks continuous physical-parameter adaptation. The fixed-parameter ablations show that neither weak static regulation nor over-active phase transition can sustain the full foraging cycle.

\begin{figure}[htbp]
    \centering
    \includegraphics[width=0.8\linewidth]{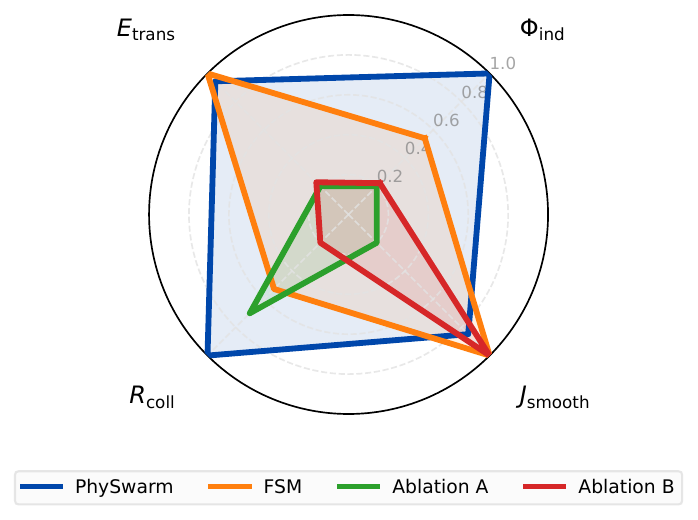} 
    \caption{
    \textbf{Versatility in Trail-Guided Swarm Foraging.}
    Radar comparison across normalized per-robot foraging efficiency \(\Phi_{\mathrm{ind}}\), transport economy \(E_{\mathrm{trans}}\), collision safety \(R_{\mathrm{coll}}\) and control smoothness \(J_{\mathrm{smooth}}\). Larger values indicate better normalized performance.
    }
    \label{fig:versatility_forage}
\end{figure}

\textbf{Formation-Reconfigurable Swarm Navigation. }
Fig.~\ref{fig:versatility_formation} compares the four methods in the formation-reconfigurable navigation task. The radar plot reports normalized navigation efficiency \(T_{\mathrm{nav}}\), formation-quality score \(E_{\mathrm{form}}\), collision-safety score \(C_{\mathrm{rate}}\) and control-smoothness score \(J_{\mathrm{smooth}}\). PhySwarm maintains a strong and balanced profile, showing that the swarm can coordinate formation keeping, corridor traversal and formation recovery. Ablation A over-emphasizes shape preservation and weakens forward progression, whereas Ablation B over-emphasizes propulsion and weakens formation fidelity. The result supports the need for adaptive balancing among morphology fields, forward-flow fields and diffusion-based safety regulation.

\begin{figure}[htbp]
    \centering
    \includegraphics[width=0.8\linewidth]{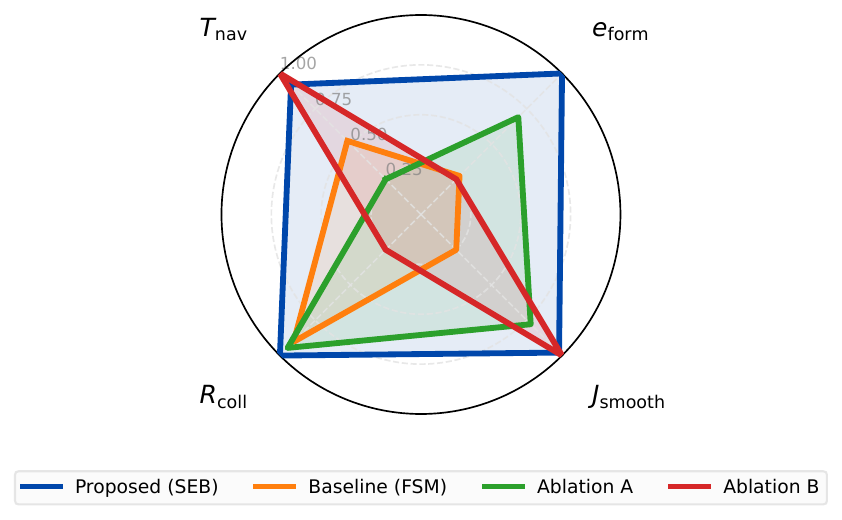} 
    \caption{
    \textbf{Versatility in Formation-Reconfigurable Swarm Navigation.}
    Radar comparison across normalized navigation efficiency \(T_{\mathrm{nav}}\), formation-quality score \(E_{\mathrm{form}}\), collision safety \(R_{\mathrm{coll}}\) and control smoothness \(J_{\mathrm{smooth}}\). Larger values indicate better normalized performance.
    }
    \label{fig:versatility_formation}
\end{figure}

\textbf{Role-Adaptive Swarm Search and Rescue. }
Fig.~\ref{fig:versatility_rescue} compares the four methods in the search-and-rescue task. The radar plot reports task success rate \(\eta_{\mathrm{succ}}\), time-to-connectivity score \(T_{\mathrm{TTC}}\), collision-safety score \(C_{\mathrm{rate}}\) and control-smoothness score \(J_{\mathrm{smooth}}\). PhySwarm provides a balanced advantage because it can jointly regulate distributed search, target response, relay formation and local spacing. The FSM baseline uses fixed role-allocation rules, whereas the ablations reveal that aggressive target seeking or excessive diffusion alone cannot maintain both relay connectivity and safe response behavior.

\begin{figure}[htbp]
    \centering
    \includegraphics[width=0.8\linewidth]{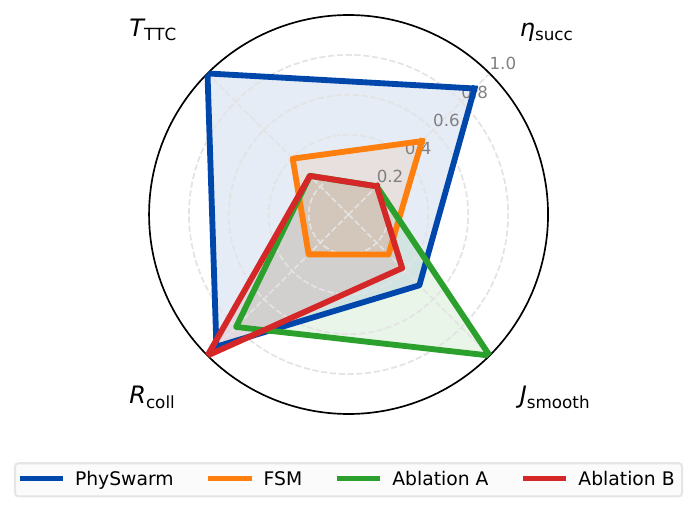} 
    \caption{
    \textbf{Versatility in Role-Adaptive Swarm Search and Rescue.}
    Radar comparison across normalized task success rate \(\eta_{\mathrm{succ}}\), time-to-connectivity score \(T_{\mathrm{TTC}}\), collision safety \(R_{\mathrm{coll}}\) and control smoothness \(J_{\mathrm{smooth}}\). Larger values indicate better normalized performance.
    }
    \label{fig:versatility_rescue}
\end{figure}

Overall, the versatility analysis shows that PhySwarm does not rely on a single fixed physical regime. Instead, the NPC adapts \(\omega\), \(D\) and \(\lambda\) according to task stage and local observations, allowing the same Macro-ADR--Micro-EDM interface to support information-guided foraging, geometry-constrained formation reconfiguration and role-adaptive relay construction.

\clearpage

\subsection{The boundedness}\label{sec:bound}

This section proves that, under the Neural-Physics Controller, the physical parameters, advection field, diffusion term and reaction term in the multi-stage Macro-ADR behavior representation remain bounded. Unlike controllers that directly output unconstrained robot actions, the proposed NPC outputs the physical parameters of Macro-the ADR system, \(P(t)=\{\omega(t),D(t),\lambda(t)\}\). Therefore, the central question in the boundedness analysis is whether these neural-network-predicted parameters may introduce unbounded velocities, unbounded diffusion or unbounded phase-transition rates when applied to the ADR dynamics.

To address this question, we first define the physical parameters predicted by the controller and the corresponding bounded feasible parameter manifold. This specifies how the advection weights, diffusion coefficients and reaction rates are projected into physically admissible ranges. We then prove that the advection term generated by weighted potential-field bases, the diffusion term driven by density gradients and the reaction term induced by the conservative phase-transition structure all admit finite upper bounds.

This analysis shows that the NPC does not learn arbitrary unconstrained control inputs. Instead, it modulates the ADR mechanisms within a bounded physical parameter space, thereby preventing non-physical divergence and providing a foundation for the subsequent controllability and convergence analyses.

\subsubsection{Notations and assumptions}

For robot \(i\), the Neural-Physics Controller predicts the control parameters from its local observation \(O_i(t)\) and temporal memory \(h_{i,t-1}\):
\begin{equation}
    P_i(t)
    =
    F_{\theta}\big(O_i(t),h_{i,t-1}\big)
    =
    \{\omega_i(t),D_i(t),\lambda_i(t)\}.
    \label{eq:boundedness_agent_parameter_cn}
\end{equation}
Here, \(\omega_i(t)\) denotes the coupling weights of different potential-field bases in the advection field, \(D_i(t)\) denotes the diffusion coefficient, and \(\lambda_i(t)\) denotes the reaction transition rates among behavioral phases.

At the continuum-field level, the physical parameters associated with behavioral phase \(\sigma\) can be written as
\(P_{\sigma}(x,t)=\{\omega_{\sigma}(x,t),D_{\sigma}(x,t),\lambda(x,t)\}\), where
\(\omega_{\sigma}(x,t)=[\omega_{\sigma,1}(x,t),\ldots,\omega_{\sigma,K_b}(x,t)]^{\top}\) combines \(K_b\) potential-field bases, \(D_{\sigma}(x,t)\) controls the diffusion strength of phase \(\sigma\), and \(\lambda_{mn}(x,t)\) controls the local transition rate from behavioral phase \(m\) to behavioral phase \(n\).

For each behavioral phase \(\sigma\), the advection velocity field is generated by a weighted combination of physically meaningful potential-field bases. Let
\begin{equation}
    b_r(x,t)
    =
    -\nabla \Phi_r(x,t),
    \qquad
    r=1,\ldots,K_b,
    \label{eq:boundedness_velocity_basis_cn}
\end{equation}
where \(\Phi_r(x,t)\) is the \(r\)-th potential-field basis function, and \(b_r(x,t)\) is the velocity basis induced by this field. The potential-field bases may represent physical mechanisms such as target attraction, obstacle repulsion, boundary constraints, swarm cohesion, individual separation, velocity alignment or formation maintenance. Further details are provided in Sec.~\ref{fields_example}.

Thus, the advection velocity field of phase \(\sigma\) is
\begin{equation}
    u_{\sigma}(x,t)
    =
    \sum_{r=1}^{K_b}
    \omega_{\sigma,r}(x,t)b_r(x,t)
    =
    -
    \sum_{r=1}^{K_b}
    \omega_{\sigma,r}(x,t)\nabla\Phi_r(x,t).
    \label{eq:boundedness_advection_field_cn}
\end{equation}
This representation shows that the controller changes the contribution of different physical fields to the current motion direction by regulating \(\omega_{\sigma,r}\), rather than directly generating an arbitrary unconstrained velocity field.

\begin{assumption}[Bounded potential-field bases~\cite{rimon1992navigation,olfati2006flocking}]
\label{assumption_boundedness_basis}
There exists a constant \(B_{\Phi}>0\) such that, for all \(r=1,\ldots,K_b\),
\begin{equation}
    \|b_r\|_{\infty}
    =
    \|\nabla\Phi_r\|_{\infty}
    \leq B_{\Phi}.
    \label{eq:boundedness_basis_assumption_cn}
\end{equation}
This assumption ensures that the velocity direction induced by each physical potential-field basis has finite magnitude.
\end{assumption}

\begin{assumption}[Bounded density field~\cite{chandrasekhar1943stochastic,cosner2014reaction}]
\label{assumption_boundedness_density}
There exists a constant \(B_{\rho}>0\) such that, for all behavioral phases \(\sigma=1,\ldots,M\),
\begin{equation}
    \|\rho_{\sigma}\|_{\infty}
    \leq B_{\rho}.
    \label{eq:boundedness_density_assumption_cn}
\end{equation}
This condition excludes non-physical cases in which the reaction term generates unbounded density.
\end{assumption}

Unless otherwise specified, the boundedness analysis uses the \(L^\infty\) norm. For a scalar field \(f(x,t)\), we define
\begin{equation}
    \|f\|_{\infty}
    =
    \operatorname*{ess\,sup}_{(x,t)\in\Omega\times[0,T]}
    |f(x,t)|.
    \label{eq:boundedness_scalar_norm_cn}
\end{equation}
For a vector field \(v(x,t)\), we define
\begin{equation}
    \|v\|_{\infty}
    =
    \operatorname*{ess\,sup}_{(x,t)\in\Omega\times[0,T]}
    \|v(x,t)\|_{2}.
    \label{eq:boundedness_vector_norm_cn}
\end{equation}
For a matrix or tensor field, \(\|\cdot\|_{\infty}\) denotes the essential supremum of the corresponding matrix or tensor norm over \(\Omega\times[0,T]\).

\clearpage

\subsubsection{Boundedness of the parameters}
\label{bound_of_para}

To ensure that the parameters produced by the Neural-Physics Controller always have clear physical meaning, we do not directly use the unconstrained outputs of the network. Instead, these outputs are mapped onto a bounded physical parameter manifold. Specifically, the ADR parameter set \(P_i(t)=\{\omega_i(t),D_i(t),\lambda_i(t)\}\) is constrained to the following feasible physical domain:
\begin{equation}
    \mathcal{M}
    =
    \mathcal{M}_{\omega}
    \times
    \mathcal{M}_{D}
    \times
    \mathcal{M}_{\lambda}.
    \label{eq:bounded_physical_manifold_cn}
\end{equation}
The advection weights, diffusion coefficients and reaction rates are constrained as
\begin{equation}
\label{eq:manifold}
    \begin{aligned}
        &\mathcal{M}_{\omega}
    =
    \prod_{\sigma=1}^{M}
    \Delta^{K_b-1},
    \qquad
    \Delta^{K_b-1}
    =
    \left\{
    \omega_{\sigma}\in\mathbb{R}_{+}^{K_b}
    :
    \sum_{r=1}^{K_b}
    \omega_{\sigma,r}=1
    \right\},\\
        &\mathcal{M}_{D}
    =
    \left\{
    D_{\sigma}
    :
    D_{\min}
    \leq
    D_{\sigma}
    \leq
    D_{\max},
    \quad
    \sigma=1,\ldots,M
    \right\},
    \qquad
    0<D_{\min}\leq D_{\max}<\infty,\\
        &\mathcal{M}_{\lambda}
    =
    \left\{
    \lambda_{mn}
    :
    0\leq
    \lambda_{mn}
    \leq
    A_{mn}\lambda_{\max},
    \quad
    m\neq n
    \right\},
    \qquad
    \lambda_{\max}<\infty .
    \end{aligned}
\end{equation}
Here, \(\Delta^{K_b-1}\) is the standard simplex, and \(K_b\) denotes the number of potential-field bases. This constraint requires the advection weights to be non-negative and to sum to one. Therefore, the advection velocity field is a convex combination of bounded physical field bases, rather than an arbitrary unbounded linear superposition. The condition \(D_{\min}>0\) keeps the ADR equation non-degenerate in its diffusive component, whereas \(D_{\max}<\infty\) prevents the diffusion strength from growing without bound. The reaction rate \(\lambda_{mn}\) represents the transition rate from behavioral phase \(m\) to behavioral phase \(n\). Its non-negativity gives the phase transition the interpretation of a continuous-time Markov process, while the upper bound \(\lambda_{\max}\) prevents unbounded switching frequency. The mask \(A_{mn}\in\{0,1\}\) specifies whether a transition is allowed by the task logic: \(A_{mn}=0\) disables the transition from phase \(m\) to phase \(n\), whereas \(A_{mn}=1\) allows the corresponding rate to be learned by the network. If no explicit transition mask is used, all transitions considered in the phase graph can be set to \(A_{mn}=1\).

In implementation, the neural network first produces unconstrained real-valued outputs
\[
    z_i(t)
    =
    \{z_{\omega,i}(t),z_{D,i}(t),z_{\lambda,i}(t)\}.
\]
These outputs are then mapped to feasible physical parameters through a differentiable physical parameterization layer~\cite{box1966comparison,alizadeh2023power}:
\begin{equation}
\label{eq:manifold_projection}
    \begin{aligned}
    & \omega_{i,r}(t)
    =
    \frac{\exp(z_{\omega,i,r}(t))}
    {\sum_{q=1}^{K_b}\exp(z_{\omega,i,q}(t))},
    \qquad
    r=1,\ldots,K_b,\\
    & D_i(t)
    =
    D_{\min}
    +
    (D_{\max}-D_{\min})
    \operatorname{sigmoid}
    \left(
    z_{D,i}(t)
    \right),\\
    &
    \lambda_{mn,i}(t)
    =
    A_{mn}
    \lambda_{\max}
    \operatorname{sigmoid}
    \left(
    z_{\lambda,mn,i}(t)
    \right).
    \end{aligned}
\end{equation}
The softmax mapping ensures that the advection weights always satisfy \(\omega_{i,r}(t)\geq0\) and \(\sum_{r=1}^{K_b}\omega_{i,r}(t)=1\). The sigmoid interval mappings ensure that the diffusion coefficient satisfies \(D_{\min}<D_i(t)<D_{\max}\), and that the reaction rate satisfies \(0\leq\lambda_{mn,i}(t)\leq A_{mn}\lambda_{\max}\). Therefore, for any network parameter \(\theta\), any time \(t\) and any robot \(i\),
\[
    P_i(t)\in\mathcal{M}.
\]

This differentiable physical parameterization layer constitutes the hard-constraint mechanism of the Neural-Physics Controller. It ensures that the NPC does not learn unconstrained black-box actions, but instead produces physically meaningful ADR parameters: \(\omega\) regulates the relative contributions of different field mechanisms, \(D\) regulates the strength of spatial diffusion, and \(\lambda\) regulates the transition rates among behavioral phases. Based on this parameter boundedness, we next prove the boundedness of the advection, diffusion and reaction terms in the Macro-ADR.

\clearpage

\subsubsection{Boundedness of the field}

\begin{lemma}\label{lemma:adv}
    The advection term is bounded.
\end{lemma}

\begin{proof}[\textbf{Proof of Lemma~\ref{lemma:adv}}]

For behavioral phase \(\sigma\), the advection velocity field is generated by a weighted combination of potential-field gradients:
\[
u_\sigma(x,t)
=
-\sum_{r=1}^{K_b}
\omega_{\sigma,r}(x,t)\nabla\Phi_r(x,t).
\]
If all potential-field bases have bounded gradients (Assumption~\ref{assumption_boundedness_basis}), then, using \(\omega_{\sigma,r}\geq 0\) and \(\sum_r\omega_{\sigma,r}=1\), we have
\[
\|u_\sigma\|_\infty
\leq
\sum_{r=1}^{K_b}
\omega_{\sigma,r}
\|\nabla\Phi_r\|_\infty
\leq
B_\Phi.
\]
Therefore, the advection velocity field is bounded, and hence the advection term remains bounded.
    
\end{proof}

\begin{lemma}\label{lemma:diff}
    The diffusion term is bounded.
\end{lemma}

\begin{proof}[\textbf{Proof of Lemma~\ref{lemma:diff}}]

For the diffusion term, the physical manifold projection ensures that
\[
0<D_{\min}\leq D_\sigma(x,t)\leq D_{\max}<\infty.
\]
Thus, the diffusion process is neither degenerate nor unbounded. The corresponding diffusive flux is
\[
J_{\mathrm{diff},\sigma}
=
-D_\sigma\nabla\rho_\sigma .
\]
If the density gradient is bounded (Assumption~\ref{assumption_boundedness_density}), namely \(\|\nabla\rho_\sigma\|_\infty<\infty\), then
\[
\|J_{\mathrm{diff},\sigma}\|_\infty
\leq
D_{\max}
\|\nabla\rho_\sigma\|_\infty
<\infty.
\]
At the microscopic execution level, the diffusion-compensation velocity is given by Eq.~\eqref{eq:micro_diffusion_velocity}. Since \(D_i\leq D_{\max}\) and \(\hat{\rho}(x_i)+\varepsilon\geq\varepsilon>0\), we obtain
\[
\|v_{\mathrm{diff},i}\|
\leq
\frac{D_{\max}}{\varepsilon}
\|\nabla\hat{\rho}\|_\infty
<\infty.
\]
Thus, the diffusion mechanism remains bounded at both the macroscopic flux level and the microscopic velocity level.

If, in addition, \(D_\sigma\in W^{1,\infty}(\Omega)\) and \(\rho_\sigma\in W^{2,\infty}(\Omega)\), then the diffusion operator satisfies
\[
\nabla\cdot(D_\sigma\nabla\rho_\sigma)
=
D_\sigma\Delta\rho_\sigma
+
\nabla D_\sigma\cdot\nabla\rho_\sigma.
\]
Therefore,
\[
\|\nabla\cdot(D_\sigma\nabla\rho_\sigma)\|_\infty
\leq
D_{\max}\|\Delta\rho_\sigma\|_\infty
+
\|\nabla D_\sigma\|_\infty
\|\nabla\rho_\sigma\|_\infty
<\infty.
\]
When \(D_\sigma\) is spatially constant, this reduces to
\[
\|\nabla\cdot(D_\sigma\nabla\rho_\sigma)\|_\infty
\leq
D_{\max}\|\Delta\rho_\sigma\|_\infty.
\]
Hence, under the stated regularity conditions, the diffusion term is bounded.
\end{proof}

\begin{lemma}\label{lemma:rea}
    The reaction term is bounded.
\end{lemma}

\begin{proof}[\textbf{Proof of Lemma~\ref{lemma:rea}}]

For the reaction term, we use \(R(\rho)=\Lambda(\lambda)\rho\), whose \(\sigma\)-th component is given by Eq.~\eqref{eq:conservative_reaction_structure_cn}. If all reaction rates satisfy \(0\leq\lambda_{mn}\leq\lambda_{\max}\), and all phase densities are bounded (Assumption~\ref{assumption_boundedness_density}), then
\[
\mid R_\sigma(\rho)\mid
\leq
\sum_{\eta\neq\sigma}
\lambda_{\eta\sigma}\mid \rho_\eta\mid 
+
\sum_{\ell\neq\sigma}
\lambda_{\sigma\ell}\mid \rho_\sigma\mid 
\leq
2(M-1)\lambda_{\max}B_\rho
<\infty.
\]
Therefore, under bounded densities and bounded transition rates, the reaction term cannot diverge.

\end{proof}

\clearpage

\subsection{Controllability of the model}
\label{sec:control}

This section analyzes the conditional controllability of the Macro-ADR representation in PhySwarm. Specifically, we examine whether there exists a bounded physical parameter trajectory that drives the swarm density from an initial distribution \(\rho^0\) to a target distribution \(\rho^\ast\), or to a neighborhood of it.

The proof decomposes the control of multi-stage emergent behavior into two coupled components: phase-mass redistribution and within-phase spatial transport. The reaction rates \(\lambda\) regulate mass transfer among behavioral phases, so that the phase-wise mass variation along a target density path can be represented by the reaction term. Given this phase-mass evolution, the advection weights \(\omega\) and diffusion coefficients \(D\) regulate the spatial redistribution of density within each phase.

More specifically, we first construct a regular target path connecting the initial density and the target density. We then use bounded reaction rates to match the phase-mass evolution induced by this path. Next, we construct a bounded spatial flux satisfying the local conservation relation and recover the desired advection field from the flux relation. Since the controller can only generate advection fields through a finite set of potential-field bases, controllability ultimately depends on the expressive capacity of these bases. If the bases can exactly represent the desired advection field, exact controllability is obtained within the physically expressible space. If the desired field can only be approximated, the model achieves approximate controllability with an error controlled by the field-basis approximation error \(\varepsilon_{\mathrm{basis}}\).

\subsubsection{Notations and assumptions}

\begin{definition}[Target density path]\label{definitoin_density}
To prove reachability from an initial density to a target density, we assume that there exists a target density path connecting \(\rho^{0}\) and \(\rho^{\ast}\):
\begin{equation}
    \gamma(x,t)
    =
    [\gamma_{1}(x,t),\gamma_{2}(x,t),\ldots,\gamma_{M}(x,t)]^{\top},
    \qquad
    x\in\Omega,\quad t\in[0,T],
    \label{eq:target_density_path_ctrl_cn}
\end{equation}
satisfying
\begin{equation}
    \gamma_{\sigma}(x,0)=\rho^{0}_{\sigma}(x),
    \qquad
    \gamma_{\sigma}(x,T)=\rho^{\ast}_{\sigma}(x),
    \qquad
    \sigma=1,\ldots,M.
    \label{eq:target_density_path_boundary_cn}
\end{equation}
This path represents an ideal continuous evolution of the swarm density from the initial behavioral distribution to the target behavioral distribution.
\end{definition}

\begin{definition}[Controllability of emergent behavior]\label{definitoin_emergent}
Given an initial density
\(\rho^0(x)=[\rho_1^0(x),\ldots,\rho_M^0(x)]^T\)
and a target behavioral density
\(\rho^\ast(x)=[\rho_1^\ast(x),\ldots,\rho_M^\ast(x)]^T\),
which satisfy the same total-mass condition in Theorem~\ref{thm:local_and_global_mass_conservation_cn}, the Macro-ADR behavior representation is said to be \textbf{controllable} from \(\rho^0\) to \(\rho^\ast\) over the time interval \([0,T]\) if there exists a bounded physical parameter trajectory
\[
P(t)=\{\omega(t),D(t),\lambda(t)\}\in\mathcal{M},
\qquad t\in[0,T],
\]
such that the solution of the ADR system satisfies
\[
\rho(x,0)=\rho^0(x),
\qquad
\rho(x,T)=\rho^\ast(x).
\]
If there exists \(P(t)\in\mathcal{M}\) such that
\[
\|\rho(\cdot,T)-\rho^\ast\|
\leq
\varepsilon_c,
\]
then the system is said to be \textbf{approximately controllable}, where \(\varepsilon_c\) denotes the controllability error.
\end{definition}

\begin{assumption}[Regular and strictly positive target density path~\cite{chandrasekhar1943stochastic,cosner2014reaction,chueh1977positively}]
\label{assump:target_density_path}
For each behavioral phase \(\sigma=1,\ldots,M\), the target density path \(\gamma_{\sigma}\) satisfies
\begin{equation}
    \gamma_{\sigma}(x,t)\geq \rho_{\min}>0,
    \qquad
    x\in\Omega,\quad t\in[0,T],
    \label{eq:positive_target_path_ctrl_cn}
\end{equation}
and \(\gamma_\sigma\), its time derivative \(\partial_t\gamma_\sigma\), and its spatial gradient \(\nabla\gamma_\sigma\) are all bounded on \(\Omega\times[0,T]\). Equivalently, we assume
\begin{equation}
    \gamma_{\sigma}
    \in
    C^{1}([0,T];L^{\infty}(\Omega))
    \cap
    C([0,T];W^{1,\infty}(\Omega)).
    \label{eq:target_path_regular_ctrl_cn}
\end{equation}
Here, \(L^\infty(\Omega)\) denotes the space of bounded functions on \(\Omega\), and \(W^{1,\infty}(\Omega)\) denotes the Sobolev space of functions whose values and first-order spatial derivatives are bounded.
\end{assumption}

\begin{remark}
The boundedness of the spatial gradient of the density path ensures that the diffusion flux in the ADR model remains finite. The positive lower bound is required because the desired advection field is recovered from the density flux, which involves division by the density. If the density vanishes, the required velocity field may become singular. This assumption is therefore a technical condition for proving bounded controllability, rather than a strict restriction on practical target behaviors. If a target density contains zero-density regions, it can be regularized by adding a small background density, thereby avoiding singular velocities while remaining close to the original target distribution.
\end{remark}

\begin{assumption}[Strong connectivity of the behavioral phase-transition graph~\cite{zhu2007asymptotic,zhu2009strong,elamvazhuthi2019bilinear}]
\label{assump:phase_graph}
Let \(\mathcal{S}=\{1,2,\ldots,M\}\) denote the set of behavioral phases in the swarm system, such as exploration, aggregation, formation, navigation or return. Define the behavioral phase-transition graph as
\(\mathcal{G}=(\mathcal{S},\mathcal{E})\), where \(\mathcal{S}\) is the node set and \(\mathcal{E}\) is the set of directed edges. If behavioral phase \(m\) can transition to behavioral phase \(n\) through the reaction rate \(\lambda_{mn}\), then the directed edge \(m\rightarrow n\) exists in the graph. Here, \(\lambda_{mn}(t)\) denotes the transition rate from behavioral phase \(m\) to behavioral phase \(n\) per unit time, and satisfies the physical feasibility constraint
\[
0\leq \lambda_{mn}(t)\leq \lambda_{\max}.
\]

We assume that the behavioral phase-transition graph \(\mathcal{G}\) is strongly connected. That is, for any two behavioral phases \(m,n\in\mathcal{S}\), there exists a directed path composed of finitely many edges such that the system can reach phase \(n\) from phase \(m\), either directly or indirectly.

Furthermore, given the target density path \(\gamma(x,t)\), define the total mass of behavioral phase \(\sigma\) at time \(t\) as
\[
m_\sigma(t)
=
\int_\Omega
\gamma_\sigma(x,t)\,dx .
\]
We assume that there exists a non-negative and bounded reaction-rate trajectory
\(\lambda(t)\in\mathcal{M}_\lambda,\ t\in[0,T]\),
such that the target phase masses satisfy
\[
\dot{m}_\sigma(t)
=
\sum_{\eta\neq\sigma}
\lambda_{\eta\sigma}(t)m_\eta(t)
-
\sum_{\ell\neq\sigma}
\lambda_{\sigma\ell}(t)m_\sigma(t),
\qquad
\sigma=1,\ldots,M.
\]
\end{assumption}

\begin{remark}
This assumption ensures that the Macro-ADR reaction term has sufficient capacity to switch among behavioral phases. The reaction term
\[
R_\sigma(\rho)=\sum_{\eta\neq\sigma}\lambda_{\eta\sigma}\rho_\eta-\sum_{\ell\neq\sigma}\lambda_{\sigma\ell}\rho_\sigma
\]
describes density transfer among behavioral phases. Strong connectivity guarantees that every behavioral phase is structurally reachable from every other phase, so the system does not contain phases that are unreachable by construction. If the transition graph were not strongly connected, then even if the advection and diffusion terms could reshape the spatial distribution of the swarm, robot mass could not be transferred to unreachable behavioral phases, making certain target behavioral distributions uncontrollable.

Strong connectivity only indicates structural reachability; it does not imply reachability within an arbitrarily short time. Therefore, this assumption further requires the existence of a bounded reaction-rate trajectory satisfying \(0\leq\lambda_{mn}(t)\leq\lambda_{\max}\) that can realize the target phase-mass path \(m_\sigma(t)\). This excludes target paths that would require infinitely fast state switching, and keeps the controllability statement consistent with physical constraints.

Thus, Assumption~\ref{assump:phase_graph} transforms the multi-stage behavioral switching problem into a phase-mass transfer problem under bounded reaction rates. It ensures that \(\lambda\) has a clear physical interpretation and can regulate the proportion of robots in different behavioral phases within feasible rate bounds, thereby providing a consistent phase-mass basis for the subsequent spatial-flux construction and desired-advection recovery.
\end{remark}

\begin{assumption}[Expressive capacity of the potential-field bases]
\label{assump:supp_basis_approx_cn}
Let \(u^{\mathrm{des}}_\sigma(x,t)\) denote the desired advection field. The Neural-Physics Controller constructs the advection field using a finite set of potential-field bases
\(\mathcal{B}=\{b_r(x,t)\}_{r=1}^{K_b}\):
\begin{equation}
    u_\sigma^\theta(x,t)
    =
    \sum_{r=1}^{K_b}
    \omega_{\sigma,r}^\theta(x,t)b_r(x,t).
    \label{eq:supp_basis_advection_cn}
\end{equation}
We assume that there exists an ideal weight vector \(\omega^\dagger\in\mathcal{M}_\omega\) such that
\begin{equation}
    \left\|
    u_\sigma^{\dagger}
    -
    u_\sigma^{\mathrm{des}}
    \right\|_{L^2(\Omega\times[0,T])}
    \leq
    \varepsilon_{\mathrm{basis}},
    \label{eq:supp_basis_error_cn}
\end{equation}
where
\[
u_\sigma^{\dagger}
=
\sum_{r=1}^{K_b}
\omega_{\sigma,r}^{\dagger}b_r .
\]
That is, the potential-field bases can approximate the desired advection field up to the error \(\varepsilon_{\mathrm{basis}}\).
\end{assumption}

\begin{assumption}[Constructible and bounded spatial flux~\cite{dacorogna1990prescribing}]
\label{assump:space_flux_bounded}
Given the target density path
\(\gamma(x,t)=[\gamma_1(x,t),\ldots,\gamma_M(x,t)]^T\)
and the reaction term
\(R(\gamma)=\Lambda(\lambda)\gamma\)
induced by the reaction-rate trajectory \(\lambda(t)\), define, for each behavioral phase \(\sigma\in\{1,\ldots,M\}\),
\[
h_\sigma(x,t)
=
R_\sigma(\gamma)
-
\frac{\partial \gamma_\sigma}{\partial t}.
\]
Here, \(h_\sigma(x,t)\) represents the part of the density change along the target path that remains to be realized by spatial transport after subtracting the phase-transition contribution \(R_\sigma(\gamma)\). We assume that \(h_\sigma\) satisfies the zero-mean compatibility condition~\cite{chandrasekhar1943stochastic,cosner2014reaction}:
\begin{equation}
    \int_\Omega h_\sigma(x,t)\,dx=0,
\qquad
\forall t\in[0,T].
\label{eq:zero-mean}
\end{equation}
Under the zero-flux boundary condition, we assume that there exists a spatial flux field
\(J_\sigma:\Omega\times[0,T]\rightarrow\mathbb{R}^d\)
such that
\[
\nabla\cdot J_\sigma(x,t)=h_\sigma(x,t),
\]
\[
\mathbf{n}\cdot J_\sigma(x,t)=0,
\qquad x\in\partial\Omega,
\]
where \(\mathbf{n}\) denotes the outward normal vector on the boundary \(\partial\Omega\).

We further assume that this flux field is bounded over the control interval: there exists a constant \(B_J>0\) such that
\[
\|J_\sigma\|_\infty\leq B_J,
\qquad
\sigma=1,\ldots,M.
\]
Here, \(\|J_\sigma\|_\infty\) denotes the maximum magnitude of the flux field over space and time.
\end{assumption}

\begin{remark}
This assumption ensures that the portion of the target density-path variation that must be realized by spatial transport can be implemented by a bounded spatial flux field. The reaction term \(R_\sigma(\gamma)\) controls the mass redistribution among behavioral phases, but by itself it does not determine how density is spatially rearranged within a given phase. Therefore, after the target path \(\gamma\) and the reaction term \(R_\sigma(\gamma)\) are specified, a spatial flux \(J_\sigma\) is required to transport density within each phase.

The zero-mean condition in Eq.~\eqref{eq:zero-mean} is the compatibility condition for constructing a spatial flux under zero-flux boundary conditions. Physically, when robots cannot enter or leave the domain through the boundary, spatial flux can only redistribute density inside the domain and cannot change the total mass of that behavioral phase. The phase mass variation has already been controlled by the reaction rates \(\lambda\). Therefore, the spatial integral of \(h_\sigma\) must be zero.

The boundedness of the flux is used to ensure that the desired advection field recovered in the subsequent step is also bounded (see the proof of Theorem~\ref{thm:desired_advection_bounded_basis_cn}).
\end{remark}

\clearpage

\subsubsection{Controllability of emergent behaviors}

We now formalize the controllability argument outlined above. The analysis proceeds in two steps. First, we show that, once the target phase-mass evolution is matched by bounded reaction rates and the remaining density variation is represented by a bounded spatial flux, the desired advection field recovered from the flux relation is also bounded. We further relate this theoretical field to the realizable advection field generated by the finite potential-field bases used by the Neural-Physics Controller. This gives an explicit approximation error \(\varepsilon_{\mathrm{basis}}\), which quantifies the expressive capacity of the chosen field dictionary.

Second, we use this field-approximation result together with the finite-time stability of the Macro-ADR dynamics to establish conditional controllability of PhySwarm. The resulting statement should be interpreted within the physically admissible space defined by the bounded parameter manifold and the potential-field bases: exact controllability is obtained when the desired advection field is exactly representable, whereas approximate controllability is obtained when it can only be approximated. This formulation is consistent with the design of PhySwarm, in which \(\lambda\), \(D\) and \(\omega\) regulate phase transitions, density diffusion and directed migration, respectively, rather than generating arbitrary unconstrained velocity fields.

\begin{theorem}[Boundedness of the desired advection field and approximation by field bases]
\label{thm:desired_advection_bounded_basis_cn}
Suppose that the target density path is regular and strictly positive (Assumption~\ref{assump:target_density_path}), the spatial flux is constructible and bounded (Assumption~\ref{assump:space_flux_bounded}), and the diffusion coefficient is bounded (Eq.~\eqref{eq:manifold_projection}). Then, for each behavioral phase \(\sigma=1,\ldots,M\), the desired advection field \(u_\sigma^{\mathrm{des}}\) required to realize the target density path is bounded. Moreover, the realizable advection field \(u_\sigma^\omega\) generated by the potential-field bases is also bounded and satisfies
\begin{equation}
    \left\|
    u_\sigma^{\mathrm{des}}
    -
    u_\sigma^\omega
    \right\|_{\mathcal V}
    \leq
    \varepsilon_{\mathrm{basis}},
    \label{eq:desired_advection_basis_error_cn}
\end{equation}
where \(\varepsilon_{\mathrm{basis}}\geq 0\) denotes the approximation error of the field bases with respect to the desired advection field, and \(\|\cdot\|_{\mathcal V}\) can be taken as the \(L^2(\Omega\times[0,T])\) or \(L^\infty(\Omega\times[0,T])\) norm.
\end{theorem}

\begin{proof}[\textbf{Proof of Theorem~\ref{thm:desired_advection_bounded_basis_cn}:}]
From the spatial-flux construction above, for the target density path \(\gamma_\sigma(x,t)\), there exists a spatial flux \(J_\sigma(x,t)\) satisfying the conservation relation. By the total flux relation
\begin{equation}
    J_\sigma
    =
    u_\sigma\gamma_\sigma
    -
    D_\sigma\nabla\gamma_\sigma,
    \label{eq:flux_relation_for_desired_advection_cn}
\end{equation}
the desired advection field required to realize the target density path can be recovered as
\begin{equation}
    u_\sigma^{\mathrm{des}}(x,t)
    =
    \frac{
    J_\sigma(x,t)
    +
    D_\sigma(x,t)\nabla\gamma_\sigma(x,t)
    }{
    \gamma_\sigma(x,t)
    }.
    \label{eq:desired_advection_field_cn}
\end{equation}

By the positivity and regularity assumption on the target density path (Assumption~\ref{assump:target_density_path}), we have
\begin{equation}
    \gamma_\sigma(x,t)\geq \rho_{\min}>0,
    \qquad
    \|\nabla\gamma_\sigma\|_{\infty}<\infty.
\end{equation}
By the boundedness of the spatial flux (Assumption~\ref{assump:space_flux_bounded}) and the boundedness of the diffusion coefficient, we also have
\begin{equation}
    \|J_\sigma\|_{\infty}\leq B_J,
    \qquad
    0<D_{\min}\leq D_\sigma(x,t)\leq D_{\max}<\infty.
\end{equation}
Therefore, Eq.~\eqref{eq:desired_advection_field_cn} gives
\begin{equation}
    \|u_\sigma^{\mathrm{des}}\|_{\infty}
    \leq
    \frac{
    \|J_\sigma\|_{\infty}
    +
    D_{\max}\|\nabla\gamma_\sigma\|_{\infty}
    }{
    \rho_{\min}
    }
    <
    \infty.
    \label{eq:desired_advection_bound_cn}
\end{equation}
Thus, the desired advection field is bounded.

We next prove that the realizable advection field generated by the potential-field bases is also bounded. Let the field bases be
\(b_r(x,t)=-\nabla\Phi_r(x,t),\ r=1,\ldots,K_b.\)
The Neural-Physics Controller combines these bases through the weights \(\omega_\sigma(x,t)\) to generate the advection field
\begin{equation}
    u_\sigma^\omega(x,t)
    =
    \sum_{r=1}^{K_b}
    \omega_{\sigma,r}(x,t)b_r(x,t).
    \label{eq:realizable_advection_basis_cn}
\end{equation}
Here,
\begin{equation}
    \omega_\sigma(x,t)\in\mathcal M_\omega,
    \qquad
    \mathcal M_\omega
    =
    \left\{
    \omega\in\mathbb R_+^{K_b}:
    \sum_{r=1}^{K_b}\omega_r=1
    \right\}.
    \label{eq:omega_simplex_for_advection_cn}
\end{equation}
If all field bases are bounded (Assumption~\ref{assumption_boundedness_basis}), then the simplex constraint gives
\begin{align}
    \|u_\sigma^\omega\|_{\infty}
    =
    \left\|
    \sum_{r=1}^{K_b}
    \omega_{\sigma,r}b_r
    \right\|_{\infty}
    \nonumber\\
    \leq
    \sum_{r=1}^{K_b}
    \omega_{\sigma,r}
    \|b_r\|_{\infty}
    \nonumber\\
    \leq
    B_\Phi
    \sum_{r=1}^{K_b}
    \omega_{\sigma,r}
    =
    B_\Phi .
    \label{eq:realizable_advection_bound_cn}
\end{align}
Therefore, the advection field generated by the field-basis combination is also bounded.

Finally, by the expressive-capacity assumption of the field bases (Assumption~\ref{assump:supp_basis_approx_cn}), there exists \(\omega_\sigma\in\mathcal M_\omega\) such that
\begin{equation}
    \left\|
    u_\sigma^{\mathrm{des}}
    -
    \sum_{r=1}^{K_b}
    \omega_{\sigma,r}b_r
    \right\|_{\mathcal V}
    \leq
    \varepsilon_{\mathrm{basis}}.
    \label{eq:basis_approximation_error_cn}
\end{equation}
Together with Eq.~\eqref{eq:realizable_advection_basis_cn}, this yields
\begin{equation}
    \left\|
    u_\sigma^{\mathrm{des}}
    -
    u_\sigma^\omega
    \right\|_{\mathcal V}
    \leq
    \varepsilon_{\mathrm{basis}}.
\end{equation}
\end{proof}

\begin{remark}[Desired advection field and expressive capacity of field bases]
Theorem~\ref{thm:desired_advection_bounded_basis_cn} shows that the desired advection field \(u_\sigma^{\mathrm{des}}\) is a theoretical velocity field recovered from the target density path and the corresponding spatial flux. It represents the ideal migration direction required for the swarm to evolve along the target density path. Because the target density path has a positive lower bound, and because \(J_\sigma\), \(D_\sigma\) and \(\nabla\gamma_\sigma\) are bounded, this theoretical velocity field does not introduce singularities or unbounded control inputs.

The proposed controller does not directly output an arbitrary velocity field. Instead, it generates the realizable advection field \(u_\sigma^\omega\) by taking a convex combination of finitely many physical potential-field bases using \(\omega_\sigma\). Therefore, controllability depends on the expressive capacity of these field bases with respect to \(u_\sigma^{\mathrm{des}}\). If \(\varepsilon_{\mathrm{basis}}=0\), the desired advection field can be represented exactly. If \(\varepsilon_{\mathrm{basis}}>0\), the target advection field can only be approximated, and this approximation error affects the final density approximation accuracy. At the same time, the boundedness of the field bases and the simplex constraint on the weights ensure that \(u_\sigma^\omega\) remains bounded, providing the basis for the subsequent controllability and boundedness analyses.
\end{remark}

\begin{theorem}[Conditional controllability of the Macro-ADR representation]
\label{tr:controllability}
Under the conditions that the target density path is regular, phase-mass transfer is realizable, the spatial flux is constructible, the desired advection field is bounded and the field bases have sufficient expressive capacity for advection, there exists a bounded physical parameter trajectory
\[
    P(t)=\{\omega(t),D(t),\lambda(t)\}\in\mathcal{M},
\]
such that the system can drive the swarm density from the initial density \(\rho^0\) to a neighborhood of the target density \(\rho^\ast\):
\begin{equation}
    \|\rho(\cdot,T)-\rho^\ast\|
    \leq
    C_T\varepsilon_{\mathrm{basis}},
    \label{eq:adr_controllability_bound_cn}
\end{equation}
where \(\varepsilon_{\mathrm{basis}}\) denotes the approximation error of the finite field bases with respect to the desired advection field, and \(C_T>0\) is a stability constant of the system over the finite time interval. In particular, when \(\varepsilon_{\mathrm{basis}}=0\), the system is exactly controllable within the physically expressible space:
\begin{equation}
    \rho(\cdot,T)=\rho^\ast .
    \label{eq:exact_controllability_cn}
\end{equation}
\end{theorem}

\begin{proof}[\textbf{Proof of Theorem~\ref{tr:controllability}:}]
By the target-density-path assumption (Assumption~\ref{assump:target_density_path}), there exists a regular path connecting the initial and target densities:
\begin{equation}
    \gamma(\cdot,0)=\rho^0,
    \qquad
    \gamma(\cdot,T)=\rho^\ast .
    \label{eq:target_path_endpoint_cn}
\end{equation}
By the strong connectivity of the behavioral phase-transition graph (Assumption~\ref{assump:phase_graph}), one can choose non-negative and bounded reaction-rate trajectories
\(0\leq\lambda_{mn}(t)\leq\lambda_{\max}\)
such that the reaction term \(R_\sigma(\gamma)\) matches the phase-wise mass variation induced by the target path. Therefore, \(\lambda(t)\) provides a controllable channel for phase-mass redistribution.

Once the phase-mass variation is matched by the reaction term, the constructibility of the spatial flux (Assumption~\ref{assump:space_flux_bounded}) guarantees the existence of a bounded spatial flux \(J_\sigma\) satisfying the zero-flux boundary condition, so that the target path satisfies the conservation form
\begin{equation}
    \frac{\partial \gamma_\sigma}{\partial t}
    +
    \nabla\cdot J_\sigma
    =
    R_\sigma(\gamma),
    \qquad
    \sigma=1,\ldots,M .
    \label{eq:target_path_conservation_for_ctrl_cn}
\end{equation}
Together with the flux relation
\begin{equation}
    J_\sigma
    =
    u_\sigma\gamma_\sigma
    -
    D_\sigma\nabla\gamma_\sigma ,
\end{equation}
this yields the desired advection field \(u_\sigma^{\mathrm{des}}\) required to realize the target density path. By Theorem~\ref{thm:desired_advection_bounded_basis_cn}, this desired advection field is bounded. Moreover, there exists a set of field-basis weights \(\omega_\sigma(t)\in\mathcal{M}_\omega\) such that the realizable advection field \(u_\sigma^\omega\) generated by the field bases satisfies
\begin{equation}
    \|u_\sigma^\omega-u_\sigma^{\mathrm{des}}\|
    \leq
    \varepsilon_{\mathrm{basis}} .
    \label{eq:basis_approx_for_ctrl_cn}
\end{equation}
Because \(\omega(t)\), \(D(t)\) and \(\lambda(t)\) are all constrained to the bounded physical parameter manifold \(\mathcal{M}\), the constructed parameter trajectory
\(P(t)=\{\omega(t),D(t),\lambda(t)\}\)
is physically feasible and bounded.

Let \(\rho^\omega(x,t)\) denote the solution of the system induced by this feasible parameter trajectory, and define its error relative to the target path as
\begin{equation}
    e(x,t)
    =
    \rho^\omega(x,t)-\gamma(x,t).
    \label{eq:controllability_error_cn}
\end{equation}
Under bounded ADR coefficients and local Lipschitz continuity, the finite-time stability estimate for parabolic ADR systems gives
\begin{equation}
    \frac{d}{dt}
    \|e(t)\|^2
    \leq
    C_1\|e(t)\|^2
    +
    C_2
    \|u^\omega-u^{\mathrm{des}}\|^2,
    \label{eq:controllability_stability_estimate_cn}
\end{equation}
where \(C_1\) and \(C_2\) are constants depending on the bounded system coefficients. By Gr\"onwall's inequality~\cite{chandrasekhar1943stochastic,cosner2014reaction},
\begin{equation}
    \|e(T)\|
    \leq
    C_T
    \|u^\omega-u^{\mathrm{des}}\|.
    \label{eq:controllability_gronwall_cn}
\end{equation}
Combining this with the field-basis approximation error in Eq.~\eqref{eq:basis_approx_for_ctrl_cn} and \(\gamma(\cdot,T)=\rho^\ast\), we obtain
\begin{equation}
    \|\rho^\omega(\cdot,T)-\rho^\ast\|
    =
    \|e(T)\|
    \leq
    C_T\varepsilon_{\mathrm{basis}}.
    \label{eq:controllability_final_bound_cn}
\end{equation}
Therefore, the Macro-ADR behavior representation is approximately controllable within the physically expressible space spanned by the field bases.

When \(\varepsilon_{\mathrm{basis}}=0\), we have \(u^\omega=u^{\mathrm{des}}\). Eq.~\eqref{eq:controllability_gronwall_cn} then gives
\begin{equation}
    \|\rho^\omega(\cdot,T)-\rho^\ast\|=0.
    \label{eq:controllability_exact_result_cn}
\end{equation}
Thus, the system is exactly controllable within the physically expressible space.
\end{proof}

\begin{remark}
Theorem~\ref{tr:controllability} shows that the conditional controllability of PhySwarm arises from three complementary physical control channels. The reaction parameter \(\lambda\) regulates mass redistribution among behavioral phases and represents phase transitions in multi-stage behaviors. The diffusion parameter \(D\) regulates spatial smoothing, coverage and local pressure release within each phase. The advection weights \(\omega\) approximate the spatial migration direction required by the target density path by combining a finite set of physical potential-field bases. Therefore, the proposed method does not directly control individual robots in an unconstrained action space. Instead, it controls directed migration, spatial diffusion and behavioral phase transitions through \(\omega\), \(D\) and \(\lambda\) within a bounded physical parameter manifold, thereby providing a unified representation and conditional controllability of multi-stage swarm emergent behaviors.
\end{remark}

\clearpage

\subsection{Convergence of the controller}
\label{sec:convergence}

The preceding section has shown that the PhySwarm representation of multi-stage swarm emergent behaviors is conditionally controllable within the physically expressible space. This section further provides a conditional convergence analysis of the Neural-Physics Controller (NPC), explaining how the controller can learn physical parameter trajectories that satisfy the controllability requirements of PhySwarm. Specifically, we interpret the training process of the NPC as a stochastic approximation procedure over a bounded physical parameter manifold. We then show that, under the existence of a physically feasible reference trajectory, bounded optimization error and controlled PINN physics residuals~\cite{cybenko1989approximation,hornik1989multilayer}, the learned parameter trajectory can approximate the ideal parameter trajectory constructed in PhySwarm controllability analysis, thereby inducing swarm density evolution close to the target density path.

\subsubsection{Notations and assumptions}

For robot \(i\), the Neural-Physics Controller takes the local observation \(O_i(t)\) and temporal memory \(h_{i,t-1}\) as inputs, and outputs the Macro-ADR physical parameters:
\begin{equation}
    P_{\theta,i}(t)
    =
    F_{\theta}\big(O_i(t),h_{i,t-1}\big)
    =
    \{\omega_{\theta,i}(t),D_{\theta,i}(t),\lambda_{\theta,i}(t)\}.
    \label{eq:mappo_pinn_agent_parameter_cn}
\end{equation}
Here, \(\theta\) denotes the trainable parameters of the NPC, \(\omega_{\theta,i}(t)\) denotes the weights of the advection potential-field bases, \(D_{\theta,i}(t)\) denotes the diffusion coefficient, and \(\lambda_{\theta,i}(t)\) denotes the behavioral phase-transition rates.

Unlike reinforcement-learning policies that directly output unconstrained actions, the controller output in this work is first mapped by a physical projection layer onto the bounded physical parameter manifold (see Eqs.~\eqref{eq:manifold} and \eqref{eq:manifold_projection}):
\begin{equation}
    P_{\theta,i}(t)
    \in
    \mathcal{M}
    =
    \mathcal{M}_{\omega}
    \times
    \mathcal{M}_{D}
    \times
    \mathcal{M}_{\lambda}.
    \label{eq:mappo_pinn_parameter_manifold_cn}
\end{equation}
Therefore, for any trainable parameter \(\theta\), any time \(t\in[0,T]\) and any robot \(i\),
\begin{equation}
    \omega_{\theta,i}(t)\in\mathcal{M}_{\omega},
    \qquad
    D_{\theta,i}(t)\in\mathcal{M}_{D},
    \qquad
    \lambda_{\theta,i}(t)\in\mathcal{M}_{\lambda}.
    \label{eq:mappo_pinn_parameter_feasibility_cn}
\end{equation}
This means that the learning process of the NPC always takes place within the physically feasible parameter space.

At the continuum-field level, the density solution induced by the Macro-ADR parameter field generated by the controller is denoted by \(\rho_{\theta}(x,t)\). Given the parameter field \(P_{\theta}(t)\), the density of behavioral phase \(\sigma\) evolves according to
\begin{equation}
    \frac{\partial \rho_{\theta,\sigma}}{\partial t}
    +
    \nabla\cdot
    \left(
        u_{\theta,\sigma}\rho_{\theta,\sigma}
    \right)
    =
    \nabla\cdot
    \left(
        D_{\theta,\sigma}\nabla\rho_{\theta,\sigma}
    \right)
    +
    R_{\theta,\sigma}(\rho_{\theta}),
    \qquad
    \sigma=1,\ldots,M.
    \label{eq:mappo_pinn_induced_adr_cn}
\end{equation}
where
\begin{equation}
    u_{\theta,\sigma}(x,t)
    =
    \sum_{r=1}^{K_b}
    \omega_{\theta,\sigma,r}(x,t)b_r(x,t),
    \qquad
    b_r(x,t)=-\nabla\Phi_r(x,t).
    \label{eq:mappo_pinn_induced_velocity_cn}
\end{equation}
In practical training, the continuous density \(\rho_{\theta}\) is usually approximated from robot trajectories or simulation samples using kernel density estimation, grid-based statistics or particle approximations~\cite{botev2010kernel,arulampalam2002particle}. The PINN physics residuals used below are computed from this estimated density. We denote the estimated density as
\begin{equation}
    \widehat{\rho}_{\theta}(x,t)
    =
    [\widehat{\rho}_{\theta,1}(x,t),\ldots,\widehat{\rho}_{\theta,M}(x,t)]^{\top}.
    \label{eq:mappo_pinn_estimated_density_cn}
\end{equation}

We use MAPPO as the multi-agent policy optimizer~\cite{schulman2017ppo,yu2022mappo}, and incorporate PINN-based physics-consistency residuals~\cite{karniadakis2021physics}. Task-reward maximization and physical regularization are combined in the following total loss:
\begin{equation}
    \min_{\theta}
    L_{\mathrm{total}}(\theta)
    =
    L_{\mathrm{RL}}(\theta)
    +
    \eta L_{\mathrm{PINN}}(\theta),
    \label{eq:supp_total_loss_cn}
\end{equation}
where \(L_{\mathrm{RL}}\) is the MAPPO policy loss, \(L_{\mathrm{PINN}}\) is the physics-consistency loss, and \(\eta>0\) is the physics-regularization weight.

In the proposed framework, the physics-consistency loss consists of two residual terms:
\begin{equation}
    L_{\mathrm{PINN}}
    =
    L_{\mathrm{micro}}
    +
    \beta L_{\mathrm{macro}},
    \label{eq:supp_pinn_loss_cn}
\end{equation}
where \(\beta>0\) is a weighting coefficient. Here, \(L_{\mathrm{micro}}\) is the microscopic dynamics-consistency residual, corresponding to the Micro-EDM consistency term, and \(L_{\mathrm{macro}}\) is the macroscopic ADR residual, corresponding to the Macro-ADR consistency term.

The microscopic dynamics-consistency loss \(L_{\mathrm{micro}}\) constrains the actual robot velocity to be consistent with the advection--diffusion velocity induced by the Macro-ADR parameters:
\begin{equation}
    L_{\mathrm{micro}}
    =
    \frac{1}{N}
    \sum_{i=1}^{N}
    \left\|
    v_{i,\mathrm{actual}}
    -
    \left(
    v_{\mathrm{adv},i}(\omega_i^\theta)
    -
    \frac{D_i^\theta}{\hat{\rho}(x_i,t)+\varepsilon}
    \nabla\hat{\rho}(x_i,t)
    \right)
    \right\|^2 .
    \label{eq:supp_L_dyn_cn}
\end{equation}
Here, \(\hat{\rho}\) is the density field estimated from robot positions, and \(\varepsilon>0\) is used to avoid division-by-zero singularities.

The macroscopic ADR residual \(L_{\mathrm{macro}}\) constrains the parameters generated by the controller to satisfy the multi-phase Macro-ADR equation:
\begin{equation}
    L_{\mathrm{macro}}
    =
    \sum_{\sigma=1}^{M}
    \left\|
    \frac{\partial \hat{\rho}_{\sigma}}{\partial t}
    +
    \nabla\cdot
    \left(
    u_{\sigma}^{\theta}\hat{\rho}_{\sigma}
    \right)
    -
    \nabla\cdot
    \left(
    D_{\sigma}^{\theta}\nabla \hat{\rho}_{\sigma}
    \right)
    -
    \left[
    \Lambda(\lambda^\theta)\hat{\rho}
    \right]_{\sigma}
    \right\|_{L^2(\Omega\times[0,T])}^{2}.
    \label{eq:supp_L_ADR_cn}
\end{equation}
Thus, \(L_{\mathrm{RL}}\) improves task performance, whereas \(L_{\mathrm{PINN}}\) restricts policy search to a neighborhood of the PhySwarm-consistent parameter manifold. Together, these two terms determine the learning direction of the Neural-Physics Controller.

\begin{assumption}[Bounded observation space and bounded physical parameter space~\cite{alizadeh2023power}]
\label{assump:supp_compact_spaces_cn}
The robot observation space \(\mathcal{O}\), the memory-state space \(\mathcal{H}\), and the physical parameter manifold \(\mathcal{M}\) are compact sets. As shown in Sec.~\ref{bound_of_para}, \(\mathcal{M}\) is closed and bounded, and is therefore compact.
\end{assumption}

\begin{assumption}[Existence of an ideal controllable parameter trajectory]
\label{assump:supp_ideal_parameter_cn}
Under the controllability assumptions, there exists a physically feasible parameter trajectory
\begin{equation}
    P^{\dagger}(t)
    =
    \{\omega^{\dagger}(t),D^{\dagger}(t),\lambda^{\dagger}(t)\}
    \in
    \mathcal{M},
    \qquad
    t\in[0,T],
    \label{eq:ideal_physical_parameter_cn}
\end{equation}
that can drive the system from the initial density \(\rho^{0}\) to a neighborhood of the target density \(\rho^{\ast}\). We refer to this trajectory as the ideal physical parameter trajectory. The density solution induced by this trajectory is denoted by \(\rho^{\dagger}(x,t)\). If the field-basis approximation error is \(\varepsilon_{\mathrm{basis}}\), then
\begin{equation}
    \|\rho^{\dagger}(\cdot,T)-\rho^{\ast}\|_{L^{2}(\Omega)}
    \leq
    C_{\mathrm{basis}}\varepsilon_{\mathrm{basis}},
    \label{eq:ideal_density_target_error_cn}
\end{equation}
where \(C_{\mathrm{basis}}>0\) is a constant related to the stability of the ADR system.
\end{assumption}

\begin{assumption}[Bounded optimization error~\cite{bottou2018optimization,yu2022mappo}]
\label{assump:supp_limited_error_cn}
Let \(\widehat{\theta}\) denote the trained parameter. We assume that, after finite training, \(\widehat{\theta}\) satisfies
\begin{equation}
    \mathcal{L}_{\mathrm{total}}(\widehat{\theta})
    \leq
    \mathcal{L}_{\mathrm{total}}(\theta^{\dagger})
    +
    \varepsilon_{\mathrm{opt}}.
    \label{eq:finite_optimization_error_cn}
\end{equation}
Here, \(\theta^{\dagger}\) denotes the
physically feasible parameter, \(\varepsilon_{\mathrm{opt}}\geq0\) denotes the training error caused by non-convex optimization, finite sampling, gradient noise and function-approximation error. This condition means that NPC training can find a local approximate solution whose total loss is not much worse than that of a feasible reference solution.
\end{assumption}

\begin{assumption}[Local parameter identifiability~\cite{kitamura1977identifiability,mishra2022inverse}]
\label{assump:supp_identifiability_cn}
In a neighborhood \(\mathcal{U}(P^{\dagger})\) of the physically feasible reference parameter trajectory \(P^{\dagger}\), we assume that the Macro-ADR physics residual is locally identifiable with respect to parameter perturbations. Specifically, there exist constants \(c_{\mathrm{id}}>0\) and \(\varepsilon_{\mathrm{id}}\geq0\) such that, for any
\(P_{\theta}\in\mathcal{U}(P^{\dagger})\),
\begin{equation}
    \left\|
    \mathcal{R}(P_{\theta})
    -
    \mathcal{R}(P^{\dagger})
    \right\|_{\mathcal{Y}}
    \geq
    c_{\mathrm{id}}
    \left\|
    P_{\theta}-P^{\dagger}
    \right\|_{\mathcal{H}_{P}}
    -
    \varepsilon_{\mathrm{id}}.
    \label{eq:local_identifiability_cn}
\end{equation}
Equivalently, there exists a constant \(C_{P}=1/c_{\mathrm{id}}\) such that
\begin{equation}
    \left\|
    P_{\theta}-P^{\dagger}
    \right\|_{\mathcal{H}_{P}}
    \leq
    C_{P}
    \left(
    \left\|
    \mathcal{R}(P_{\theta})
    \right\|_{\mathcal{Y}}
    +
    \left\|
    \mathcal{R}(P^{\dagger})
    \right\|_{\mathcal{Y}}
    +
    \varepsilon_{\mathrm{id}}
    \right).
    \label{eq:residual_to_parameter_bound_cn}
\end{equation}
If the reference trajectory satisfies
\(\|\mathcal{R}(P^{\dagger})\|_{\mathcal{Y}}\leq \varepsilon_{\mathrm{phys}}\),
then
\begin{equation}
    \left\|
    P_{\theta}-P^{\dagger}
    \right\|_{\mathcal{H}_{P}}
    \leq
    C_{P}
    \left(
    \left\|
    \mathcal{R}(P_{\theta})
    \right\|_{\mathcal{Y}}
    +
    \varepsilon_{\mathrm{phys}}
    +
    \varepsilon_{\mathrm{id}}
    \right).
    \label{eq:residual_to_parameter_with_phys_error_cn}
\end{equation}
\end{assumption}

\begin{assumption}[Stability of Macro-ADR solutions under parameter perturbations~\cite{chandrasekhar1943stochastic,cosner2014reaction,mishra2023generalization}]
\label{assump:supp_pde_stability_cn}
Under the given initial condition and zero-flux boundary condition, suppose two parameter trajectories \(P\) and \(\tilde{P}\) both belong to the bounded physical parameter manifold \(\mathcal{M}\), and the corresponding Macro-ADR solutions are \(\rho^P\) and \(\rho^{\tilde{P}}\), respectively. Then there exists a constant \(C_T>0\) such that
\begin{equation}
    \|\rho^P(\cdot,T)-\rho^{\tilde{P}}(\cdot,T)\|_{L^2(\Omega)}
    \leq
    C_T
    \|P-\tilde{P}\|_{\mathcal{H}_P}    
    +
    \varepsilon_{\mathrm{KDE}}
    +
    \varepsilon_{\mathrm{disc}}.
    \label{eq:supp_pde_stability_cn}
\end{equation}
Because practical training uses finite-sample density estimation and numerical discretization, the bound explicitly includes the density-estimation error \(\varepsilon_{\mathrm{KDE}}\) and the numerical discretization error \(\varepsilon_{\mathrm{disc}}\).
\end{assumption}

\begin{remark}[On parameter identifiability, Macro-ADR solution stability and density stability]
Local parameter identifiability, stability of Macro-ADR solutions under parameter perturbations, and stability from physics residuals to density-solution errors play different roles in the convergence analysis. Local parameter identifiability is an inverse-problem condition. It states that, in a local neighborhood of the ideal physical parameter trajectory \(P^{\dagger}\), a small PINN physics residual can constrain the Macro-ADR parameter error, that is,
\[
    \mathcal{L}_{\mathrm{PINN}}(\theta)\ \text{small}
    \Rightarrow
    P_{\theta}\approx P^{\dagger}.
\]
Because different combinations of \(\omega\), \(D\) and \(\lambda\) may induce similar macroscopic density evolution, we do not require global uniqueness of the parameters. Instead, the local non-identifiability error is represented by \(\varepsilon_{\mathrm{id}}\).

Stability of Macro-ADR solutions under parameter perturbations is a continuous-dependence condition for the forward problem. It further implies that closeness of parameter trajectories leads to closeness of density solutions:
\[
    P_{\theta}\approx P^{\dagger}
    \Rightarrow
    \rho_{\theta}\approx \rho^{\dagger}.
\]
Therefore, this condition is not a repetition of local parameter identifiability, but rather the intermediate link between parameter convergence and density convergence.
\end{remark}

\clearpage

\subsubsection{Convergence of the neural-physics controller}
\label{subsec:supp_main_theorem_cn}

As shown above, the physical-manifold projection ensures that the output of the Neural-Physics Controller (NPC) always lies in the bounded physical parameter manifold. Therefore, this section focuses on how the NPC learns parameters that are close to the ideal Macro-ADR controllable trajectory within this feasible parameter space. The analysis proceeds in three steps. First, under the existence of a low-residual physically feasible reference policy and bounded optimization error, we establish an upper bound on the PINN residual of the NPC. This shows that the deviation of the learned policy from the Macro-ADR equation can be explicitly controlled by \(\varepsilon_{\mathrm{phys}}\), \(\varepsilon_{\mathrm{opt}}\), \(B_{\mathrm{RL}}\) and the physics weight \(\eta\). Second, using local parameter identifiability, we translate the controlled physics residual into an error bound between the learned parameter trajectory and the ideal controllable parameter trajectory. Finally, using the continuous dependence of Macro-ADR solutions on parameter perturbations, we propagate the parameter error to the density-evolution error, thereby proving that the learned swarm density can approximate the target density path constructed in the controllability analysis.

\begin{theorem}[Physics-residual bound of the NPC]
\label{thm:supp_physical_residual_bound_cn}
Let the joint loss of the NPC be
\begin{equation}
    L_{\mathrm{total}}(\theta)
    =
    L_{\mathrm{RL}}(\theta)
    +
    \eta L_{\mathrm{PINN}}(\theta),
    \qquad
    \eta>0.
    \label{eq:supp_total_loss_cn}
\end{equation}
Assume that there exists a physically feasible reference policy \(\theta^\dagger\) whose PINN physics residual satisfies
\begin{equation}
    L_{\mathrm{PINN}}(\theta^\dagger)
    \leq
    \varepsilon_{\mathrm{phys}}.
    \label{eq:supp_feasible_policy_cn}
\end{equation}
If the trained policy parameter \(\hat{\theta}\) satisfies the bounded optimization-error condition in Assumption~\ref{assump:supp_limited_error_cn}, and if the reinforcement-learning loss difference is bounded, namely there exists a constant \(B_{\mathrm{RL}}>0\) such that
\begin{equation}
    L_{\mathrm{RL}}(\theta^\dagger)
    -
    L_{\mathrm{RL}}(\hat{\theta})
    \leq
    B_{\mathrm{RL}},
    \label{eq:supp_RL_loss_bound_cn}
\end{equation}
then the learned policy satisfies the following physics-residual bound:
\begin{equation}
    L_{\mathrm{PINN}}(\hat{\theta})
    \leq
    \varepsilon_{\mathrm{phys}}
    +
    \frac{
    B_{\mathrm{RL}}
    +
    \varepsilon_{\mathrm{opt}}
    }{\eta}.
    \label{eq:supp_pinn_residual_bound_cn}
\end{equation}
\end{theorem}

\begin{proof}[\textbf{Proof of Theorem~\ref{thm:supp_physical_residual_bound_cn}:}]
By the definition of the joint loss in Eq.~\eqref{eq:supp_total_loss_cn} and the bounded optimization-error condition in Eq.~\eqref{eq:finite_optimization_error_cn}, we have
\begin{align}
    L_{\mathrm{RL}}(\hat{\theta})
    +
    \eta L_{\mathrm{PINN}}(\hat{\theta})
    &\leq
    L_{\mathrm{RL}}(\theta^\dagger)
    +
    \eta L_{\mathrm{PINN}}(\theta^\dagger)
    +
    \varepsilon_{\mathrm{opt}}.
    \label{eq:supp_physical_residual_proof_1_cn}
\end{align}
Moving \(L_{\mathrm{RL}}(\hat{\theta})\) to the right-hand side gives
\begin{align}
    \eta L_{\mathrm{PINN}}(\hat{\theta})
    &\leq
    L_{\mathrm{RL}}(\theta^\dagger)
    -
    L_{\mathrm{RL}}(\hat{\theta})
    +
    \eta L_{\mathrm{PINN}}(\theta^\dagger)
    +
    \varepsilon_{\mathrm{opt}}.
    \label{eq:supp_physical_residual_proof_2_cn}
\end{align}
Using \(L_{\mathrm{PINN}}(\theta^\dagger)\leq\varepsilon_{\mathrm{phys}}\) and
\(L_{\mathrm{RL}}(\theta^\dagger)-L_{\mathrm{RL}}(\hat{\theta})\leq B_{\mathrm{RL}}\), we obtain
\begin{equation}
    \eta L_{\mathrm{PINN}}(\hat{\theta})
    \leq
    B_{\mathrm{RL}}
    +
    \eta\varepsilon_{\mathrm{phys}}
    +
    \varepsilon_{\mathrm{opt}}.
    \label{eq:supp_physical_residual_proof_3_cn}
\end{equation}
Dividing both sides by \(\eta\) yields
\begin{equation}
    L_{\mathrm{PINN}}(\hat{\theta})
    \leq
    \varepsilon_{\mathrm{phys}}
    +
    \frac{
    B_{\mathrm{RL}}
    +
    \varepsilon_{\mathrm{opt}}
    }{\eta}.
\end{equation}
\end{proof}

\begin{remark}
Theorem~\ref{thm:supp_physical_residual_bound_cn} shows that whether the policy learned by the NPC has a small physics residual depends on three factors: the physical feasibility of the reference policy \(\varepsilon_{\mathrm{phys}}\), the training optimization error \(\varepsilon_{\mathrm{opt}}\), and the weighting strength \(\eta\) between the reinforcement-learning objective and the physics residual. When a low-residual reference policy exists, the optimization error is sufficiently small and the physics-regularization weight \(\eta\) is sufficiently large, the PINN residual of the learned policy is also controlled. Thus, \(L_{\mathrm{PINN}}\) is not merely an empirical regularizer, but provides a quantifiable upper bound for the physical consistency of the NPC.
\end{remark}

\begin{theorem}[Conditional convergence of neural-physical parameters]
\label{thm:supp_parameter_convergence_cn}
Assume that the trained policy \(\hat{\theta}\) satisfies Theorem~\ref{thm:supp_physical_residual_bound_cn}. Then the learned ADR parameters \(P^{\hat{\theta}}\) and the ideal parameters \(P^\dagger\) satisfy
\begin{equation}
    \left\|
    P^{\hat{\theta}}
    -
    P^\dagger
    \right\|_{\mathcal{H}_P}
    \leq
    \sqrt{
    \frac{1}{c_P}
    \left(
    \varepsilon_{\mathrm{phys}}
    +
    \frac{B_{\mathrm{RL}}+\varepsilon_{\mathrm{opt}}}{\eta}
    +
    \varepsilon_{\mathrm{id}}
    \right)
    } .
    \label{eq:supp_parameter_convergence_bound_cn}
\end{equation}
\end{theorem}

\begin{proof}[\textbf{Proof of Theorem~\ref{thm:supp_parameter_convergence_cn}:}]
By the local parameter-identifiability assumption (Assumption~\ref{assump:supp_identifiability_cn}), we have
\begin{equation}
    L_{\mathrm{PINN}}(P^{\hat{\theta}})
    \geq
    c_P
    \left\|
    P^{\hat{\theta}}-P^\dagger
    \right\|_{\mathcal{H}_P}^{2}
    -
    \varepsilon_{\mathrm{id}} .
    \label{eq:supp_parameter_proof_1_cn}
\end{equation}
Therefore,
\begin{equation}
    c_P
    \left\|
    P^{\hat{\theta}}-P^\dagger
    \right\|_{\mathcal{H}_P}^{2}
    \leq
    L_{\mathrm{PINN}}(P^{\hat{\theta}})
    +
    \varepsilon_{\mathrm{id}} .
    \label{eq:supp_parameter_proof_2_cn}
\end{equation}
By Theorem~\ref{thm:supp_physical_residual_bound_cn},
\begin{equation}
    L_{\mathrm{PINN}}(P^{\hat{\theta}})
    \leq
    \varepsilon_{\mathrm{phys}}
    +
    \frac{B_{\mathrm{RL}}+\varepsilon_{\mathrm{opt}}}{\eta}.
    \label{eq:supp_parameter_proof_3_cn}
\end{equation}
Substituting Eq.~\eqref{eq:supp_parameter_proof_3_cn} into Eq.~\eqref{eq:supp_parameter_proof_2_cn} yields Eq.~\eqref{eq:supp_parameter_convergence_bound_cn}.
\end{proof}

\begin{remark}
This result shows that the parameters output by the Neural-Physics Controller,
\(P_i(t)=\{\omega_i(t),D_i(t),\lambda_i(t)\}\), can conditionally converge to a neighborhood of the ideal parameter trajectory required by the PhySwarm controllability analysis. The error is jointly determined by the physical residual of the feasible reference policy \(\varepsilon_{\mathrm{phys}}\), the MAPPO optimization error \(\varepsilon_{\mathrm{opt}}\), the physics-regularization weight \(\eta\), and the parameter non-identifiability error \(\varepsilon_{\mathrm{id}}\).
\end{remark}

\begin{theorem}[From parameter convergence to advection-field approximation]
\label{thm:advection_field_approximation_cn}
Let \(Q=\Omega\times[0,T]\). Under bounded potential-field bases, field-basis approximation error \(\varepsilon_{\mathrm{basis}}\), and the conditional convergence of NPC parameters, for each behavioral phase \(\sigma=1,\ldots,M\), the learned advection field \(u_{\sigma}^{\widehat{\theta}}\) satisfies
\begin{equation}
    \left\|
    u_{\sigma}^{\widehat{\theta}}
    -
    u_{\sigma}^{\mathrm{des}}
    \right\|_{L^2(Q)}
    \leq
    \varepsilon_{\mathrm{basis}}
    +
    C_u
    \sqrt{
    \varepsilon_{\mathrm{phys}}
    +
    \frac{B_{\mathrm{RL}}+\varepsilon_{\mathrm{opt}}}{\eta}
    +
    \varepsilon_{\mathrm{id}}
    },
    \label{eq:thm_advection_approx_final_cn}
\end{equation}
where \(u_{\sigma}^{\mathrm{des}}\) is the desired advection field required to realize the target density path, and \(C_u>0\) is a constant depending on the boundedness constant of the field bases, the number of field bases and the parameter-convergence constant.
\end{theorem}

\begin{proof}[\textbf{Proof of Theorem~\ref{thm:advection_field_approximation_cn}:}]
By the field-basis advection representation, the advection field generated by parameter \(\theta\) for behavioral phase \(\sigma\) is
\begin{equation}
    u_{\sigma}^{\theta}(x,t)
    =
    \sum_{r=1}^{K_b}
    \omega_{\sigma,r}^{\theta}(x,t)b_r(x,t),
    \qquad
    b_r(x,t)=-\nabla\Phi_r(x,t).
    \label{eq:proof_advection_basis_cn}
\end{equation}
Thus, for any two sets of weights \(\omega_{\sigma}^{\theta}\) and \(\omega_{\sigma}^{\dagger}\), we have
\begin{equation}
    u_{\sigma}^{\theta}
    -
    u_{\sigma}^{\dagger}
    =
    \sum_{r=1}^{K_b}
    \left(
    \omega_{\sigma,r}^{\theta}
    -
    \omega_{\sigma,r}^{\dagger}
    \right)b_r .
    \label{eq:proof_advection_difference_cn}
\end{equation}

Since all potential-field bases are bounded (Assumption~\ref{assumption_boundedness_basis}), the Cauchy--Schwarz inequality gives
\begin{align}
    \left\|
    u_{\sigma}^{\theta}
    -
    u_{\sigma}^{\dagger}
    \right\|_{L^2(Q)}
    &\leq
    C_{\Phi}
    \left\|
    \omega_{\sigma}^{\theta}
    -
    \omega_{\sigma}^{\dagger}
    \right\|_{L^2(Q)},
    \label{eq:proof_u_weight_lipschitz_cn}
\end{align}
where \(C_{\Phi}>0\) is a constant determined by the boundedness constant of the field bases and the number of bases; for example, one may take
\(C_{\Phi}\leq \sqrt{K_b}B_{\Phi}\). Hence, the advection field is Lipschitz continuous with respect to the field weights.

Setting \(\theta=\widehat{\theta}\) and applying the triangle inequality gives
\begin{align}
    \left\|
    u_{\sigma}^{\widehat{\theta}}
    -
    u_{\sigma}^{\mathrm{des}}
    \right\|_{L^2(Q)}
    &\leq
    \left\|
    u_{\sigma}^{\widehat{\theta}}
    -
    u_{\sigma}^{\dagger}
    \right\|_{L^2(Q)}
    +
    \left\|
    u_{\sigma}^{\dagger}
    -
    u_{\sigma}^{\mathrm{des}}
    \right\|_{L^2(Q)}
    \nonumber\\
    &\leq
    C_{\Phi}
    \left\|
    \omega_{\sigma}^{\widehat{\theta}}
    -
    \omega_{\sigma}^{\dagger}
    \right\|_{L^2(Q)}
    +
    \varepsilon_{\mathrm{basis}},
    \label{eq:proof_u_des_bound_cn}
\end{align}
where the second term uses the field-basis approximation error
\[
    \left\|
    u_{\sigma}^{\dagger}
    -
    u_{\sigma}^{\mathrm{des}}
    \right\|_{L^2(Q)}
    \leq
    \varepsilon_{\mathrm{basis}}.
\]

Since \(\omega\) is part of the physical parameter trajectory \(P=\{\omega,D,\lambda\}\), the parameter-convergence result~\cite{ghadimi2013stochastic,raissi2019pinn} implies that there exists a constant \(C_{\omega}>0\) such that
\begin{equation}
    \left\|
    \omega_{\sigma}^{\widehat{\theta}}
    -
    \omega_{\sigma}^{\dagger}
    \right\|_{L^2(Q)}
    \leq
    C_{\omega}
    \sqrt{
    \varepsilon_{\mathrm{phys}}
    +
    \frac{B_{\mathrm{RL}}+\varepsilon_{\mathrm{opt}}}{\eta}
    +
    \varepsilon_{\mathrm{id}}
    }.
    \label{eq:proof_weight_convergence_cn}
\end{equation}
Substituting Eq.~\eqref{eq:proof_weight_convergence_cn} into Eq.~\eqref{eq:proof_u_des_bound_cn}, and letting
\(C_u=C_{\Phi}C_{\omega}\), yields
\begin{equation}
    \left\|
    u_{\sigma}^{\widehat{\theta}}
    -
    u_{\sigma}^{\mathrm{des}}
    \right\|_{L^2(Q)}
    \leq
    \varepsilon_{\mathrm{basis}}
    +
    C_u
    \sqrt{
    \varepsilon_{\mathrm{phys}}
    +
    \frac{B_{\mathrm{RL}}+\varepsilon_{\mathrm{opt}}}{\eta}
    +
    \varepsilon_{\mathrm{id}}
    }.
\end{equation}
\end{proof}

\begin{remark}
Theorem~\ref{thm:advection_field_approximation_cn} shows that the error between the learned advection field and the desired advection field has two sources. The first is \(\varepsilon_{\mathrm{basis}}\), which represents the approximation error of the finite potential-field bases with respect to the desired advection field. The second comes from the learning error of the NPC relative to the ideal physical parameter trajectory, which is controlled by the physics-residual error \(\varepsilon_{\mathrm{phys}}\), the optimization error \(\varepsilon_{\mathrm{opt}}\), the physics-residual weight \(\eta\), and the parameter non-identifiability error \(\varepsilon_{\mathrm{id}}\). Therefore, the ability of the NPC to approximate the desired advection field is not guaranteed by MAPPO alone, but by the combination of field-basis expressivity, parameter convergence, PINN physics-residual constraints and physical-manifold projection.
\end{remark}

\begin{theorem}[Conditional convergence of the NPC: from parameter convergence to target-density approximation]
\label{thm:density_approximation_from_parameter_convergence_cn}
Let \(P^\dagger(t)\) be the ideal physical parameter trajectory constructed in the Macro-ADR controllability analysis, and let \(P^{\hat{\theta}}(t)\) be the parameter trajectory learned by the NPC. Denote the corresponding Macro-ADR solutions by \(\rho^{P^\dagger}(x,t)\) and \(\rho^{\hat{\theta}}(x,t)\), respectively. The terminal density induced by the learned NPC parameter trajectory satisfies
\begin{equation}
    \left\|
    \rho^{\hat{\theta}}(\cdot,T)
    -
    \rho^\ast
    \right\|_{L^2(\Omega)}
    \leq
    C_T
    \left[
    \varepsilon_{\mathrm{basis}}
    +
    \sqrt{
    \varepsilon_{\mathrm{phys}}
    +
    \frac{B_{\mathrm{RL}}+\varepsilon_{\mathrm{opt}}}{\eta}
    +
    \varepsilon_{\mathrm{id}}
    }
    +
    \varepsilon_{\mathrm{KDE}}
    +
    \varepsilon_{\mathrm{disc}}
    \right].
    \label{eq:final_density_bound_cn}
\end{equation}
Here, \(C_T>0\) is a uniform stability constant over the finite time interval, which may absorb constants associated with ADR stability, parameter convergence and boundedness of the field bases.
\end{theorem}

\begin{proof}[\textbf{Proof of Theorem~\ref{thm:density_approximation_from_parameter_convergence_cn}:}]
By the existence assumption of the ideal controllable parameter trajectory (Assumption~\ref{assump:supp_ideal_parameter_cn}), we have
\begin{equation}
    \left\|
    \rho^{P^\dagger}(\cdot,T)
    -
    \rho^\ast
    \right\|_{L^2(\Omega)}
    \leq
    C_T\varepsilon_{\mathrm{basis}}.
    \label{eq:ideal_density_basis_error_reuse_cn}
\end{equation}
By the conditional convergence theorem of neural-physical parameters (Theorem~\ref{thm:supp_parameter_convergence_cn}), we have
\begin{equation}
    \left\|
    P^{\hat{\theta}}
    -
    P^\dagger
    \right\|_{\mathcal{H}_P}
    \leq
    C_T
    \sqrt{
    \varepsilon_{\mathrm{phys}}
    +
    \frac{B_{\mathrm{RL}}+\varepsilon_{\mathrm{opt}}}{\eta}
    +
    \varepsilon_{\mathrm{id}}
    }.
    \label{eq:parameter_convergence_reuse_cn}
\end{equation}
By the triangle inequality,
\begin{align}
    \left\|
    \rho^{\hat{\theta}}(\cdot,T)
    -
    \rho^\ast
    \right\|_{L^2(\Omega)}
    &\leq
    \left\|
    \rho^{\hat{\theta}}(\cdot,T)
    -
    \rho^{P^\dagger}(\cdot,T)
    \right\|_{L^2(\Omega)}
    \nonumber\\
    &\quad+
    \left\|
    \rho^{P^\dagger}(\cdot,T)
    -
    \rho^\ast
    \right\|_{L^2(\Omega)}.
    \label{eq:density_triangle_reuse_cn}
\end{align}

Using the stability of ADR solutions under parameter perturbations (Assumption~\ref{assump:supp_pde_stability_cn}), we obtain
\begin{equation}
    \left\|
    \rho^{\hat{\theta}}(\cdot,T)
    -
    \rho^{P^\dagger}(\cdot,T)
    \right\|_{L^2(\Omega)}
    \leq
    C_T
    \left\|
    P^{\hat{\theta}}
    -
    P^\dagger
    \right\|_{\mathcal{H}_P}
    +
    C_T
    \left(
    \varepsilon_{\mathrm{KDE}}
    +
    \varepsilon_{\mathrm{disc}}
    \right).
    \label{eq:density_stability_reuse_cn}
\end{equation}
Substituting the parameter-convergence bound in Eq.~\eqref{eq:parameter_convergence_reuse_cn} into Eq.~\eqref{eq:density_stability_reuse_cn}, and combining it with the ADR controllability error bound in Eq.~\eqref{eq:ideal_density_basis_error_reuse_cn}, gives
\begin{equation}
    \left\|
    \rho^{\hat{\theta}}(\cdot,T)
    -
    \rho^\ast
    \right\|_{L^2(\Omega)}
    \leq
    C_T
    \left[
    \varepsilon_{\mathrm{basis}}
    +
    \sqrt{
    \varepsilon_{\mathrm{phys}}
    +
    \frac{B_{\mathrm{RL}}+\varepsilon_{\mathrm{opt}}}{\eta}
    +
    \varepsilon_{\mathrm{id}}
    }
    +
    \varepsilon_{\mathrm{KDE}}
    +
    \varepsilon_{\mathrm{disc}}
    \right].
    \label{eq:final_density_bound_proof_cn}
\end{equation}
\end{proof}

\begin{remark}[Relationship to Macro-ADR controllability and sources of error]
Theorem~\ref{thm:density_approximation_from_parameter_convergence_cn} establishes the main link between the conditional convergence of the NPC and the controllability of the Macro-ADR representation. The Macro-ADR controllability analysis proves the existence of an ideal physical parameter trajectory \(P^\dagger(t)\), showing that, within the physically expressible space, the Macro-ADR system can drive the swarm density towards the target emergent behavior through suitable trajectories of \(\omega\), \(D\) and \(\lambda\). The present section further shows that the Neural-Physics Controller can, under the stated conditions, learn an approximate realization of this ideal parameter trajectory. Thus, the two analyses correspond to two complementary levels of the theoretical framework: the Macro-ADR system provides the macroscopic controllability basis, whereas the NPC provides a learnable mechanism for realizing the required physical parameters.

The final density error is determined by several sources. The term \(\varepsilon_{\mathrm{basis}}\) arises from the approximation error of the finite potential-field bases with respect to the desired advection field, and therefore reflects the expressive capacity and controllability limitation of the Macro-ADR behavior representation itself. The terms \(\varepsilon_{\mathrm{phys}}\) and \(\varepsilon_{\mathrm{opt}}\) arise from the PINN physics residual and the NPC optimization error, respectively, and quantify how accurately the learning process approximates the ideal physical parameter trajectory. The term \(\varepsilon_{\mathrm{id}}\) represents local non-identifiability caused by parameter equivalence, weak excitation or limited observations. The terms \(\varepsilon_{\mathrm{KDE}}\) and \(\varepsilon_{\mathrm{disc}}\) denote the density-estimation error and numerical discretization error, respectively. Therefore, when these error terms are sufficiently small, the physical parameter trajectory learned by the NPC induces a terminal swarm density that approximates the target emergent-behavior density.

In particular, when
\(\varepsilon_{\mathrm{basis}}=\varepsilon_{\mathrm{phys}}=\varepsilon_{\mathrm{opt}}=\varepsilon_{\mathrm{id}}=\varepsilon_{\mathrm{KDE}}=\varepsilon_{\mathrm{disc}}=0\),
the error bound reduces to
\[
    \rho^{\widehat{\theta}}(\cdot,T)=\rho^{\ast},
\]
which means that, under ideal conditions, the learned controller can exactly approximate the target density. This also indicates that improving the expressive capacity of the field bases, increasing the strength of physical regularization, reducing policy-optimization error, enhancing parameter identifiability, improving density-estimation accuracy and adopting more stable numerical discretization schemes can systematically improve the ability of the NPC to approximate target emergent behaviors.
\end{remark}

\clearpage

\section{Supplementary Methods}

\subsection{General modeling guidelines}
\label{sec:general_modeling_procedure_cn}

This section provides a practical modeling guide for translating a concrete swarm-robot task into a multi-stage Macro-ADR behavior representation within PhySwarm (Fig.~\ref{fig:field_modeling_guideline}). The purpose is not to rederive the Macro-ADR equation, but to clarify how task semantics can be systematically converted into behavioral phases, physical field bases, diffusion-regulation fields, phase-transition structures and, ultimately, a physically parameterized model that can be learned and modulated by the Neural-Physics Controller (NPC).

The three benchmark tasks in this paper provide concrete examples of this modeling procedure. They instantiate the same PhySwarm interface with different phase sets, field dictionaries, diffusion-regulation densities, transition graphs and reward objectives. 
These examples illustrate how the proposed guidelines can be used to translate different swarm-task semantics into a unified PhySwarm implementation. Details could be referred to  Sec.~\ref{sec:scenario_specific_details_cn} and Sec.~\ref{sec:mappo_pinn_network_implementation_cn} of the supplementary materials.

\begin{figure}[htbp]
    \centering
    \includegraphics[width=0.8\textwidth]{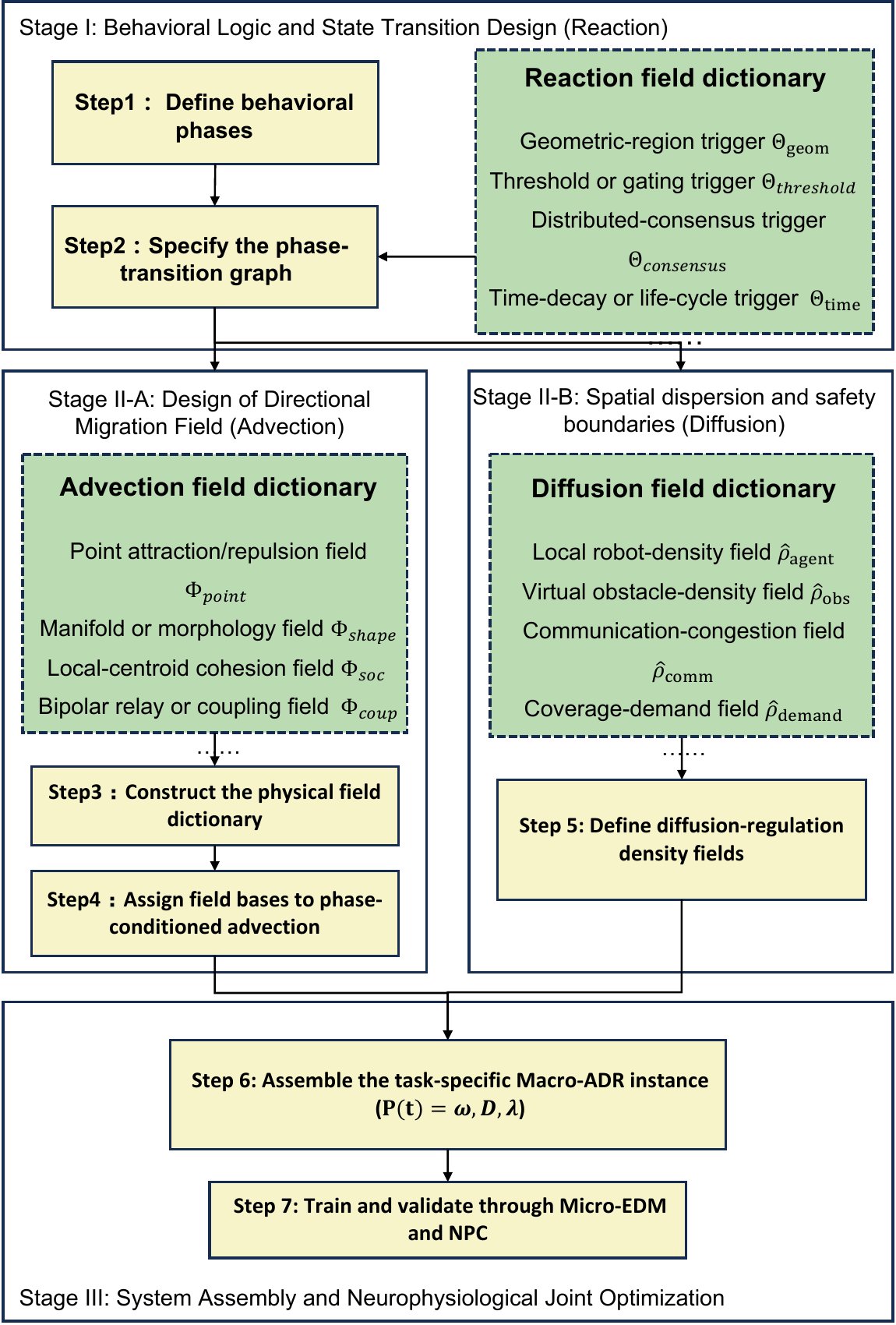} 
    \caption{
    \textbf{Modeling guide for constructing task-specific PhySwarm instances.}
    A swarm-robot task is translated into a multi-stage Macro-ADR representation by defining behavioral phases, constructing physical field bases, specifying diffusion-regulation fields and designing phase-transition rules. The resulting physical parameters are learned and modulated by the NPC.
    }
    \label{fig:field_modeling_guideline}
\end{figure}

In PhySwarm, a complex swarm emergent behavior can be modeled by answering three basic questions:

\begin{itemize}
    \item \textbf{Where should the swarm move?}
    This question is addressed by the advection term, which represents directional transport mechanisms such as target attraction, obstacle avoidance, formation guidance, trail following, etc.

    \item \textbf{How should the swarm regulate its spatial distribution?}
    This question is addressed by the diffusion term, which represents density-regulation mechanisms such as spatial coverage, local collision avoidance, safety-margin maintenance, etc.

    \item \textbf{When should robots switch behavioral phases or functional roles?}
    This question is addressed by the reaction term, which represents transitions among behavioral phases or roles, such as exploration, aggregation, navigation, rescue response, communication relay, etc.
\end{itemize}

Thus, constructing a task-specific PhySwarm instance amounts to designing three coupled components: directional migration, density regulation and phase transition. The advection term determines task-directed swarm motion, the diffusion term regulates the quality of the spatial distribution, and the reaction term organizes switching among behavioral phases.

Operationally, instantiating PhySwarm for a new task can be specified by the tuple
\begin{equation}
    \mathcal{I}_{\mathrm{task}}
    =
    \left(
    \Omega,
    \mathcal{S},
    \mathcal{B},
    \mathcal{G}_{\mathrm{phase}},
    \chi,
    \hat{\rho}_{\mathrm{loc}},
    \mathcal{M}_{P},
    \mathcal{R}_{\mathrm{task}}
    \right),
    \label{eq:task_instantiation_template}
\end{equation}
where \(\Omega\) is the task domain, \(\mathcal{S}\) is the set of behavioral phases, \(\mathcal{B}=\{b_r\}_{r=1}^{K_b}\) is the dictionary of physical field bases, \(\mathcal{G}_{\mathrm{phase}}\) is the admissible phase-transition graph, \(\chi=\{\chi_{mn}\}\) denotes task-dependent activation functions, \(\hat{\rho}_{\mathrm{loc}}\) is the density or generalized density field used for diffusion regulation, \(\mathcal{M}_{P}\) is the bounded physical parameter manifold for \(P(t)=\{\omega(t),D(t),\lambda(t)\}\), and \(\mathcal{R}_{\mathrm{task}}\) is the task reward or evaluation objective used during training. This tuple provides the operational interface between the task description and the implementation described in Methods of the main paper.

\paragraph{Step 1: Define behavioral phases.}
The first step is to decompose the task into a finite set of behavioral phases,
\[
    \mathcal{S}=\{1,\ldots,M\}.
\]
Each phase should correspond to a macroscopic behavior pattern that is observable, interpretable or controllable. For example, in Trail-Guided Swarm Foraging, the phases may include exploration, resource-approach, homing and trail-following. In Formation-Reconfigurable Swarm Navigation, the phases may include formation-keeping, goal-directed navigation, morphology-adaptation and formation-recovery. In Role-Adaptive Swarm Search and Rescue, the phases may include exploration, target-response and communication-relay.

The phase definition should be neither too fine nor too coarse. If the phases are too fine-grained, the transition graph becomes overly complex and increases the difficulty of learning and interpretation. If the phases are too coarse, the model may fail to capture key transitions in multi-stage behavior. A suitable behavioral phase should have a clear task-level interpretation and should be adjustable through advection weights, diffusion coefficients or reaction rates.

The output of this step is the phase set \(\mathcal{S}\), which defines the components of the multi-phase density vector
\[
    \rho(x,t)
    =
    [\rho_1(x,t),\ldots,\rho_M(x,t)]^{\top}.
\]
This density vector is used in the Macro-ADR representation and in the conservative reaction structure.

\paragraph{Step 2: Specify the phase-transition graph.}
After defining the behavioral phases, the next step is to specify the admissible phase-transition graph
\[
    \mathcal{G}_{\mathrm{phase}}
    =
    (\mathcal{S},\mathcal{E}_{\mathrm{phase}}),
\]
where an edge \(m\rightarrow n\) indicates that robots are allowed to switch from phase \(m\) to phase \(n\). For example, in foraging, a robot may switch from exploration to carrying-to-nest after detecting and picking up a resource; after dropping the resource at the nest, it may return to a search or resource-approach phase. In search and rescue, a robot may switch from a searcher phase to a responder phase after detecting a target, while some robots may switch to a relay phase according to communication quality and spatial position.

For each admissible transition \(m\rightarrow n\), a task-dependent activation function \(\chi_{mn}(x,t)\) is defined. These triggers may depend on task events, local observations, communication quality or learned decision signals. The transition intensity is controlled by the reaction rate \(\lambda_{mn}\), which determines how rapidly robots switch from one behavioral phase to another. Because the proposed model adopts a conservative reaction structure, the reaction term only redistributes robots among behavioral phases and does not change the total number of robots.

The output of this step is the transition mask \(A_{mn}\), the activation functions \(\chi_{mn}\), and the reaction-rate variables \(\lambda_{mn}\). Together, they define the reaction structure used in Macro-ADR and the phase-switching rule used in Micro-EDM.

\paragraph{Step 3: Construct the physical field dictionary.}
After defining behavioral phases and transition logic, task-relevant physical fields or field bases are constructed. These fields are not arbitrary neural features; they are interpretable variables directly related to task objectives, environmental constraints and inter-agent interactions. Common field types include:

\begin{itemize}
    \item \textbf{Target-related fields:} attractive fields generated by targets, nests, rescue locations, resource positions or desired motion directions.

    \item \textbf{Obstacle and boundary fields:} repulsive fields generated by obstacles, walls, hazardous regions or task-domain boundaries.

    \item \textbf{Swarm-interaction fields:} local interaction fields describing cohesion, separation, alignment, formation maintenance and connectivity preservation.

    \item \textbf{Information-memory fields:} fields representing trails, pheromone-like memory, target confidence, explored-region maps or communication-quality distributions that evolve with the task.

    \item \textbf{Role-function fields:} fields encoding spatial preferences or organizational relations for functional roles such as searchers, responders and relays.
\end{itemize}

Each field can be represented as a scalar potential field, a vector field or a grid function computed from sensor data, robot positions, communication states and task events. For a scalar potential field \(\Phi_r(x,t)\), the corresponding velocity base is
\[
    b_r(x,t)
    =
    -\nabla \Phi_r(x,t).
\]
Examples of typical field functions are provided in Sec.~\ref{fields_example}.

The output of this step is the physical field dictionary
\[
    \mathcal{B}
    =
    \{b_r(x,t)\}_{r=1}^{K_b},
\]
which provides the reusable basis functions for constructing the advection field.

\paragraph{Step 4: Assign field bases to phase-conditioned advection.}
The advection term determines the dominant direction of swarm motion. The PhySwarm framework does not allow the policy network to directly output unconstrained velocities. Instead, task-relevant motion tendencies are generated through weighted combinations of physical field bases. For behavioral phase \(\sigma\), the advection velocity field is written as
\[
    u_{\sigma}(x,t)
    =
    \sum_{r=1}^{K_b}
    \omega_{\sigma,r}(x,t)b_r(x,t),
\]
where \(\omega_{\sigma,r}(x,t)\) denotes the contribution of field basis \(r\) in phase \(\sigma\).

In task modeling, the designer specifies which physical fields are relevant to each behavioral phase. For example, in foraging, the resource-approach phase is typically dominated by the food-attraction field, whereas the homing phase is dominated by the nest-attraction field. In formation navigation, the goal-directed navigation phase usually requires the simultaneous consideration of a forward target field, a boundary-repulsion field and a morphology-preserving field. The NPC then learns the weights \(\omega\) dynamically from observations, thereby regulating the contribution of different physical mechanisms to the current motion direction.

The output of this step is the phase-conditioned advection field \(u_{\sigma}(x,t)\), which contributes to both the Macro-ADR flux and the Micro-EDM velocity.

\paragraph{Step 5: Define diffusion-regulation density fields.}
The diffusion term regulates local spatial distribution, density equalization and safety margins. In PhySwarm, diffusion is not treated as aimless random perturbation, but as an active spatial-regulation mechanism implemented through density-gradient compensation. A larger diffusion coefficient \(D\) represents stronger spatial spreading and density equalization, which is useful for search, coverage or pressure release in crowded regions. A smaller \(D\) helps preserve compact structures and topological stability, which is important for formation maintenance, narrow-corridor traversal or communication-chain preservation.

For a new task, the designer can flexibly specify the density field by selecting or combining elements from table~\ref{tab:typical_diffusion_fields_cn}. This table provides a rich set of candidate diffusion fields. Depending on the task requirements, the selected diffusion field $\hat{\rho}_{\mathrm{diff}}(x,t)$ can be formulated as a generalized density field constructed from a weighted combination of these basis fields:
\begin{equation}
    \hat{\rho}_{\mathrm{diff}}(x,t)
    =
    \sum_{\hat{\rho}_k \in \mathcal{D}_{\mathrm{diff}}} \beta_k \hat{\rho}_k(x,t),
    \label{eq:generalized_diffusion_field}
\end{equation}
where $\mathcal{D}_{\mathrm{diff}}$ denotes the set of typical diffusion fields, and $\beta_k \ge 0$ are task-dependent weighting coefficients. This basis-based formulation allows the framework to scale seamlessly across diverse spatial-regulation behaviors, ranging from local density equalization to boundary-aware safety margins.

The output of this step is the diffusion-regulation density \(\hat{\rho}_{\mathrm{diff,\sigma}}\) and the diffusion coefficient \(D_{\sigma}\). At the continuum level, they define the diffusion flux
\[
    J_{\mathrm{diff},\sigma}(x,t)
    =
    -
    D_{\sigma}(x,t)\nabla\hat{\rho}_{diff,\sigma}(x,t),
\]
whereas at the robot level they define the density-gradient compensation term in Micro-EDM.

\paragraph{Step 6: Assemble the task-specific Macro-ADR instance.}
After the behavioral phases, field dictionary, diffusion-regulation density and reaction structure have been specified, the task can be assembled into a Macro-ADR representation. The core physical parameters are
\[
    P(t)
    =
    \{\omega(t),D(t),\lambda(t)\},
\]
where \(\omega\) controls the contribution of different physical fields to directed migration, \(D\) controls diffusion strength and spatial regulation, and \(\lambda\) controls transition rates among behavioral phases.

Operationally, the task-specific Macro-ADR instance is defined by
\[
    \frac{\partial \rho_{\sigma}}{\partial t}
     =-
    \nabla\cdot
    \left(
    u_{\sigma}\rho_{\sigma}
    \right)
    +
    \nabla\cdot
    \left(
    D_{\sigma}\nabla\rho_{\sigma}
    \right)
    +
    R_{\sigma}(\rho),
    \qquad
    \sigma=1,\ldots,M,
\]
with
\[
    u_{\sigma}
    =
    \sum_{r=1}^{K_b}\omega_{\sigma,r}b_r,
    \qquad
    R(\rho)
    =
    \Lambda(\lambda)\rho .
\]
The output of this step is a complete Macro-ADR model instance, including the phase densities \(\rho_{\sigma}\), advection fields \(u_{\sigma}\), diffusion coefficients \(D_{\sigma}\), reaction matrix \(\Lambda(\lambda)\) and task-specific activation functions \(\chi_{mn}\).

\paragraph{Step 7: Train and validate through Micro-EDM and NPC.}
At the microscopic execution level, the advection, diffusion and reaction mechanisms defined by Macro-ADR are converted by Micro-EDM into executable local motion and phase switching for individual robots. Each robot responds to the advection field associated with its current phase, performs diffusion regulation through density-gradient compensation, and switches behavioral phases according to reaction rates and triggering conditions.

During training, the NPC learns the physical parameters \(P(t)=\{\omega(t),D(t),\lambda(t)\}\) from robot observations and temporal states. The reinforcement-learning objective improves task performance, whereas the PINN physics residual constrains the learned density evolution and microscopic motion to remain consistent with PhySwarm. Therefore, the NPC does not learn a purely black-box action policy, but physically meaningful control quantities: advection weights, diffusion coefficients and phase-transition rates.

The output of this step is a trained NPC policy and its induced physical parameter trajectory. 

\begin{remark}[Mapping task semantics to physical mechanisms]
The above procedure can be viewed as a modeling compilation process from task semantics to physical dynamics. Task objectives, environmental constraints and inter-agent interactions are encoded as physical field bases, forming the advection mechanism for directed transport. Requirements related to coverage, safety margins, pressure release and density equalization are encoded as diffusion mechanisms. Stage switching, role differentiation and task-triggered events are encoded as reaction mechanisms. The NPC further learns the time-varying weights and intensities of these physical mechanisms, enabling a unified representation, continuous switching and task adaptation of multi-stage swarm emergent behaviors.

It is important to keep the modeling roles of the advection, diffusion and reaction terms distinct in PhySwarm. The advection term describes directional deterministic transport and is suitable for encoding motion preferences such as target attraction, obstacle repulsion, formation guidance, trail following and communication-quality-guided motion. The diffusion term describes density-gradient-driven spatial regulation and is suitable for encoding coverage expansion, local collision avoidance, pressure release and density equalization. The reaction term describes transitions among behavioral phases or functional roles; it does not directly generate spatial potential fields, nor does it represent the creation or disappearance of robots. In practice, the same physical mechanism should not be redundantly encoded into multiple ADR components unless it plays different roles at different modeling levels. Through this separation of roles, different swarm-robot tasks can be represented as interpretable, learnable and physically constrained ADR dynamical models. A checklist is provided in Table~\ref{tab:seb_model_operational_checklist} for reference.
\end{remark}

\begin{table}[htbp]
\centering
\renewcommand{\arraystretch}{1.5}
\caption{
\textbf{Operational checklist for instantiating PhySwarm.}
Each modeling step produces a specific object that is used by Macro-ADR, Micro-EDM or the NPC.
}
\label{tab:seb_model_operational_checklist}
\begin{tabular}{p{0.16\linewidth}p{0.28\linewidth}p{0.28\linewidth}p{0.20\linewidth}}
\hline
\textbf{Step} & \textbf{Design question} & \textbf{Output object} & \textbf{Used in} \\
\hline
Behavioral phases &
What macroscopic stages or roles are required? &
Phase set \(\mathcal{S}\) and density components \(\rho_\sigma\) &
Macro-ADR, reaction term \\

Phase transitions &
Which phase changes are allowed, and when are they triggered? &
Transition graph \(\mathcal{G}_{\mathrm{phase}}\), mask \(A_{mn}\), triggers \(\chi_{mn}\) &
Reaction term, Micro-EDM switching \\

Physical field bases &
Which task, environment and interaction cues guide motion? &
Field dictionary \(\mathcal{B}=\{b_r\}_{r=1}^{K_b}\) &
Advection term \\

Advection parameterization &
How are field bases combined in each phase? &
Weights \(\omega_{\sigma,r}\) and advection field \(u_\sigma\) &
Macro-ADR, Micro-EDM \\

Diffusion regulation &
Which density or constraint field should regulate spacing and safety? &
Local or generalized density \(\hat{\rho}_{\mathrm{loc}}\), diffusion coefficient \(D_\sigma\) &
Diffusion term, Micro-EDM \\

Reaction construction &
How is phase mass redistributed? &
Reaction rates \(\lambda_{mn}\), reaction matrix \(\Lambda(\lambda)\) &
Macro-ADR residual \\

Training and validation &
What task objective and physical constraints are optimized? &
Reward \(\mathcal{R}_{\mathrm{task}}\), RL--PINN losses and evaluation metrics &
NPC training and evaluation \\
\hline
\end{tabular}
\end{table}

\clearpage

\subsection{Examples of typical fields}
\label{fields_example}

This section provides examples of typical field functions to help translate high-level swarm-task semantics into physical fields that can be used in ADR-based modeling. In practice, the designer does not need to manually derive a new force law for every task. Instead, suitable fields can be selected from a set of standardized field-function templates and then combined to construct the advection, diffusion and reaction components of PhySwarm.

\subsubsection{Advection field dictionary.}

The advection term describes the directional migration tendency of the swarm. In the proposed framework, the advection velocity field is generated by a weighted combination of physically interpretable potential-field bases. Each potential field \(\Phi_k(x)\) encodes a task-relevant motion preference, and its negative gradient \(-\nabla\Phi_k(x)\) gives the corresponding transport direction. The Neural-Physics Controller learns the weights \(\omega_k\) to dynamically regulate the contribution of different fields to the current behavior.

Table~\ref{tab:typical_advection_fields_cn} summarizes typical potential fields used to construct the advection term. Point attraction/repulsion fields, manifold or morphology fields, and uniform vector-flow fields can usually be regarded as external fields because they are mainly determined by task goals and environmental geometry. Local-centroid cohesion fields, alignment fields, virtual pheromone fields and communication-quality fields can be regarded as internal or information-driven fields because they depend on local interactions, historical memory or accumulated swarm-level information.

\begin{table}[h]
\centering
\renewcommand{\arraystretch}{1.5}
\caption{\textbf{Typical potential fields for constructing the advection term.}}
\label{tab:typical_advection_fields_cn}
\begin{tabular}{p{0.22\textwidth} p{0.30\textwidth} p{0.38\textwidth}}
\hline
\textbf{Field type} & \textbf{Typical template} & \textbf{Behavioral semantics} \\
\hline
Point attraction/repulsion field &
\(\Phi_{\mathrm{point}}(\mathbf x) = \frac{1}{2} \|\mathbf x - \mathbf p_{c}\|^2\) &
Provides a localized directional bias towards or away from a task-relevant point, such as a target, nest, rescue site, moving object or hazardous source. \\

Manifold or morphology field &
\(\Phi_{\mathrm{shape}}(x)=\frac{1}{2}(d(x,\mathcal{M}_\beta)-R_0)^2\) &
Encodes attraction towards a desired geometric structure, such as a circle, ellipse, line, corridor centreline or formation contour, thereby supporting morphology regulation. \\

Uniform vector-flow field &
\(\Phi_{\mathrm{flow}}(x)=-v_{\mathrm{dir}}\cdot x\) &
Introduces a global directional preference for collective motion, corridor traversal, corridor following or background-flow compensation. \\

Coordinate-anchored information field &
\(\Phi_{\mathrm{info}}(x)=\frac{1}{2}\|x-P_{\mathrm{shared}}\|^2\) &
Encodes a shared virtual anchor obtained from communication, memory or task allocation, enabling robots to respond to non-local task information. \\

Distribution-accumulated information field &
\(\Phi_{\mathrm{info}}(x)=\sum_j \mathcal{K}(x,x_j,Q_j)\) &
Represents accumulated task evidence, such as trails, pheromone-like memory, target confidence or explored-region information, and biases motion towards informative regions. \\

Local-centroid cohesion field &
\(\Phi_{\mathrm{soc}}(x)=\frac{1}{2}\|x-x_{\mathrm{lc}}\|^2\) &
Supports local cohesion by biasing robots towards the centroid of nearby neighbours, helping maintain swarm integrity or local formation structure. \\

Bipolar relay or coupling field & {\tiny \(\Phi_{\mathrm{coup}}(x) = \frac{1}{2} (1 - \gamma) \|x - p_{A}\|^2 + \frac{1}{2} \gamma \|x - p_{B}\|^2\)} & Creates a balanced spatial preference between two task anchors, supporting relay formation, cooperative transport or base--target linkage. \\

Regional safety or soft-boundary barrier field &
\(\Phi_{\mathrm{bar}}(x) = \frac{\eta}{d(x, \partial \Omega)^2}\) &
Encodes soft avoidance of boundaries, forbidden regions or large-scale hazards, reducing boundary violation at the advection level. \\ 

Anisotropic or direction-weighted field &
\(\Phi_{\mathrm{ani}}(x) = \frac{1}{2} (x - p_{c})^T \Sigma^{-1} (x - p_{c})\) &
Imposes direction-dependent motion preference, enabling stronger correction along selected axes and weaker regulation along others. \\
\hline
\end{tabular}
\end{table}

\begin{figure}[h]
    \centering
    \includegraphics[width=1.0\textwidth]{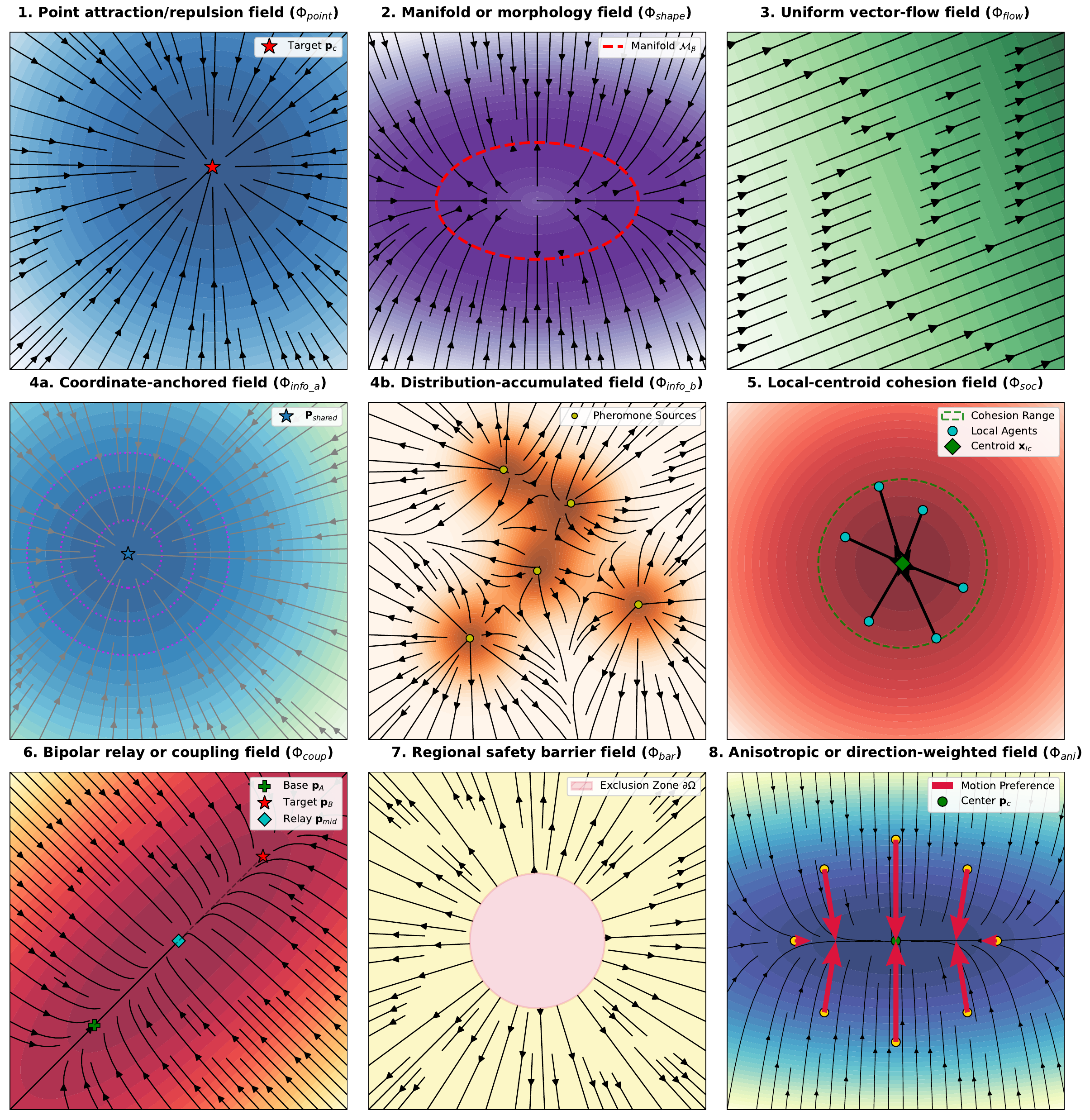} 
    \caption{\textbf{Examples of typical advection potential fields.} Background contours indicate the potential-energy landscape, and streamlines indicate the macroscopic transport direction induced by the negative potential gradients.}
    \label{fig:advection_field_examples}
\end{figure}

Fig.~\ref{fig:advection_field_examples} illustrates several representative advection fields. The background contours show the potential-energy landscape, and the streamlines show the transport direction induced by the negative gradients of the fields. In task modeling, the designer can select several field types according to the task requirements, and the Neural-Physics Controller then dynamically combines these fields by learning the weights \(\omega_k\).

\clearpage

\subsubsection{Diffusion field dictionary.}

The diffusion term does not directly specify task goals. Instead, it provides fundamental physical mechanisms for density regulation, spatial coverage, collision avoidance, and safety constraints. In the proposed framework, the diffusion direction is always induced by the negative gradient of the constructed generalized density field $\hat{\rho}_{\mathrm{diff}}(x,t)$, pushing robots away from crowded, hazardous, or unsafe regions. To support diverse spatial-regulation behaviors under varying environmental and task constraints, the candidate physical components are organized into a structured candidate dictionary. In this subsection, we detail these typical candidate fields and their physical meanings.

Table~\ref{tab:typical_diffusion_fields_cn} summarizes typical fields for constructing diffusion terms and safety constraints. In implementation, the reinforcement-learning controller (NPC) does not need to directly output complex repulsive-force functions. It only needs to regulate the diffusion coefficient \(D\), which determines the response strength of the swarm to the generalized density landscape. A larger \(D\) enhances diffusion, exploration and coverage, whereas a smaller \(D\) helps preserve compact formations or stable swarm structures.

\begin{table}[h]
\centering
\renewcommand{\arraystretch}{1.5}
\caption{\textbf{Typical fields for constructing diffusion terms and safety constraints.}}
\label{tab:typical_diffusion_fields_cn}
\begin{tabular}{p{0.22\textwidth} p{0.30\textwidth} p{0.38\textwidth}}
\hline
\textbf{Field type} & \textbf{Mathematical formulation} & \textbf{Behavioral semantics in the ADR model} \\
\hline
Local robot-density field &
\(\hat{\rho}_{\mathrm{robot}}(x) = \sum_{j \in \mathcal{N}} \mathcal{K}_h(x - x_j)\) &
Measures local crowding from robot positions and supports density-gradient compensation for spacing regulation and local collision avoidance. \\ 

Virtual obstacle-density field &
\(\hat{\rho}_{\mathrm{obs}}(x) = \max\left(0, \frac{\eta}{\mathrm{SDF}(x)}\right)\) &
Represents obstacles, walls or hazardous boundaries as high-density virtual regions, allowing diffusion regulation to encode safety-aware spatial repulsion. \\ 

Generalized total-density field &
\(\hat{\rho}_{\mathrm{total}}(x) = \alpha \hat{\rho}_{\mathrm{robot}}(x) + \gamma \hat{\rho}_{\mathrm{obs}}(x)\) &
Combines inter-robot spacing and environmental safety constraints into a unified density landscape for diffusion-based spatial regulation. \\ 

Hazard or risk field &
\(\hat{\rho}_{\mathrm{risk}}(x) = \sum_{k} I_k e^{-\lambda_k \|x - p_k\|^2}\) &
Encodes risk intensity associated with hazardous sources, interference or forbidden regions, and can be coupled to diffusion to reduce exposure to unsafe areas. \\ 

Coverage-demand field &
\(\hat{\rho}_{\mathrm{demand}}(x) = 1 - c(x)\), where \(c(x) \in [0,1]\) is the coverage rate &
Represents remaining exploration demand and can be used to bias density regulation towards under-covered regions. \\ 

Communication-congestion field &
\(\hat{\rho}_{\mathrm{comm}}(x) = \sum_{j \in \mathcal{N}_c(x)} \frac{B_j}{C_j}\), where \(B_j\) and \(C_j\) denote load and capacity &
Represents communication load or congestion as a density-like quantity, supporting spatial redistribution to reduce local communication bottlenecks. \\

Anisotropic spatial-diffusion field &
\(f_{\mathrm{diff}}(x) = - D \cdot \mathbf{J}(x) \nabla \hat{\rho}_{\mathrm{total}}(x)\), where \(\mathbf{J}\) is an anisotropic damping tensor &
Introduces direction-dependent density regulation, allowing diffusion to be stronger along safety-critical directions and weaker along task-progress directions. \\ 

Cooperative state-consensus field &
\(\dot{\theta}_i = D_{\theta} \sum_{j \in \mathcal{N}_i} (\theta_j - \theta_i)\), where \(\theta\) is an individual decision state or clock &
Extends diffusion from spatial density to internal robot states, supporting local consensus of beliefs, clocks, task variables or map priors over the communication graph. \\
\hline
\end{tabular}
\end{table}

\begin{figure}[h]
    \centering
    \includegraphics[width=1.0\textwidth]{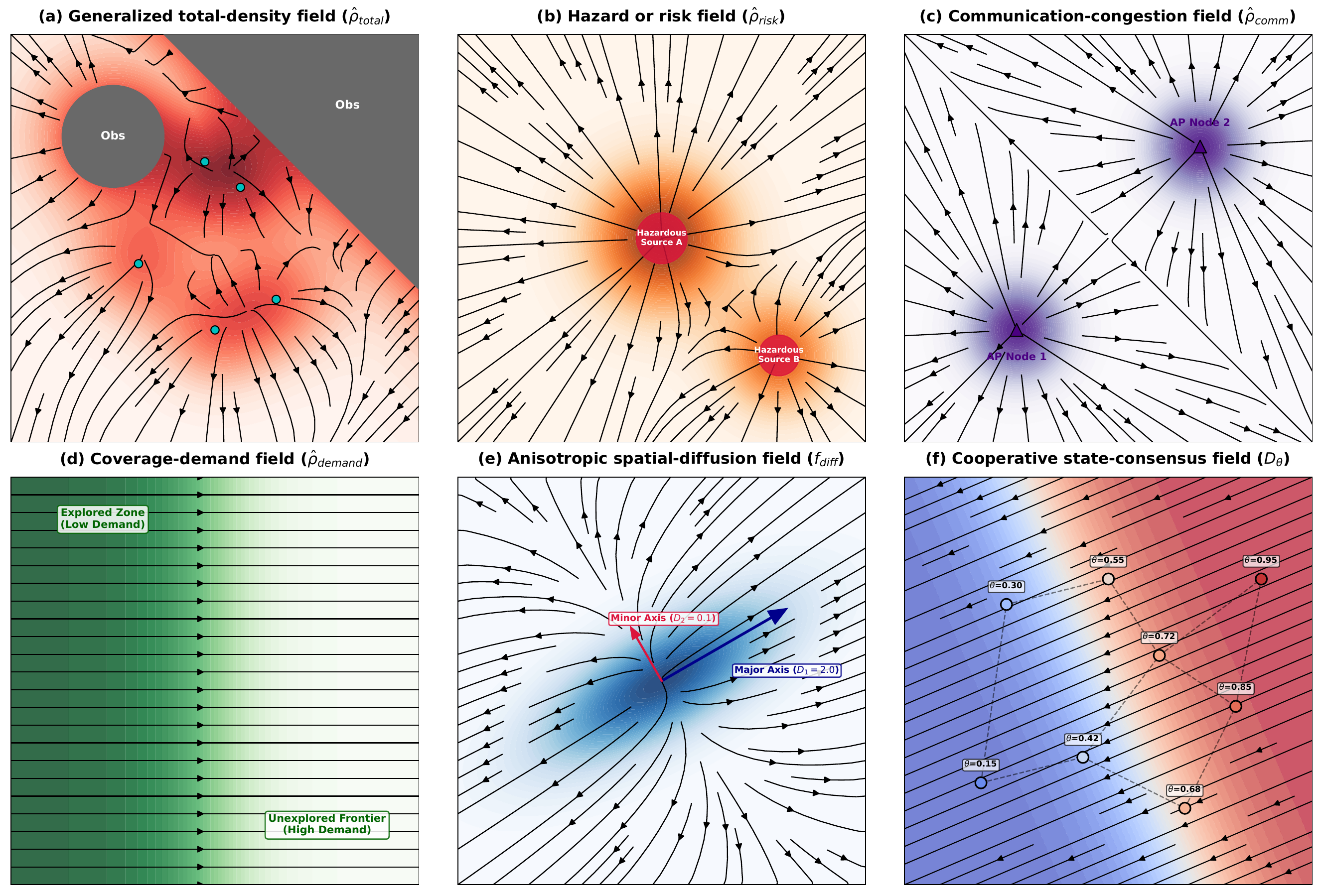} 
    \caption{\textbf{Examples of diffusion-related density fields.} \textbf{A–F,} Diverse physical configurations constructed from typical candidate fields. Streamlines indicate the outward dispersion direction induced by the negative density gradient.}
    \label{fig:diffusion_field_examples}
\end{figure}

Fig.~\ref{fig:diffusion_field_examples} illustrates this idea. The local robot-density field describes local crowding among robots, the virtual obstacle-density field represents environmental constraints, and their combination forms a generalized total-density field. The diffusion term induces outward dispersion along the negative gradient of this density field.

\clearpage

\subsubsection{Reaction field dictionary.}

The reaction term describes transitions among behavioral phases and is the key mechanism by which the Macro-ADR model represents multi-stage emergent behaviors. Table~\ref{tab:typical_reaction_fields_cn} summarizes typical triggering mechanisms, mathematical structures and behavioral semantics for constructing reaction terms.

In implementation, rather than requiring the policy network to learn all switching logic from scratch, the designer can first specify the phase-transition graph and the associated triggering conditions \(\Theta(x)\) in Table~\ref{tab:typical_reaction_fields_cn}. The Neural-Physics Controller then learns the transition rates \(\lambda_{mn}\) during training. This design preserves the behavioral safety and logical interpretability of multi-stage tasks while retaining adaptive agility in dynamic environments.

\begin{table}[h]
\centering
 \renewcommand{\arraystretch}{1.5} 
\caption{\textbf{Typical triggering mechanisms, mathematical structures and behavioral semantics for constructing the reaction term.}}
\label{tab:typical_reaction_fields_cn}
\begin{tabular}{p{0.22\textwidth} p{0.30\textwidth} p{0.38\textwidth}}
\hline
\textbf{Reaction or trigger component} & \textbf{Mathematical definition or logic} & \textbf{Behavioral semantics in the ADR model} \\
\hline
Phase-transition rate term &
\(R_{m \to n}(x) = \lambda_{mn} \cdot \Theta(x)\) &
Defines the local rate at which robot density is transferred from phase \(\sigma_m\) to phase \(\sigma_n\), combining a learnable transition intensity with a task-dependent trigger. \\ 

Event-driven trigger &
\(\Theta_{\mathrm{event}}(x) = \mathbb{I}(\text{Event Detected})\) &
Activates phase or role transitions in response to discrete task events, such as target detection, food pick-up, collision warning or communication loss. \\ 

Threshold or gating trigger &
\(\Theta_{\mathrm{threshold}}(x) = \sigma\left( \kappa (f(x) - \tau) \right)\) &
Provides a smooth transition gate when a continuous task variable, such as confidence, trail strength or interference level, crosses a predefined threshold. \\

Geometric-region trigger &
\(\Theta_{\mathrm{geom}}(x) = \mathbb{I}(x \in \Omega_{\mathrm{zone}})\) &
Activates phase transitions when robots enter task-relevant spatial regions, such as corridors, target neighbourhoods, relay zones or formation-reconfiguration regions. \\

Density-coupled trigger &
\(\Theta_{\mathrm{dens}}(x) = \sigma\left(\kappa(\hat{\rho}_{\mathrm{total}}(x) - \rho_{\mathrm{crit}})\right)\) &
Couples phase switching to local density or congestion, enabling density-dependent redistribution of behavioral phases or functional roles. \\

Role-assignment trigger &
\(\mathcal{S} = \{\sigma_{\mathrm{search}}, \sigma_{\mathrm{relay}}, \sigma_{\mathrm{resp}}\}\) &
Defines the admissible functional roles and constrains the phase-transition graph for role-adaptive tasks. \\

Distributed-consensus trigger &
{\tiny \(\Theta_{\mathrm{consensus}}(x) = \mathbb{I}\left( \frac{1}{\mid \mathcal{N}_i\mid } \sum_{j \in \mathcal{N}_i} \mathbb{I}(\sigma_j = \sigma_n) \ge \delta \right)\) }&
Enables decentralized coordination of phase switching by requiring sufficient local agreement among neighboring robots before a transition is activated. \\

Time-decay or life-cycle trigger &
\(\Theta_{\mathrm{time}}(t) = 1 - e^{-\lambda_t (t - t_{\mathrm{start}})}\) &
Models time-dependent transition tendencies associated with task duration, battery constraints, timeout conditions or lifecycle-dependent role changes. \\
\hline
\end{tabular}
\end{table}

\vspace{10mm}

\textbf{Summary.} The above field dictionary provides a reusable interface for PhySwarm modeling. The advection fields answer the question of where the swarm should move; the diffusion fields answer how the swarm should spread or avoid unsafe regions; and the reaction fields answer when robots should switch behavioral phases or functional roles. By selecting suitable fields from this dictionary and assigning them to different behavioral phases, the designer can construct a task-specific Macro-ADR model without deriving a new dynamical system from scratch. The Neural-Physics Controller then learns the time-varying parameters \(\omega\), \(D\) and \(\lambda\), allowing these physical mechanisms to adapt to the current task state.

\clearpage

\subsection{Scenario-specific model implementation}
\label{sec:scenario_specific_details_cn}

Building on the general modeling guidelines summarized in Fig.~\ref{fig:field_modeling_guideline}, this section describes how PhySwarm is instantiated in the three experimental scenarios. For each scenario, we follow a consistent construction procedure: we first define the behavioral phase set, specify the phase-transition graph and conservative reaction structure, construct the advection potential fields and diffusion-regulation density field, and finally present the resulting Macro-ADR representation together with the physical parameters learned by the Neural-Physics Controller (NPC).

\subsubsection{Trail-Guided Swarm Foraging}
\label{sec:trail_guided_swarm_foraging_cn}

\textbf{Behavioral phases.}
Trail-Guided Swarm Foraging is designed to evaluate whether the swarm can perform distributed exploration, resource discovery, carrying-to-nest motion, and trail- or virtual-pheromone-guided reuse in an unknown environment. We decompose this task into four behavioral phases:
\begin{equation}
    \mathcal{S}_{\mathrm{forage}}
    =
    \{
    \sigma_{\mathrm{exp}},
    \sigma_{\mathrm{app}},
    \sigma_{\mathrm{home}},
    \sigma_{\mathrm{trail}}
    \}.
    \label{eq:forage_phase_set_cn}
\end{equation}
Here, \(\sigma_{\mathrm{exp}}\), \(\sigma_{\mathrm{app}}\), \(\sigma_{\mathrm{home}}\) and \(\sigma_{\mathrm{trail}}\) denote the exploration, resource-approach, homing and trail-following phases, respectively.

\textbf{Phase-transition structure.}
With the phase-density vector ordered as
\[
    \rho_{\mathrm{forage}}
    =
    [
    \rho_{\mathrm{exp}},
    \rho_{\mathrm{app}},
    \rho_{\mathrm{home}},
    \rho_{\mathrm{trail}}
    ]^{\top},
\]
the conservative phase-transition matrix for the foraging task is written as
\begin{equation}
\begin{aligned}
\left[
\begin{array}{c}
R_{\mathrm{exp}} \\
R_{\mathrm{app}} \\
R_{\mathrm{home}} \\
R_{\mathrm{trail}}
\end{array}
\right]
&=
\left[
\begin{array}{cccc}
 -(\lambda_{\mathrm{exp},\mathrm{app}} + \lambda_{\mathrm{exp},\mathrm{trail}}) & 0 & \lambda_{\mathrm{home},\mathrm{exp}} & 0 \\
 \lambda_{\mathrm{exp},\mathrm{app}} & -\lambda_{\mathrm{app},\mathrm{home}} & 0 & 0 \\
 0 & \lambda_{\mathrm{app},\mathrm{home}} & -\lambda_{\mathrm{home},\mathrm{exp}} & \lambda_{\mathrm{trail},\mathrm{home}} \\
 \lambda_{\mathrm{exp},\mathrm{trail}} & 0 & 0 & -\lambda_{\mathrm{trail},\mathrm{home}}
\end{array}
\right]
\left[
\begin{array}{c}
\rho_{\mathrm{exp}} \\
\rho_{\mathrm{app}} \\
\rho_{\mathrm{home}} \\
\rho_{\mathrm{trail}}
\end{array}
\right].
\end{aligned}
\label{eq:forage_matrix_reaction_cn}
\end{equation}
Each column of this matrix sums to zero, and thus the reaction structure conserves the total swarm mass. The corresponding transition rates are activated by task events such as resource detection, resource pick-up, nest arrival and trail availability:
\begin{equation}
    \begin{aligned}
    \lambda_{\mathrm{exp},\mathrm{app}}
    &=
    \lambda_{\max}
    g_{\mathrm{resource}}(c_{\mathrm{resource}}),\\
    \lambda_{\mathrm{app},\mathrm{home}}
    &=
    \lambda_{\max}
    g_{\mathrm{pick-up}}(q_{\mathrm{pick-up}}),\\
    \lambda_{\mathrm{home},\mathrm{exp}}
    &=
    \lambda_{\max}
    g_{\mathrm{nest}}(q_{\mathrm{drop}}),\\
    \lambda_{\mathrm{exp},\mathrm{trail}}
    &=
    \lambda_{\max}
    g_{\mathrm{trail}}(c_{\mathrm{trail}}),\\
    \lambda_{\mathrm{trail},\mathrm{home}}
    &=
    \lambda_{\max}
    g_{\mathrm{pick-up}}(q_{\mathrm{pick-up}}).
    \label{eq:forage_lambda_trail_home_cn}
\end{aligned}
\end{equation}
Here, \(g(\cdot)\) can be implemented as a sigmoid or threshold-gating function; \(c_{\mathrm{resource}}\) denotes the resource confidence, \(q_{\mathrm{pick-up}}\) denotes the pick-up or carrying state, \(q_{\mathrm{drop}}\) denotes the drop-off state at the nest, and \(c_{\mathrm{trail}}\) denotes the trail or virtual-pheromone strength.

Under the convention that \(A_{mn}=1\) means the transition \(m\rightarrow n\) is allowed, the phase-transition mask for the foraging task is
\begin{equation}
    A_{\mathrm{forage}}
    =
    \left[
    \begin{array}{cccc}
        0 & 1 & 0 & 1 \\
        0 & 0 & 1 & 0 \\
        1 & 0 & 0 & 0 \\
        0 & 0 & 1 & 0
    \end{array}
    \right],
    \label{eq:forage_mask_matrix_cn}
\end{equation}
where rows denote source phases and columns denote target phases. The general conservative reaction term for the foraging task is therefore
\begin{equation}
    R_{\sigma}^{\mathrm{forage}}(\rho)
    =
    \sum_{\eta\neq\sigma}
    A_{\eta\sigma}^{\mathrm{forage}}\lambda_{\eta\sigma}\rho_{\eta}
    -
    \sum_{\ell\neq\sigma}
    A_{\sigma\ell}^{\mathrm{forage}}\lambda_{\sigma\ell}\rho_{\sigma},
    \qquad
    \sigma,\eta,\ell\in\mathcal{S}_{\mathrm{forage}}.
    \label{eq:forage_general_reaction_cn}
\end{equation}

\textbf{Advection potential fields.}
The advection component in the foraging task is constructed from a food-attraction field, a nest-attraction field, a trail/information field and an exploration-bias field (see Sec.~\ref{fields_example} for details):
\begin{equation}
    \begin{aligned}
    \Phi_{\mathrm{food}}(x)
    &=
    \frac{1}{2}\|x-p_{\mathrm{food}}\|^2,\\
    \Phi_{\mathrm{nest}}(x)
    &=
    \frac{1}{2}\|x-p_{\mathrm{nest}}\|^2,\\
    \Phi_{\mathrm{info}}(x)
    &=
    \frac{1}{2}\|x-P_{\mathrm{shared}}\|^2,\\
    \Phi_{\mathrm{rand}}(x)
    &=
    \frac{1}{2}(x-p_c)^{\top}\Sigma^{-1}(x-p_c).
    \label{eq:forage_exp_field_cn}
    \end{aligned}
\end{equation}
Here, \(p_{\mathrm{food}}\) and \(p_{\mathrm{nest}}\) denote the resource and nest locations, respectively; \(P_{\mathrm{shared}}\) denotes a shared anchor inferred from neighbor communication or trail memory; and \(p_c\) and \(\Sigma^{-1}\) define the center and inverse covariance matrix of the exploration-bias field.

The phase-conditioned advection field is
\begin{equation}
    u_{\sigma}^{\mathrm{forage}}
    =
    \sum_{k\in\mathcal{K}_{\sigma}^{\mathrm{forage}}}
    \omega_{\sigma,k}b_k,
    \qquad
    b_k=-\nabla\Phi_k,
    \qquad
    \sigma\in\mathcal{S}_{\mathrm{forage}},
    \label{eq:forage_advection_cn}
\end{equation}
where \(\mathcal{K}_{\sigma}^{\mathrm{forage}}\) is the subset of active field indices in phase \(\sigma\). The field dictionary is
\[
    \mathcal{K}^{\mathrm{forage}}
    =
    \{
    \Phi_{\mathrm{food}},
    \Phi_{\mathrm{nest}},
    \Phi_{\mathrm{info}},
    \Phi_{\mathrm{rand}}
    \}.
\]
Thus, exploration, resource approach, homing and trail following are all generated by phase-dependent weighting of the same field dictionary.

\textbf{Diffusion regulation.}
The diffusion term in the foraging task improves search coverage and suppresses local congestion near resource points, the nest or trail regions. We define the generalized local density field as
\begin{equation}
    \hat{\rho}_{\mathrm{loc}}^{\mathrm{forage}}(x,t)
    =
    \alpha_{\mathrm{a}}\hat{\rho}_{\mathrm{robot}}(x,t)
    +
    \alpha_{\mathrm{o}}\hat{\rho}_{\mathrm{obs}}(x),
    \label{eq:forage_local_density_cn}
\end{equation}
where \(\hat{\rho}_{\mathrm{robot}}\) is the local robot-density estimate, \(\hat{\rho}_{\mathrm{obs}}\) is the virtual density induced by obstacles or boundaries, and \(\alpha_{\mathrm{a}}\) and \(\alpha_{\mathrm{o}}\) are non-negative weights. The diffusion coefficient \(D_{\sigma}\) is learned by the NPC and controls the response strength to the density gradient. In general, a larger \(D_{\mathrm{exp}}\) is used in the exploration phase to enhance coverage, whereas smaller or moderate \(D_{\mathrm{home}}\) and \(D_{\mathrm{trail}}\) help preserve the main transport path while avoiding local congestion.

\textbf{Final Macro-ADR model and trainable parameters.}
The final Macro-ADR representation for the foraging task is
\begin{equation}
    \frac{\partial \rho_{\sigma}}{\partial t}
    =-
    \nabla\cdot
    \left(
        \rho_{\sigma}
        \sum_{k\in\mathcal{K}^{\mathrm{forage}}_{\sigma}}
        \omega_{\sigma,k}b_k
    \right)
    +
    \nabla\cdot
    \left(
        D_{\sigma}\nabla\rho_{\sigma}
    \right)
    +
    R_{\sigma}^{\mathrm{forage}}(\rho),
    \qquad
    \sigma\in\mathcal{S}_{\mathrm{forage}}.
    \label{eq:forage_final_adr_cn}
\end{equation}
The corresponding trainable physical parameter set is
\begin{equation}
    P_{\mathrm{forage}}(t)
    =
    \left\{
    \omega_{\sigma,k}(t),
    D_{\sigma}(t),
    \lambda_{mn}(t)
    \right\}_{\sigma\in\mathcal{S}_{\mathrm{forage}},\,k\in\mathcal{K}^{\mathrm{forage}}_{\sigma},\,(m,n)\in\mathcal{E}_{\mathrm{forage}}}.
    \label{eq:forage_trainable_parameters_cn}
\end{equation}
This scenario evaluates PhySwarm in multi-stage foraging with trail-mediated information reuse.

\subsubsection{Formation-Reconfigurable Swarm Navigation}
\label{sec:formation_reconfigurable_swarm_navigation_cn}

\textbf{Behavioral phases.}
Formation-Reconfigurable Swarm Navigation evaluates whether the swarm can adapt its formation during navigation under environmental geometric constraints. The behavioral phase set is
\begin{equation}
    \mathcal{S}_{\mathrm{nav}}
    =
    \{
    \sigma_{\mathrm{keep}},
    \sigma_{\mathrm{nav}},
    \sigma_{\mathrm{morph}},
    \sigma_{\mathrm{rec}}
    \}.
    \label{eq:nav_phase_set_cn}
\end{equation}
Here, \(\sigma_{\mathrm{keep}}\), \(\sigma_{\mathrm{nav}}\), \(\sigma_{\mathrm{morph}}\) and \(\sigma_{\mathrm{rec}}\) denote the formation-keeping, goal-directed navigation, morphology-adaptation and formation-recovery phases, respectively.

\textbf{Phase-transition structure.}
With the phase-density vector ordered as
\[
    \rho_{\mathrm{nav}}
    =
    [
    \rho_{\mathrm{keep}},
    \rho_{\mathrm{nav}},
    \rho_{\mathrm{morph}},
    \rho_{\mathrm{rec}}
    ]^{\top},
\]
the conservative phase-transition matrix for the formation-navigation task is
\begin{equation}
\begin{aligned}
\left[
\begin{array}{c}
R_{\mathrm{keep}} \\
R_{\mathrm{nav}} \\
R_{\mathrm{morph}} \\
R_{\mathrm{rec}}
\end{array}
\right]
&=
\left[
\begin{array}{cccc}
 -\lambda_{\mathrm{keep},\mathrm{nav}} & 0 & 0 & \lambda_{\mathrm{rec},\mathrm{keep}} \\
 \lambda_{\mathrm{keep},\mathrm{nav}} & -\lambda_{\mathrm{nav},\mathrm{morph}} & 0 & 0 \\
 0 & \lambda_{\mathrm{nav},\mathrm{morph}} & -\lambda_{\mathrm{morph},\mathrm{rec}} & 0 \\
 0 & 0 & \lambda_{\mathrm{morph},\mathrm{rec}} & -\lambda_{\mathrm{rec},\mathrm{keep}}
\end{array}
\right]
\left[
\begin{array}{c}
\rho_{\mathrm{keep}} \\
\rho_{\mathrm{nav}} \\
\rho_{\mathrm{morph}} \\
\rho_{\mathrm{rec}}
\end{array}
\right].
\end{aligned}
\label{eq:nav_matrix_reaction_cn}
\end{equation}
This matrix forms a closed phase-transition chain. Each column sums to zero, and therefore the total swarm mass is conserved. The transition rates are triggered by geometric cues, predicted corridor width and formation error:
\begin{equation}
    \begin{aligned}
    \lambda_{\mathrm{keep},\mathrm{nav}}
    &=
    \lambda_{\max}
    g_{\mathrm{start}}(e_{\mathrm{shape}}),\\
    \lambda_{\mathrm{nav},\mathrm{morph}}
    &=
    \lambda_{\max}
    g_{\mathrm{enter}}(W_{\mathrm{form}}-W_{\mathrm{corridor}}),\\
    \lambda_{\mathrm{morph},\mathrm{rec}}
    &=
    \lambda_{\max}
    g_{\mathrm{exit}}(W_{\mathrm{corridor}}-W_{\mathrm{form}}),\\
    \lambda_{\mathrm{rec},\mathrm{keep}}
    &=
    \lambda_{\max}
    g_{\mathrm{shape}}(e_{\mathrm{shape}},W_{\mathrm{corridor}}).
    \label{eq:nav_lambda_rec_keep_cn}
    \end{aligned}
\end{equation}
Here, \(W_{\mathrm{form}}\) is the current formation width, \(W_{\mathrm{corridor}}\) is the predicted traversable width, and \(e_{\mathrm{shape}}\) is the formation error. The functions \(g_{\mathrm{start}}\), \(g_{\mathrm{enter}}\), \(g_{\mathrm{exit}}\) and \(g_{\mathrm{shape}}\) can be implemented as smooth gating functions.

Under the convention that \(A_{mn}=1\) means the transition \(m\rightarrow n\) is allowed, the transition mask for this scenario is
\begin{equation}
    A_{\mathrm{nav}}
    =
    \left[
    \begin{array}{cccc}
        0 & 1 & 0 & 0 \\
        0 & 0 & 1 & 0 \\
        0 & 0 & 0 & 1 \\
        1 & 0 & 0 & 0
    \end{array}
    \right].
    \label{eq:nav_mask_matrix_cn}
\end{equation}
The conservative reaction term is
\begin{equation}
    R_{\sigma}^{\mathrm{nav}}(\rho)
    =
    \sum_{\eta\neq\sigma}
    A_{\eta\sigma}^{\mathrm{nav}}\lambda_{\eta\sigma}\rho_{\eta}
    -
    \sum_{\ell\neq\sigma}
    A_{\sigma\ell}^{\mathrm{nav}}\lambda_{\sigma\ell}\rho_{\sigma},
    \qquad
    \sigma,\eta,\ell\in\mathcal{S}_{\mathrm{nav}}.
    \label{eq:nav_general_reaction_cn}
\end{equation}

\textbf{Advection potential fields.}
The advection component in this scenario is constructed from a global navigation field and a morphology field (see Sec.~\ref{fields_example} for details):
\begin{equation}
    \begin{aligned}
    \Phi_{\mathrm{flow}}(x)
    &=
    -v_{\mathrm{goal}}\cdot x,\\
    \Phi_{\mathrm{shape}}(x;\beta)
    &=
    \frac{1}{2}
    \left(
    d(x,\mathcal{M}_{\beta})
    -
    R_0
    \right)^2.
    \label{eq:nav_soc_field_cn}
    \end{aligned}
\end{equation}
Here, \(v_{\mathrm{goal}}\) is the goal-directed navigation direction, \(\mathcal{M}_{\beta}\) is the target formation manifold controlled by the morphology parameter \(\beta\), \(d(x,\mathcal{M}_{\beta})\) denotes the distance from \(x\) to this manifold and \(R_0\) is the desired formation scale.

The phase-conditioned advection field is
\begin{equation}
    u_{\sigma}^{\mathrm{nav}}
    =
    \sum_{k\in\mathcal{K}_{\sigma}^{\mathrm{nav}}}
    \omega_{\sigma,k}b_k,
    \qquad
    b_k=-\nabla\Phi_k,
    \qquad
    \sigma\in\mathcal{S}_{\mathrm{nav}}.
    \label{eq:nav_advection_cn}
\end{equation}
The full field dictionary is
\[
    \mathcal{K}^{\mathrm{nav}}
    =
    \{
    \Phi_{\mathrm{flow}},
    \Phi_{\mathrm{shape}}
    \}.
\]
Thus, formation keeping, goal-directed navigation, morphology adaptation and formation recovery are all represented by phase-dependent weighting of the goal-flow and morphology fields.

\textbf{Diffusion regulation.}
The diffusion term balances formation preservation, inter-robot spacing and environmental safety. The generalized local density field is
\begin{equation}
    \hat{\rho}_{\mathrm{loc}}^{\mathrm{nav}}(x,t)
    =
    \alpha_{\mathrm{a}}\hat{\rho}_{\mathrm{robot}}(x,t)
    +
    \alpha_{\mathrm{o}}\hat{\rho}_{\mathrm{obs}}(x)
    +
    \alpha_{\mathrm{w}}\hat{\rho}_{\mathrm{wall}}(x),
    \label{eq:nav_local_density_cn}
\end{equation}
where \(\hat{\rho}_{\mathrm{wall}}\) denotes the virtual density induced by corridor walls or narrow-region boundaries. In open areas, \(D_{\mathrm{nav}}\) can remain moderate to maintain safe spacing. During morphology adaptation, \(D_{\mathrm{morph}}\) is typically reduced or adjusted anisotropically to avoid excessive dispersion of the formation. Near obstacles and boundaries, the virtual density field provides additional safety repulsion.

\textbf{Final Macro-ADR model and trainable parameters.}
The final Macro-ADR representation for the formation-navigation task is
\begin{equation}
    \frac{\partial \rho_{\sigma}}{\partial t}
    =-
    \nabla\cdot
    \left(
        \rho_{\sigma}
        \sum_{k\in\mathcal{K}^{\mathrm{nav}}_{\sigma}}
        \omega_{\sigma,k}b_k
    \right)
    +
    \nabla\cdot
    \left(
        D_{\sigma}\nabla\rho_{\sigma}
    \right)
    +
    R_{\sigma}^{\mathrm{nav}}(\rho),
    \qquad
    \sigma\in\mathcal{S}_{\mathrm{nav}}.
    \label{eq:nav_final_adr_cn}
\end{equation}
The corresponding trainable physical parameter set is
\begin{equation}
    P_{\mathrm{nav}}(t)
    =
    \left\{
    \omega_{\sigma,k}(t),
    D_{\sigma}(t),
    \lambda_{mn}(t)
    \right\}_{\sigma\in\mathcal{S}_{\mathrm{nav}},\,k\in\mathcal{K}^{\mathrm{nav}}_{\sigma},\,(m,n)\in\mathcal{E}_{\mathrm{nav}}}.
    \label{eq:nav_trainable_parameters_cn}
\end{equation}
This scenario evaluates PhySwarm in geometry-constrained formation-reconfigurable navigation.

\subsubsection{Role-Adaptive Swarm Search and Rescue}
\label{sec:role_adaptive_swarm_search_rescue_cn}

\textbf{Behavioral phases.}
Role-Adaptive Swarm Search and Rescue evaluates distributed search, target response, role differentiation and communication relay in a search-and-rescue setting. The phase set is defined as
\begin{equation}
    \mathcal{S}_{\mathrm{SAR}}
    =
    \{
    \sigma_{\mathrm{exp}},
    \sigma_{\mathrm{resp}},
    \sigma_{\mathrm{relay}}
    \}.
    \label{eq:sar_phase_set_cn}
\end{equation}
Here, \(\sigma_{\mathrm{exp}}\), \(\sigma_{\mathrm{resp}}\) and \(\sigma_{\mathrm{relay}}\) denote the exploration, target-response and communication-relay phases, respectively.

\textbf{Phase-transition structure.}
With the phase-density vector ordered as
\[
    \rho_{\mathrm{SAR}}
    =
    [
    \rho_{\mathrm{exp}},
    \rho_{\mathrm{resp}},
    \rho_{\mathrm{relay}}
    ]^{\top},
\]
the conservative phase-transition matrix for the search-and-rescue task is
\begin{equation}
\begin{aligned}
\left[
\begin{array}{c}
R_{\mathrm{exp}} \\
R_{\mathrm{resp}} \\
R_{\mathrm{relay}}
\end{array}
\right]
&=
\left[
\begin{array}{ccc}
 -(\lambda_{\mathrm{exp},\mathrm{resp}} + \lambda_{\mathrm{exp},\mathrm{relay}}) & 0 & 0 \\
 \lambda_{\mathrm{exp},\mathrm{resp}} & -\lambda_{\mathrm{resp},\mathrm{relay}} & \lambda_{\mathrm{relay},\mathrm{resp}} \\
 \lambda_{\mathrm{exp},\mathrm{relay}} & \lambda_{\mathrm{resp},\mathrm{relay}} & -\lambda_{\mathrm{relay},\mathrm{resp}}
\end{array}
\right]
\left[
\begin{array}{c}
\rho_{\mathrm{exp}} \\
\rho_{\mathrm{resp}} \\
\rho_{\mathrm{relay}}
\end{array}
\right].
\end{aligned}
\label{eq:sar_matrix_reaction_cn}
\end{equation}
This structure represents task-triggered role differentiation through the branching of searchers into responders and relays, together with dynamic conversion between responder and relay roles. Each column sums to zero, and thus the total swarm mass is conserved. The transition rates are driven by target discovery, communication bottlenecks and local role demand:
\begin{equation}
    \begin{aligned}
    \lambda_{\mathrm{exp},\mathrm{resp}}
    &=
    \lambda_{\max}
    g_{\mathrm{target}}
    (c_{\mathrm{target}}),\\
    \lambda_{\mathrm{resp},\mathrm{relay}}
    &=
    \lambda_{\max}
    g_{\mathrm{relay}}
    (n_{\mathrm{target}},q_{\mathrm{comm}}),\\
    \lambda_{\mathrm{exp},\mathrm{relay}}
    &=
    \lambda_{\max}
    g_{\mathrm{comm}}
    (q_{\mathrm{comm}},P_{\mathrm{shared}}),\\
    \lambda_{\mathrm{relay},\mathrm{resp}}
    &=
    \lambda_{\max}
    g_{\mathrm{resp}}
    (n_{\mathrm{target}},c_{\mathrm{target}}).
    \label{eq:sar_lambda_relay_resp_cn}
    \end{aligned}
\end{equation}
Here, \(c_{\mathrm{target}}\) denotes target confidence, \(n_{\mathrm{target}}\) denotes the number of robots near the target or within the rescue region, \(q_{\mathrm{comm}}\) denotes communication quality, and \(P_{\mathrm{shared}}\) denotes the target-location information shared in the local neighbourhood.

Under the convention that \(A_{mn}=1\) means the transition \(m\rightarrow n\) is allowed, the transition mask for the search-and-rescue task is
\begin{equation}
    A_{\mathrm{SAR}}
    =
    \left[
    \begin{array}{ccc}
        0 & 1 & 1 \\
        0 & 0 & 1 \\
        0 & 1 & 0
    \end{array}
    \right].
    \label{eq:sar_mask_matrix_cn}
\end{equation}
The corresponding conservative reaction term is
\begin{equation}
    R_{\sigma}^{\mathrm{SAR}}(\rho)
    =
    \sum_{\eta\neq\sigma}
    A_{\eta\sigma}^{\mathrm{SAR}}\lambda_{\eta\sigma}\rho_{\eta}
    -
    \sum_{\ell\neq\sigma}
    A_{\sigma\ell}^{\mathrm{SAR}}\lambda_{\sigma\ell}\rho_{\sigma},
    \qquad
    \sigma,\eta,\ell\in\mathcal{S}_{\mathrm{SAR}}.
    \label{eq:sar_general_reaction_cn}
\end{equation}

\textbf{Advection potential fields.}
The advection component in the search-and-rescue task is constructed from a target-attraction field, a bipolar relay/coupling field and an exploration-bias field (see Sec.~\ref{fields_example} for details):
\begin{equation}
    \begin{aligned}
    \Phi_{\mathrm{target}}(x)
    &=
    \frac{1}{2}\|x-p_{\mathrm{target}}\|^2,\\
    \Phi_{\mathrm{center}}(x)
    &=
    \frac{1}{2}(1-\gamma)\|x-p_{A}\|^2
    +
    \frac{1}{2}\gamma\|x-p_{B}\|^2,\\
    \Phi_{\mathrm{rand}}(x)
    &=
    \frac{1}{2}(x-p_c)^{\top}\Sigma^{-1}(x-p_c).
    \label{eq:sar_exp_field_cn}
    \end{aligned}
\end{equation}
Here, \(p_{\mathrm{target}}\) is the detected or shared rescue-target location, \(p_A\) and \(p_B\) are relay anchors or endpoint references between the base and the target region, \(\gamma\in[0,1]\) controls the relative weight of the bipolar coupling field.

The phase-conditioned advection field is
\begin{equation}
    u_{\sigma}^{\mathrm{SAR}}
    =
    \sum_{k\in\mathcal{K}_{\sigma}^{\mathrm{SAR}}}
    \omega_{\sigma,k}b_k,
    \qquad
    b_k=-\nabla\Phi_k,
    \qquad
    \sigma\in\mathcal{S}_{\mathrm{SAR}}.
    \label{eq:sar_advection_cn}
\end{equation}
The full field dictionary is
\[
    \mathcal{K}^{\mathrm{SAR}}
    =
    \{
    \Phi_{\mathrm{target}},
    \Phi_{\mathrm{center}},
    \Phi_{\mathrm{rand}}
    \}.
\]
Thus, distributed search, target response and communication relay are all generated through phase-dependent weighting of the same field dictionary.

\textbf{Diffusion regulation.}
In the search-and-rescue task, the diffusion term supports coverage, safety and role-distribution regulation. We define the generalized local density field as
\begin{equation}
    \hat{\rho}_{\mathrm{loc}}^{\mathrm{SAR}}(x,t)
    =
    \alpha_{\mathrm{a}}\hat{\rho}_{\mathrm{robot}}(x,t)
    +
    \alpha_{\mathrm{o}}\hat{\rho}_{\mathrm{obs}}(x).
    \label{eq:sar_local_density_cn}
\end{equation}
The search phase typically uses a larger \(D_{\mathrm{search}}\) to improve coverage. The response phase uses a smaller or moderate \(D_{\mathrm{resp}}\) to support local cooperation around the target while preventing excessive congestion. The relay phase uses a moderate \(D_{\mathrm{relay}}\) to maintain appropriate spacing among relay robots and to avoid breaking the communication chain.

\textbf{Final Macro-ADR model and trainable parameters.}
The final Macro-ADR representation for the search-and-rescue task is
\begin{equation}
    \frac{\partial \rho_{\sigma}}{\partial t}
    =-
    \nabla\cdot
    \left(
        \rho_{\sigma}
        \sum_{k\in\mathcal{K}^{\mathrm{SAR}}_{\sigma}}
        \omega_{\sigma,k}b_k
    \right)
    +
    \nabla\cdot
    \left(
        D_{\sigma}\nabla\rho_{\sigma}
    \right)
    +
    R_{\sigma}^{\mathrm{SAR}}(\rho),
    \qquad
    \sigma\in\mathcal{S}_{\mathrm{SAR}}.
    \label{eq:sar_final_adr_cn}
\end{equation}
The corresponding trainable physical parameter set is
\begin{equation}
    P_{\mathrm{SAR}}(t)
    =
    \left\{
    \omega_{\sigma,k}(t),
    D_{\sigma}(t),
    \lambda_{mn}(t)
    \right\}_{\sigma\in\mathcal{S}_{\mathrm{SAR}},\,k\in\mathcal{K}^{\mathrm{SAR}}_{\sigma},\,(m,n)\in\mathcal{E}_{\mathrm{SAR}}}.
    \label{eq:sar_trainable_parameters_cn}
\end{equation}
This scenario evaluates PhySwarm in task-triggered role-adaptive search and rescue.

\begin{remark}[Cross-scenario summary]
    Although the three scenarios differ in task semantics, they are all constructed from the same Macro-ADR template. Trail-Guided Swarm Foraging mainly uses resource, nest and information fields to represent an exploration--exploitation cycle. Formation-Reconfigurable Swarm Navigation uses goal-flow, morphology and local-cohesion fields to represent geometry-constrained formation reconfiguration. Role-Adaptive Swarm Search and Rescue uses exploration, target-attraction, bipolar relay/coupling and shared-information fields to represent role adaptation and communication relay. Together, these scenarios validate the ability of PhySwarm to provide a unified, continuously switchable and learnable representation of multi-stage swarm emergent behaviors.
\end{remark}

\clearpage

\subsection{NPC network implementation}
\label{sec:mappo_pinn_network_implementation_cn}

After constructing the Macro-ADR representation for each scenario, we further instantiate the Neural-Physics Controller (NPC) to learn the corresponding physical parameter trajectories. Although the three experimental scenarios involve different behavioral phases, field dictionaries, diffusion requirements and phase-transition structures, they share the same NPC implementation pipeline. Specifically, the controller follows a unified framework of observation encoding, temporal memory, physical-parameter prediction, ADR residual regularization and policy optimization. The differences across scenarios mainly lie in the input observation variables, the number of field bases, the number of behavioral phases, the task-specific reward design and the associated loss components.

\subsubsection{Shared network configuration}
\label{sec:shared_network_architecture}

To ensure a consistent policy representation, physical-parameter interface and training protocol across all benchmark tasks, we use the same RL--PINN Neural-Physics Controller (NPC) architecture in the three scenarios. The controller follows the centralized-training and decentralized-execution (CTDE) paradigm: during execution, each robot predicts physical parameters only from its local observation and recurrent memory; during training, a centralized critic uses the global joint state to estimate the value function. The scenario-specific components are the behavioral phase set, field dictionary, transition graph, activation functions and task reward terms. The actor--critic architecture, physical-parameter projection, Micro-EDM execution interface and physics-informed losses are shared across all tasks.

Let \(N\) denote the number of robots, \(\mathcal{S}=\{1,\ldots,M\}\) the set of behavioral phases and \(s_i^t\in\mathcal{S}\) the phase of robot \(i\) at time \(t\). The NPC predicts the physical parameter set
\[
    P_i^t=\{\omega_i^t,D_i^t,\lambda_i^t\},
\]
where \(\omega_i^t\) contains the advection-field weights, \(D_i^t\) is the diffusion coefficient and \(\lambda_i^t\) contains the phase-transition rates.

\textbf{Decentralized actor.} 
For robot \(i\), the actor receives the local observation \(o_i^t\) and the previous recurrent hidden state \(h_{a,i}^{t-1}\). The observation is first encoded by a shared multilayer perceptron (MLP), and the resulting feature is then processed by a recurrent module:
\begin{equation}
    e_i^t
    =
    \phi_{\mathrm{enc}}(o_i^t),
    \qquad
    z_i^t,h_{a,i}^{t}
    =
    \mathrm{GRU}_{a}(e_i^t,h_{a,i}^{t-1}),
    \label{eq:actor_temporal_encoding}
\end{equation}
where \(\phi_{\mathrm{enc}}\) is the local observation encoder, \(e_i^t\) is the encoded observation feature, \(z_i^t\) is the actor feature used for parameter prediction and \(h_{a,i}^{t}\) is the updated actor hidden state.

The actor uses a multi-head output structure. The continuous physical-parameter head first predicts unconstrained latent variables through a diagonal Gaussian policy:
\begin{equation}
    \tilde{p}_i^t
    \sim
    \mathcal{N}
    \left(
    \mu_{\theta}(z_i^t),
    \operatorname{diag}
    \left(
    \sigma_{\theta}^{2}(z_i^t)
    \right)
    \right),
    \label{eq:actor_gaussian_head}
\end{equation}
where \(\tilde{p}_i^t\) is the unconstrained latent output, and \(\mu_{\theta}(\cdot)\) and \(\sigma_{\theta}(\cdot)\) are predicted by separate MLP branches. The latent output is then mapped to the bounded physical parameter manifold:
\begin{equation}
    P_i^t
    =
    \{\omega_i^t,D_i^t,\lambda_i^t\}
    \in
    \mathcal{M}_{\omega}
    \times
    \mathcal{M}_{D}
    \times
    \mathcal{M}_{\lambda}.
    \label{eq:actor_physical_parameters}
\end{equation}
Specifically, \(\omega_i^t\) is mapped to a non-negative normalized simplex by a softmax function, \(D_i^t\) is mapped to the interval \([D_{\min},D_{\max}]\) by a sigmoid interval transform, and each transition rate \(\lambda_{mn,i}^t\) is mapped to \([0,A_{mn}\lambda_{\max}]\) by a masked sigmoid transform. Here, \(A_{mn}\in\{0,1\}\) is the phase-transition mask. This projection ensures that the NPC outputs physically meaningful ADR parameters rather than unconstrained low-level actions.

When phase transitions require learnable soft activation, the actor also includes an optional phase-gating head:
\begin{equation}
    \chi_i^t
    =
    \operatorname{sigmoid}
    \left(
    W_{\chi}z_i^t+b_{\chi}
    \right),
    \qquad
    \chi_i^t\in[0,1]^{\mid \mathcal{E}\mid },
    \label{eq:actor_phase_gate}
\end{equation}
where \(\mathcal{E}\) denotes the set of allowed phase-transition edges. The gating output can be used to construct the activation function \(\chi_{mn}(x_i,t)\), which modulates the transition probability together with the learned rate \(\lambda_{mn}\). For event-driven transitions, \(\chi_{mn}\) can instead be directly specified by sensor events, geometric thresholds or communication-quality indicators.

\textbf{Centralized critic. }
The critic is used only during training. It receives the global joint state \(S^t\), which contains the positions, velocities, behavioral phases, local density estimates, task variables, obstacle information and other scenario-dependent global features of all robots. To capture both spatial coordination and temporal dependence, the critic uses a spatial encoder followed by a recurrent temporal module:
\begin{equation}
    e_c^t
    =
    \phi_c(S^t),
    \qquad
    z_c^t,h_c^t
    =
    \mathrm{GRU}_{c}(e_c^t,h_c^{t-1}),
    \label{eq:critic_temporal_encoding}
\end{equation}
where \(\phi_c\) is the centralized state encoder, \(z_c^t\) is the critic feature and \(h_c^t\) is the critic hidden state. The value head outputs the centralized state value:
\begin{equation}
    V_{\psi}(S^t)
    =
    f_V(z_c^t),
    \label{eq:critic_value}
\end{equation}
where \(f_V\) is a linear or shallow MLP regression head and \(\psi\) denotes the critic parameters. During deployment, the critic is removed and each robot executes the decentralized actor using only local information.

\textbf{Shared observation interface. }
All scenarios use a body-frame local observation to improve robustness to global position, orientation and robot indexing. For robot \(i\), the observation is written as
\begin{equation}
    o_i^t
    =
    \left[
    o_{i,\mathrm{shared}}^t,
    o_{i,\mathrm{task}}^t
    \right],
    \label{eq:shared_observation}
\end{equation}
where \(o_{i,\mathrm{shared}}^t\) contains physical features shared by all scenarios, and \(o_{i,\mathrm{task}}^t\) contains task-specific cues. The shared observation contains three components: 

\begin{itemize}
    \item The robot velocity is represented in the body frame: \(v_{\mathrm{body},i}^t = \operatorname{Norm} \left( R_i^{\top}v_i^t \right)\),
            where \(v_i^t\) is the planar velocity in the global frame, \(R_i\) is the rotation matrix defined by the robot heading, and \(\operatorname{Norm}(\cdot)\) denotes normalization by the maximum speed.
    \item The robot observes the local density \(\hat{\rho}_{\mathrm{loc}}(x_i,t)\) and its gradient \(\nabla\hat{\rho}_{\mathrm{loc}}(x_i,t)\),
            where \(\hat{\rho}_{\mathrm{loc}}\) can be the robot density field or a generalized density field that also incorporates obstacle, boundary or risk constraints. These features provide the input for the density-gradient compensation in Micro-EDM.
    \item The current behavioral phase is encoded as a one-hot vector: \(p_i^t  =\operatorname{OneHot}(s_i^t) \in \{0,1\}^{M}\).
\end{itemize}

The task-specific component \(o_{i,\mathrm{task}}^t\) contains scenario-dependent cues, such as resource confidence, nest direction, trail strength, corridor width, formation error, target confidence, communication quality and relay demand. These variables are instantiated according to the field dictionary, transition triggers and task objectives of each scenario, as detailed in the scenario-specific implementation section. All continuous features are normalized before being passed to the actor.

\textbf{Shared reward structure. }
The reward follows the same additive organization in all scenarios:
\begin{equation}
    r_i
    =
    r_{\mathrm{align},i}
    +
    r_{\mathrm{safety},i}
    +
    r_{\mathrm{milestone},i}
    +
    r_{\mathrm{task},i} .
    \label{eq:shared_reward}
\end{equation}
Here, \(r_{\mathrm{align},i}\) and \(r_{\mathrm{safety},i}\) are shared physical-alignment and safety terms; \(r_{\mathrm{milestone},i}\) rewards discrete progress events such as resource discovery, corridor traversal, target detection or relay-chain formation; and \(r_{\mathrm{task},i}\) denotes scenario-specific task rewards.

The advection-alignment reward encourages the executed velocity to be consistent with the current phase-conditioned advection direction. Let \(\hat{v}_{\mathrm{adv},i}^t\) denote the normalized direction induced by the current advection field. We define
\begin{equation}
    r_{\mathrm{align},i}
    =
    S_{\mathrm{align}}
    \max
    \left(
    \frac{v_i^t}{\lVert v_i^t\rVert _2+\epsilon_v}
    \cdot
    \hat{v}_{\mathrm{adv},i}^t,
    0
    \right),
    \label{eq:alignment_reward}
\end{equation}
where \(S_{\mathrm{align}}>0\) is the reward scale and \(\epsilon_v>0\) avoids numerical instability at near-zero velocity. This reward does not replace the physics-informed residual; it provides an additional task-level incentive for microscopic motion to align with the learned advection field.

The safety term penalizes collisions and unsafe proximity:
\begin{equation}
\label{eq:safety_reward}
r_{\mathrm{safety},i}
=
\left\{
\begin{array}{ll}
r_{\mathrm{penalty}}, & d_{\min,i}^{t} < d_{\mathrm{collision}},\\
0, & \mathrm{otherwise}.
\end{array}
\right.
\end{equation}
where \(d_{\min,i}^t\) is the minimum distance from robot \(i\) to neighboring robots, obstacles or boundaries, \(d_{\mathrm{collision}}\) is the collision threshold and \(r_{\mathrm{penalty}}<0\) is a fixed penalty.

\textbf{Shared RL--PINN objective. }
In addition to the MAPPO objective, the NPC is regularized by physics-informed losses. The total training objective is
\begin{equation}
    L_{\mathrm{total}}
    =
    L_{\mathrm{RL}}
    +
    \eta L_{\mathrm{PINN}},
    \qquad
    L_{\mathrm{PINN}}
    =
    L_{\mathrm{dyn}}
    + \beta
    L_{\mathrm{ADR}},
    \label{eq:shared_total_loss}
\end{equation}
where \(L_{\mathrm{RL}}\) is the MAPPO loss, \(L_{\mathrm{PINN}}\) is the physics-informed loss, \(\eta>0\) is the physics-regularization weight, and \(\beta>0\) balances the two residuals..

For each behavioral phase \(m\), the continuous phase density is reconstructed from robot trajectories as
\begin{equation}
    \hat{\rho}_m(t,x)
    =
    \frac{1}{Nh^d}
    \sum_{i=1}^{N}
    \mathbb{I}(s_i^t=m)
    K
    \left(
    \frac{x-x_i(t)}{h}
    \right),
    \label{eq:shared_density_estimation}
\end{equation}
where \(d=2\), \(h\) is the kernel bandwidth, \(K(\cdot)\) is a smooth kernel and \(\mathbb{I}(\cdot)\) is the indicator function.

The microscopic dynamics loss aligns the executed velocity with the Micro-EDM velocity. The Micro-EDM velocity of robot \(i\) is
\begin{equation}
    v_{\mathrm{EDM},i}^t
    =
    v_{\mathrm{adv},i}^t
    -
    \frac{
    D_i^t
    }{
    \hat{\rho}_{\mathrm{loc}}(x_i,t)+\varepsilon
    }
    \nabla\hat{\rho}_{\mathrm{loc}}(x_i,t),
    \label{eq:edm_velocity_for_loss}
\end{equation}
where \(\varepsilon>0\) avoids numerical singularity. The advection component is
\begin{equation}
    v_{\mathrm{adv},i}^t
    =
    \sum_{k\in\mathcal{K}_{s_i^t}}
    \omega_{i,k}^t b_k(x_i,t),
    \qquad
    b_k(x,t)=-\nabla\Phi_k(x,t).
    \label{eq:adv_velocity_for_loss}
\end{equation}
The microscopic consistency loss is then
\begin{equation}
    L_{\mathrm{dyn}}
    =
    \frac{1}{N\mid \mathcal{T}\mid }
    \sum_{t\in\mathcal{T}}
    \sum_{i=1}^{N}
    \left\|
    v_i^t
    -
    v_{\mathrm{EDM},i}^t
    \right\|_2^2,
    \label{eq:shared_dyn_loss}
\end{equation}
where \(\mathcal{T}\) is the set of sampled time steps in a training batch and \(v_i^t\) is the executed or simulated robot velocity.

The macroscopic ADR residual loss enforces the density evolution of each phase to follow the advection--diffusion--reaction dynamics. Let
\begin{equation}
    u_m(t,x)
    =
    \sum_{k\in\mathcal{K}_{m}}
    \omega_{m,k}(t,x)b_k(x,t),
    \label{eq:macro_velocity_for_loss}
\end{equation}
where \(\omega_{m,k}(t,x)\) is obtained from the NPC outputs of robots in phase \(m\) and interpolated or averaged onto the collocation grid. The residual for phase \(m\) is
\begin{equation}
    \mathcal{R}_m(t,x)
    =
    \frac{\partial \hat{\rho}_m}{\partial t}
    +
    \nabla\cdot
    \left(
    u_m\hat{\rho}_m
    \right)
    -
    \nabla\cdot
    \left(
    D_m\nabla\hat{\rho}_m
    \right)
    -
    R_m(\hat{\rho}),
    \label{eq:adr_residual_operator}
\end{equation}
where \(D_m(t,x)\) is the diffusion coefficient field and \(R_m(\hat{\rho})\) is the conservative reaction term induced by the task-specific phase-transition graph and learned rates. With collocation points \(\{(t_j,x_j)\}_{j=1}^{N_c}\), the macroscopic loss is
\begin{equation}
    L_{\mathrm{ADR}}
    =
    \frac{1}{MN_c}
    \sum_{m=1}^{M}
    \sum_{j=1}^{N_c}
    \left\|
    \mathcal{R}_m(t_j,x_j)
    \right\|_2^2 .
    \label{eq:shared_adr_loss}
\end{equation}

Overall, the shared architecture allows the three benchmark tasks to be trained and executed through the same physical parameter interface. The actor learns the mapping from local observations to ADR parameters, the critic provides centralized value estimation, the reward supplies task-level optimization signals, and the PINN losses regularize the learned behavior to remain consistent with the Micro-EDM execution mechanism and the Macro-ADR conservation dynamics.

\subsubsection{Scenario-specific network configuration}
\label{sec:scenario_specific_network_configuration}

Although the three benchmark scenarios use the same RL--PINN Neural-Physics Controller (NPC), they instantiate different physical fields, behavioral phases, transition masks, observation cues and task rewards. This section specifies these scenario-dependent components. Rather than redefining the shared actor--critic architecture, we describe how the common NPC interface is specialized for each task.

\textbf{Scenario-specific actor outputs. }
Table~\ref{tab:gaussian_head_configs} summarizes the continuous output channels activated by the actor in each scenario. These channels correspond to the physical quantities required by the task-specific PhySwarm instance. Field weights and diffusion coefficients are produced by the continuous physical-parameter head, whereas phase-transition rates are produced by the masked reaction-rate head according to the scenario-specific transition graph.

\textbf{Scenario-specific observation variables. }
All scenarios use the shared body-frame observation interface described in Sec.~\ref{sec:shared_network_architecture}. In addition to the shared physical features, each scenario appends task-specific cues required by its field dictionary, phase-transition triggers and reward design. The final local observation of robot \(i\) at time \(t\) is denoted by \(o_i^t\in\mathbb{R}^{D_{\mathrm{obs}}}.\)
In our implementation, \(D_{\mathrm{obs}}\) is set to \(17\), \(16\) and \(13\) for Trail-Guided Swarm Foraging, Formation-Reconfigurable Swarm Navigation and Role-Adaptive Swarm Search and Rescue, respectively. Table~\ref{tab:scenario_specific_obs} lists the scenario-specific observation variables and their physical meanings.

\textbf{Scenario-specific reward terms. }
The reward function in all scenarios follows the shared decomposition
\[
    r_i =
    r_{\mathrm{align},i}
    +
    r_{\mathrm{safety},i}
    +
    r_{\mathrm{milestone},i}
    +
    r_{\mathrm{task},i} .
\]
The alignment and safety terms are shared across tasks, whereas the milestone reward \(r_{\mathrm{milestone},i}\) and the task-specific reward \(r_{\mathrm{task},i}\) are instantiated according to the objective of each scenario. Table~\ref{tab:scenario_specific_rewards} summarizes these scenario-specific reward components.

\textbf{Scenario-specific physics-informed constraints. }
The shared PINN loss consists of a microscopic dynamics consistency term and a macroscopic ADR residual term. Its mathematical form remains unchanged across scenarios; only the task-specific advection field \(v_{\mathrm{adv},i}\) and conservative reaction source \(R_m(\hat{\rho})\) are instantiated differently according to the field dictionary and phase-transition graph of each task. Table~\ref{tab:unified_pinn_constraints} summarizes these scenario-specific physical mappings.

Together, these configurations instantiate the same NPC architecture under different task semantics. The actor--critic backbone, physical projection layer, Micro-EDM execution interface and RL--PINN training objective remain shared, whereas the active field bases, phase-transition graph, observation cues and task rewards are selected according to the modeling requirements of each benchmark task.

\begin{table}[htbp]
\centering
\caption{\textbf{Scenario-specific continuous output channels of the actor.} Each scenario uses the same NPC architecture but activates a different set of physical parameters according to its field dictionary and transition graph.}
\label{tab:gaussian_head_configs}
\footnotesize
\begin{tabular}{p{0.15\textwidth} | p{0.06\textwidth} | p{0.34\textwidth} | p{0.34\textwidth}}
\hline
\textbf{Scenario} &
\textbf{Dim.} &
\textbf{Mean-vector mapping} &
\textbf{Physical parameters and semantics} \\
\hline

\textbf{Trail-Guided Swarm Foraging} &
7 &
\([\mu_{\mathrm{food}}, \mu_{\mathrm{nest}}, \mu_{\mathrm{exp}}, \mu_{\mathrm{info}}, \mu_{\mathrm{diff}}, \mu_{\mathrm{pick}}, \mu_{\mathrm{drop}}]^{\top}\) &
\(\bullet\) food-attraction weight \(w_{\mathrm{food}}\) \newline
\(\bullet\) nest-attraction weight \(w_{\mathrm{nest}}\) \newline
\(\bullet\) exploration-bias weight \(w_{\mathrm{exp}}\) \newline
\(\bullet\) trail/information weight \(w_{\mathrm{info}}\) \newline
\(\bullet\) diffusion coefficient \(D\) \newline
\(\bullet\) task-event transition rates, including pick-up and drop-off gates \\
\hline

\textbf{Formation-Reconfigurable Swarm Navigation} &
4 &
\([\mu_{\mathrm{flow}}, \mu_{\mathrm{shape}}, \mu_{\mathrm{diff}}, \mu_{\beta}]^{\top}\) &
\(\bullet\) goal-flow weight \(w_{\mathrm{flow}}\) \newline
\(\bullet\) morphology/formation weight \(w_{\mathrm{shape}}\) \newline
\(\bullet\) diffusion coefficient \(D\) \newline
\(\bullet\) morphology parameter \(\beta\) for shape adaptation \\
\hline

\textbf{Role-Adaptive Swarm Search and Rescue} &
4 &
\([\mu_{\mathrm{exp}}, \mu_{\mathrm{target}}, \mu_{\mathrm{relay}}, \mu_{\mathrm{diff}}]^{\top}\) &
\(\bullet\) exploration-bias weight \(w_{\mathrm{exp}}\) \newline
\(\bullet\) target-attraction weight \(w_{\mathrm{target}}\) \newline
\(\bullet\) relay/coupling-field weight \(w_{\mathrm{relay}}\) \newline
\(\bullet\) diffusion coefficient \(D\) \\
\hline
\end{tabular}
\end{table}

\begin{table}[htbp]
\centering
\caption{\textbf{Scenario-specific observation variables.} Each task extends the shared observation interface with cues that instantiate its physical fields, transition triggers and task objectives.}
\label{tab:scenario_specific_obs}
\footnotesize
\begin{tabular}{p{0.20\textwidth} | p{0.12\textwidth} | p{0.24\textwidth} | p{0.36\textwidth}}
\hline
\textbf{Scenario} &
\textbf{Dim.} &
\textbf{Task-specific features} &
\textbf{Physical semantics} \\
\hline

\textbf{Trail-Guided Swarm Foraging} &
10 &
\(\hat{v}_{\mathrm{food}}^i\), \(e^{-d_{\mathrm{food}}^i}\), 
\(\hat{v}_{\mathrm{nest}}^i\), \(e^{-d_{\mathrm{nest}}^i}\), 
\(\hat{v}_{\mathrm{info}}^i\), \(I_{\mathrm{info}}^i\) &
Encodes the body-frame direction and distance to the nearest detected resource, the nest direction and distance for homing, and the information-guided attraction induced by trail memory or neighbor broadcasts. \\
\hline

\textbf{Formation-Reconfigurable Swarm Navigation} &
9 &
\(\hat{v}_{\mathrm{cent}}^i\), \(d_{\mathrm{cent}}^i\), 
\(\sigma_{\mathrm{ema},y}^i\), 
\(\hat{v}_{\mathrm{goal}}^i\), \(d_{\mathrm{goal}}^i\), 
\(w_{\mathrm{forward}}^i\) &
Encodes local centroid attraction for formation cohesion, lateral dispersion as a proxy for shape compression, goal-directed navigation cues and predicted traversable width for anticipatory morphology adaptation. \\
\hline

\textbf{Role-Adaptive Swarm Search and Rescue} &
6 &
\(F_{\mathrm{target}}^i\), \(\hat{v}_{\mathrm{target}}^i\), 
\(F_{\mathrm{base}}^i\), \(\hat{v}_{\mathrm{axis}}^i\) &
Encodes the attraction strength and direction to the detected rescue target, the base-anchor cue and the axis-correction vector that guides relay robots toward the base--target communication corridor. \\
\hline
\end{tabular}
\end{table}

\begin{table}[htbp]
\centering
\caption{\textbf{Scenario-specific reward terms.} The shared alignment and safety rewards are complemented by task-specific regulation and milestone terms.}
\label{tab:scenario_specific_rewards}
\footnotesize
\begin{tabular}{p{0.20\textwidth} | p{0.22\textwidth} | p{0.50\textwidth}}
\hline
\textbf{Scenario} &
\textbf{Task-specific terms} &
\textbf{Purpose and physical semantics} \\
\hline

\textbf{Trail-Guided Swarm Foraging} &
\(r_{\mathrm{disp},i}\), \(r_{\mathrm{event},i}\) &
\(\bullet\) \textbf{Spatial dispersion and congestion suppression:} encourages outward exploration away from the nest while penalizing excessive local crowding near the nest, resource or trail regions. \newline
\(\bullet\) \textbf{Event milestone reward:} provides sparse rewards when a robot successfully crosses key task events, such as pick-up, homing completion or drop-off. \\
\hline

\textbf{Formation-Reconfigurable Swarm Navigation} &
\(r_{\mathrm{form}}\), \(\psi_{\mathrm{gate}}\), \(\chi_{\mathrm{atten}}\), \(r_{\mathrm{squeeze},i}\), \(r_{\mathrm{act},i}\) &
\(\bullet\) \textbf{Formation-quality reward:} maintains the desired spatial structure while avoiding degenerate solutions. \newline
\(\bullet\) \textbf{Distance- and morphology-gated regulation:} adjusts the contribution of shape rewards according to goal progress and formation state. \newline
\(\bullet\) \textbf{Anticipatory compression penalty:} penalizes lateral deviation when the predicted corridor width decreases, encouraging the swarm to compress before entering narrow corridors. \\
\hline

\textbf{Role-Adaptive Swarm Search and Rescue} &
\(R_{\mathrm{quota},i}\), \(R_{\mathrm{topo},i}\), \(R_{\mathrm{team\_disc}}\), \(R_{\mathrm{team\_rescue}}\) &
\(\bullet\) \textbf{Role allocation:} discourages redundant responders and encourages functional differentiation between responders and relays. \newline
\(\bullet\) \textbf{Topology maintenance:} rewards relay formation near the base--target axis and penalizes relay-chain disconnection or excessive deviation. \newline
\(\bullet\) \textbf{Team-level milestones:} provides shared rewards for first target discovery and successful completion of the rescue-relay chain. \\
\hline
\end{tabular}
\end{table}

\begin{table}[htbp]
\centering
\caption{\textbf{Scenario-specific physical mappings used in the PINN constraints.}
The microscopic advection velocity is constructed from the active field bases, and the reaction source follows the conservative phase-transition structure of each task.}
\label{tab:unified_pinn_constraints}
\footnotesize
\begin{tabular}{p{0.17\textwidth} | p{0.33\textwidth} | p{0.45\textwidth}}
\hline
\textbf{Scenario} & 
\textbf{ \(\mathbf{v}_{\mathrm{adv}, i}(t)\)} & 
\textbf{ \(R_m(\hat{\boldsymbol{\rho}})\)} \\
\hline
\textbf{Trail-Guided} \par \textbf{Swarm Foraging} \par (\(M=4\)) & 
\(\begin{aligned}[t]
    \mathbf{v}_{\mathrm{adv}, i}(t) =&  -\omega_{\mathrm{food}, i}\nabla \Phi_{\mathrm{food}}(\mathbf{x}_i) \\
    &-\omega_{\mathrm{nest}, i}\nabla \Phi_{\mathrm{nest}}(\mathbf{x}_i) \\
    &-\omega_{\mathrm{info}, i}\nabla \Phi_{\mathrm{info}}(\mathbf{x}_i) \\
    &-\omega_{\mathrm{exp}, i}\nabla \Phi_{\mathrm{exp}}(\mathbf{x}_i)
\end{aligned}\) & 
\(\begin{aligned}[t]
    R_{\mathrm{exp}} &= \lambda_{\mathrm{home},\mathrm{exp}} \hat{\rho}_{\mathrm{home}}  - \left[ \lambda_{\mathrm{exp},\mathrm{app}} + \lambda_{\mathrm{exp},\mathrm{trail}} \right] \hat{\rho}_{\mathrm{exp}} \\
    R_{\mathrm{app}} &= \lambda_{\mathrm{exp},\mathrm{app}} \hat{\rho}_{\mathrm{exp}} - \lambda_{\mathrm{app},\mathrm{home}} \hat{\rho}_{\mathrm{app}} \\
    R_{\mathrm{home}} &= \lambda_{\mathrm{app},\mathrm{home}} \hat{\rho}_{\mathrm{app}} + \lambda_{\mathrm{trail},\mathrm{home}} \hat{\rho}_{\mathrm{trail}}  - \lambda_{\mathrm{home},\mathrm{exp}} \hat{\rho}_{\mathrm{home}} \\
    R_{\mathrm{trail}} &= \lambda_{\mathrm{exp},\mathrm{trail}} \hat{\rho}_{\mathrm{exp}} - \lambda_{\mathrm{trail},\mathrm{home}} \hat{\rho}_{\mathrm{trail}}
\end{aligned}\) \\ \hline

\textbf{Formation-} \par \textbf{Reconfigurable} \par \textbf{Swarm} \par \textbf{ Navigation} \par (\(M=4\)) & 
\(\begin{aligned}[t]
    \mathbf{v}_{\mathrm{adv}, i}(t) =& -\omega_{\mathrm{shape}, i}\nabla \Phi_{\mathrm{shape}}\left(\mathbf{x}_i; \beta_i\right) \\
    &-\omega_{\mathrm{flow}, i}\nabla \Phi_{\mathrm{flow}}(\mathbf{x}_i)
\end{aligned}\) & 
\(\begin{aligned}[t]
    R_{\mathrm{keep}} &= \lambda_{\mathrm{rec},\mathrm{keep}} \hat{\rho}_{\mathrm{rec}}  - \lambda_{\mathrm{keep},\mathrm{nav}} \hat{\rho}_{\mathrm{keep}} \\
    R_{\mathrm{nav}} &= \lambda_{\mathrm{keep},\mathrm{nav}} \hat{\rho}_{\mathrm{keep}}  - \lambda_{\mathrm{nav},\mathrm{morph}} \hat{\rho}_{\mathrm{nav}} \\
    R_{\mathrm{morph}} &= \lambda_{\mathrm{nav},\mathrm{morph}} \hat{\rho}_{\mathrm{nav}}  - \lambda_{\mathrm{morph},\mathrm{rec}} \hat{\rho}_{\mathrm{morph}} \\
    R_{\mathrm{rec}} &= \lambda_{\mathrm{morph},\mathrm{rec}} \hat{\rho}_{\mathrm{morph}}  - \lambda_{\mathrm{rec},\mathrm{keep}} \hat{\rho}_{\mathrm{rec}}
\end{aligned}\) \\ \hline

\textbf{Role-Adaptive} \par \textbf{Swarm} \par \textbf{Search} \par \textbf{ and Rescue} \par (\(M=3\)) & 
\(\begin{aligned}[t]
    \mathbf{v}_{\mathrm{adv}, i}(t) =& -\omega_{\mathrm{exp}, i}\nabla \Phi_{\mathrm{exp}}(\mathbf{x}_i) \\
    &-\omega_{\mathrm{target}, i}\nabla \Phi_{\mathrm{target}}(\mathbf{x}_i) \\
    &-\omega_{\mathrm{relay}, i}\nabla \Phi_{\mathrm{coup}}(\mathbf{x}_i)
\end{aligned}\) & 
\(\begin{aligned}[t]
    R_{\mathrm{search}} &= - \left[ \lambda_{\mathrm{search},\mathrm{resp}} + \lambda_{\mathrm{search},\mathrm{relay}} \right] \hat{\rho}_{\mathrm{search}} \\
    R_{\mathrm{resp}} &= \lambda_{\mathrm{search},\mathrm{resp}} \hat{\rho}_{\mathrm{search}} + \lambda_{\mathrm{relay},\mathrm{resp}} \hat{\rho}_{\mathrm{relay}}  - \lambda_{\mathrm{resp},\mathrm{relay}} \hat{\rho}_{\mathrm{resp}} \\
    R_{\mathrm{relay}} &= \lambda_{\mathrm{resp},\mathrm{relay}} \hat{\rho}_{\mathrm{resp}} + \lambda_{\mathrm{search},\mathrm{relay}} \hat{\rho}_{\mathrm{search}}  - \lambda_{\mathrm{relay},\mathrm{resp}} \hat{\rho}_{\mathrm{relay}}
\end{aligned}\) \\
\hline
\end{tabular}
\end{table}

\clearpage

\subsection{Experimental arena setup}
\label{sec:arena}

This section describes the robot platform, experimental arenas and software--hardware configurations used in the simulation and real-robot experiments. All experiments were conducted in a bounded two-dimensional workspace, denoted by \(\Omega\subset\mathbb{R}^{2}\). The workspace boundary was defined by walls or physical enclosures, and robots were constrained to move inside \(\Omega\) without crossing the boundary. This setting is consistent with the zero-flux boundary assumption used in the Macro-ADR analysis: the swarm can migrate, diffuse and switch behavioural phases within the task domain, but no non-physical loss of robot mass occurs through the boundary.

\textbf{Robot platform.}
We used \textit{e-puck} differential-drive mobile robots as the platform for both simulation and real-robot experiments, shown in Fig.~\ref{fig:epuck}. The \textit{e-puck} robot has a compact mechanical structure and a well-defined motion model, making it suitable for small-scale multi-robot experiments. Each robot consists of a differential-drive chassis, left and right wheel actuators, onboard sensors, odometry modules and communication interfaces. Its low-level motion follows differential-drive kinematics, in which planar linear and angular velocity commands are converted into left- and right-wheel speeds. For more details of the \textit{e-puck} robot, please be referred to~\cite{millard2017pipuck}.

In PhySwarm, the Neural-Physics Controller (NPC) does not directly output wheel speeds. Instead, it predicts physical parameters \(P_i(t)=\{\omega_i(t),D_i(t),\lambda_i(t)\}\), which are first converted by Micro-EDM into a desired planar velocity and then mapped by the differential-drive controller to low-level wheel-speed commands. Therefore, the simulated and physical robots share the same high-level control interface. Their differences mainly arise from low-level factors such as sensor noise, actuator constraints, communication delay and localization error.

\textbf{Simulation arena.}
The simulation experiments were built in Webots 2023b~\cite{michel2004webots}. Webots provides robot kinematic simulation, sensor modelling, scene construction and external control interfaces, which are suitable for rapid validation and parameter debugging in multi-robot systems. The simulated arena is a \(3\,\mathrm{m}\times1\,\mathrm{m}\) rectangular workspace enclosed by boundary walls (shown in Fig.~\ref{fig:sim_arena}). Fixed obstacles were added in some experiments to simulate boundary constraints and narrow corridors in small-scale indoor environments. We used \(8\) \textit{e-puck} robots for the basic experiments, with additional scalability tests using larger swarm sizes (up to 48). Each robot was initialized with an independent pose and equipped with IMU, infrared distance sensing and odometry interfaces.

During simulation, the Webots robot models were connected to the high-level controller through ROS 2 interfaces~\cite{macenski2022robot}. The physical parameters predicted by the NPC were converted by Micro-EDM into desired planar velocities and then mapped to left- and right-wheel commands through the low-level differential-drive interface. The simulation environment recorded robot poses, sensor readings, behavioural phases and control commands for trajectory analysis, density-field reconstruction, physical-consistency evaluation and task-performance statistics. To support reproducibility, Webots world files, ROS 2 node definitions, launch scripts, topic interfaces and experimental execution procedures are provided in the source-code repository.

\begin{figure}[h]
    \centering
    \includegraphics[width=0.5\textwidth]{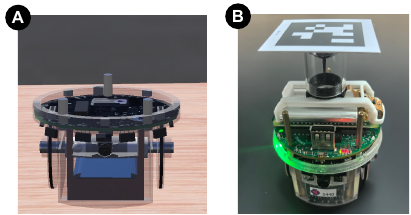} 
    \caption{\textbf{E-puck robot.} \textbf{a,} the model used in Webots simulator. \textbf{b,} real e-puck robot.
    }
    \label{fig:epuck}
\end{figure}

\begin{figure}[h]
    \centering
    \includegraphics[width=0.99\textwidth]{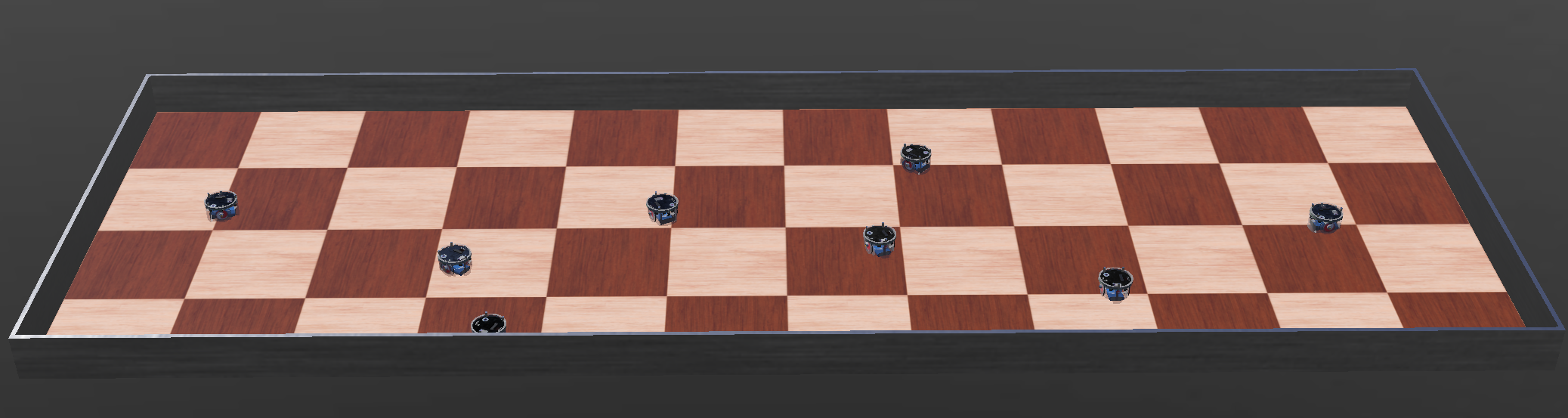} 
    \caption{\textbf{Webots arena.} A $3\,\mathrm{m}\times1\,\mathrm{m}$ simulation corridor environment for tasks. 
    }
    \label{fig:sim_arena}
\end{figure}

\textbf{Real arena.}
To evaluate the feasibility and robustness of PhySwarm on physical robots, we further constructed a real-robot experimental arena, shown in Fig.~\ref{fig:real_arena}. The physical arena is an approximately \(3\,\mathrm{m}\times1\,\mathrm{m}\) bounded rectangular workspace enclosed by surrounding barriers. Depending on the task, fixed obstacles, target regions, nest regions, base markers or narrow-corridor structures were placed in the arena. The real-robot experiments used \(8\) physical \textit{e-puck} robots. To obtain reference poses for experimental recording and evaluation, an industrial overhead camera was installed above the arena, and a unique AprilTag marker was attached to the top of each robot.

The real-robot platform used the same high-level control interface as the simulation platform. The physical parameters predicted by the NPC were first converted by Micro-EDM into desired planar velocities and then mapped by the differential-drive controller to executable wheel-speed commands. On the robot side, the data link handled chassis actuation, onboard sensing and odometry publishing. In parallel, an external reference link, implemented with an overhead camera and AprilTag detection, provided global reference poses for experiment logging, trajectory analysis and metric computation. Importantly, these global reference poses were not used as standard local observation inputs to the NPC. The controller primarily relied on onboard sensing, neighborhood information, task events and communication states, thereby preserving the decentralized-execution setting. The external localization system was used only as an evaluation and recording interface rather than as a centralized control input.

A hardware safety layer used the onboard infrared sensors of the \textit{e-puck} robots to monitor nearby obstacles and inter-robot distances. When a potential collision risk was detected, low-level protective actions were triggered to prevent unsafe contacts. Compared with simulation, the real-robot experiments introduced additional non-ideal factors, including illumination changes, AprilTag occlusion, camera-calibration error, uneven ground friction, sensor noise, communication delay and actuator uncertainty. The physical platform therefore provided a practical testbed for examining the stability, robustness and transferability of the PhySwarm coupling mechanism under real-world constraints.

\begin{figure}[h]
    \centering
    \includegraphics[width=0.7\textwidth, angle=90]{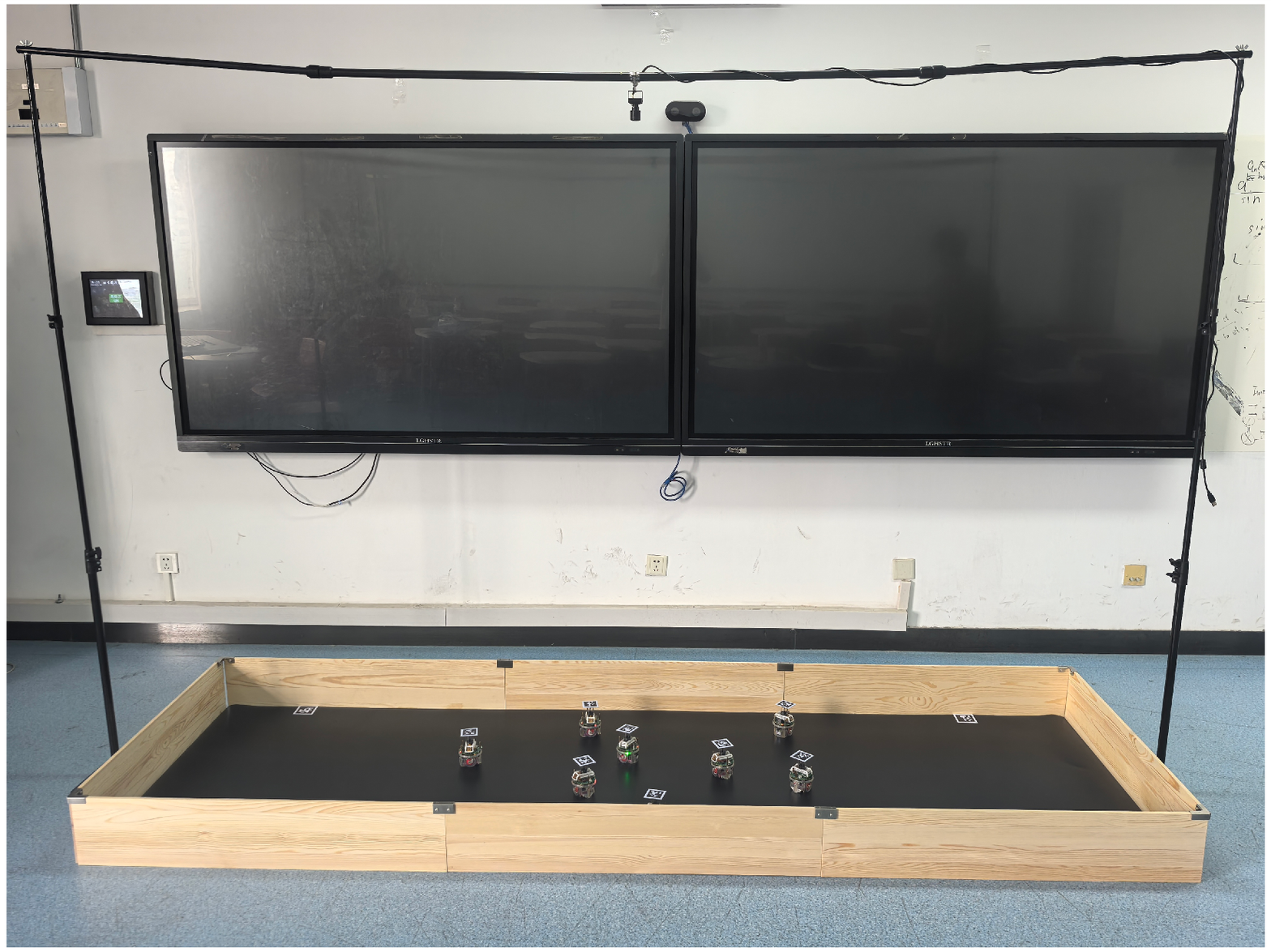} 
    \caption{\textbf{Real arena.} (a) The The $3\,\mathrm{m}\times1\,\mathrm{m}$ bounded rectangular workspace populated with eight physical e-puck robots, featuring custom home base and rescue target regions. (b) An industrial overhead camera tracking system utilizing unique AprilTag markers attached to each robot's top surface to capture high-precision global reference poses for experimental logging and trajectory analysis. 
    }
    \label{fig:real_arena}
\end{figure}

\textbf{Hardware and software configuration.}
The simulation and real-robot platforms used a consistent software architecture. Both platforms were implemented on Ubuntu 20.04 with ROS 2 Foxy. ROS 2 topics were used for velocity control, sensor reading, odometry publishing and reference-pose recording. The main differences between the two platforms lie in the experimental environment, robot embodiment, sensor-data source and reference-pose acquisition method. The main configurations are summarized in Table~\ref{tab:platform_comparison}.

\begin{table}[htbp]
\centering
\caption{Main configurations of the simulation and real-robot platforms.}
\label{tab:platform_comparison}
\begin{tabular}{p{0.24\textwidth}p{0.34\textwidth}p{0.34\textwidth}}
\toprule
\textbf{Item} & \textbf{Simulation platform} & \textbf{Real-robot platform} \\
\midrule
Experimental environment & Webots 2023b & Physical robot arena \\
Arena size & \(3\,\mathrm{m}\times1\,\mathrm{m}\) & \(3\,\mathrm{m}\times1\,\mathrm{m}\) \\
Robot type & Simulated \textit{e-puck} robots & Physical \textit{e-puck} robots \\
Number of robots & \(8\text{--}48\) & \(8\) \\
Locomotion & Differential drive & Differential drive \\
Operating system & Ubuntu 20.04 & Ubuntu 20.04 \\
Communication framework & ROS 2 Foxy & ROS 2 Foxy \\
High-level control interface & Micro-EDM desired-velocity input & Micro-EDM desired-velocity input \\
Low-level execution interface & Differential wheel-speed control & Differential wheel-speed control \\
Onboard sensors & IMU, infrared sensors and odometry & IMU, infrared sensors and odometry \\
Reference-pose source & Webots Supervisor & Overhead camera + AprilTag \\
Use of reference pose & Trajectory logging and metric computation & Trajectory logging and metric computation \\
Main uncertainty sources & Simulation noise and modelling error & Sensor noise, localization error, communication delay and actuator error \\
Main purpose & Model training, scalability testing and repeatability evaluation & Real-robot execution and transferability validation \\

\bottomrule
\end{tabular}
\end{table}

\clearpage

\subsection{Supplementary evaluation metrics}
\label{sec:versatility_metrics}

The main text uses the ADR manifold divergence as the primary task-independent physical-consistency metric. Here we provide additional execution-level and scenario-specific metrics used for supplementary analysis, ablation studies and radar-plot comparisons in Sec.~\ref{sec:versatility}. The execution-level metrics quantify whether the learned physical regulation can be implemented smoothly and safely by the robots. The scenario-specific metrics capture task objectives that are not shared across all benchmarks, including resource delivery, formation preservation and relay-chain connectivity.

For radar-plot visualization in Fig.~\ref{fig:versatility_forage}--\ref{fig:versatility_rescue}, all reported metrics are normalized to \([0,1]\) and oriented so that larger values indicate better performance. For cost-type metrics, such as control variation, collision exposure, navigation time, formation error and time to connectivity, we report the corresponding inverted normalized score. Raw metric values are used for statistical analysis.

\subsubsection{Execution-level metrics shared across tasks}

\textbf{Control smoothness.}
Control smoothness measures the temporal variation of low-level command sequences. Let \(c_i(t)\) denote the command vector of robot \(i\) at control step \(t\). For differential-drive robots, \(c_i(t)\) contains the left- and right-wheel target speeds. We define
\begin{equation}
    J_{\mathrm{smooth}}
    =
    \frac{1}{N(T_{\mathrm{ep}}-1)}
    \sum_{i=1}^{N}
    \sum_{t=2}^{T_{\mathrm{ep}}}
    \left\|
    c_i(t)-c_i(t-1)
    \right\|_{1},
    \label{eq:supp_smoothness}
\end{equation}
where \(N\) is the number of robots and \(T_{\mathrm{ep}}\) is the number of control steps in one episode. A smaller \(J_{\mathrm{smooth}}\) indicates smoother low-level execution and fewer high-frequency actuation changes.

\textbf{Collision safety.}
Collision safety is evaluated by the normalized frequency of unsafe robot--robot and robot--obstacle contacts over an episode. Let \(d_{\min,i}(t)\) be the minimum distance from robot \(i\) to neighboring robots, obstacles or workspace boundaries at step \(t\), and let \(d_{\mathrm{collision}}\) be the collision threshold. We define the normalized collision rate as
\begin{equation}
    R_{\mathrm{coll}}
    =
    \frac{1}{N T_{\mathrm{ep}}}
    \sum_{t=1}^{T_{\mathrm{ep}}}
    \sum_{i=1}^{N}
    \mathbb{I}
    \left(
    d_{\min,i}(t)<d_{\mathrm{collision}}
    \right),
    \label{eq:supp_collision_rate}
\end{equation}
where \(N\) is the number of robots, \(T_{\mathrm{ep}}\) is the episode length and \(\mathbb{I}(\cdot)\) is the indicator function. A smaller \(R_{\mathrm{coll}}\) indicates fewer unsafe contacts and safer swarm motion.

\subsubsection{Trail-Guided Swarm Foraging}

For \textit{Trail-Guided Swarm Foraging}, we evaluate cumulative resource-delivery progress, per-robot foraging efficiency and transport economy.

\textbf{Cumulative foraging throughput.}
To quantify temporal resource-delivery progress, we define the cumulative foraging throughput \(C_{\mathrm{forage}}(t)\) at control step \(t\) as
\begin{equation}
    C_{\mathrm{forage}}(t)
    =
    M_0-M(t)
    =
    \sum_{k=1}^{t}
    n_{\mathrm{del}}(k),
    \label{eq:cumulative_throughput}
\end{equation}
where \(M_0\) is the initial number of available resources, \(M(t)\) is the number of undelivered resources remaining in the arena at step \(t\), and \(n_{\mathrm{del}}(k)\) is the number of delivery events at step \(k\). At the end of an episode,
\(C_{\mathrm{forage}}(T_{\mathrm{ep}})=N_{\mathrm{delivered}}\), where \(N_{\mathrm{delivered}}\) is the total number of resources delivered to the nest. A steeper increase in \(C_{\mathrm{forage}}(t)\) indicates faster collective resource discovery and transport.

\textbf{Per-robot foraging efficiency.}
The per-robot foraging efficiency \(\Phi_{\mathrm{ind}}\) measures the average number of successfully delivered resources per robot per unit time:
\begin{equation}
    \Phi_{\mathrm{ind}}
    =
    \frac{N_{\mathrm{delivered}}}{N T_{\mathrm{ep}}},
    \label{eq:foraging_efficiency_phi}
\end{equation}
where \(N\) is the swarm size and \(T_{\mathrm{ep}}\) is the episode length. A higher \(\Phi_{\mathrm{ind}}\) indicates higher individual contribution and lower efficiency loss caused by interference, congestion or inefficient trail use.

\textbf{Transport economy.}
The transport economy \(E_{\mathrm{trans}}\) evaluates whether delivered resources are transported through spatially efficient trajectories:
\begin{equation}
    E_{\mathrm{trans}}
    =
    \frac{2N_{\mathrm{delivered}}d_{\mathrm{avg}}}{L_{\mathrm{actual}}},
    \qquad
    L_{\mathrm{actual}}
    =
    \sum_{i=1}^{N}
    \sum_{t=2}^{T_{\mathrm{ep}}}
    \left\|
    x_i(t)-x_i(t-1)
    \right\|_2 .
    \label{eq:transport_economy}
\end{equation}
Here, \(d_{\mathrm{avg}}\) is the average distance between the nest and the resource locations, and \(L_{\mathrm{actual}}\) is the cumulative travel distance of all robots. A higher \(E_{\mathrm{trans}}\) indicates more directed and economical transport.

\subsubsection{Formation-Reconfigurable Swarm Navigation}

For \textit{Formation-Reconfigurable Swarm Navigation}, we evaluate formation preservation and navigation efficiency.

\textbf{Formation Error.} 
To evaluate formation preservation independently of global translation and rotation, we use a rigid-motion-invariant formation distance~\cite{cai2019integrated}. Let \(p_i(t)\in\mathbb{R}^{2}\) be the position of robot \(i\) at step \(t\), and let \(\xi_i\in\mathbb{R}^{2}\) be its corresponding target formation point. The formation error is then defined as the minimum alignment error between the current and target formations after optimizing over planar rotations and translations:
\begin{equation}
    E_{\mathrm{form}}(t)
    =
    \min_{R\in SO(2),\,c\in\mathbb{R}^{2}}
    \left(
    \frac{1}{N}
    \sum_{i=1}^{N}
    \left\|
    p_i(t)-R\xi_i-c
    \right\|_2^2
    \right)^{1/2},
    \label{eq:formation_distance}
\end{equation}
where \(R\) is a planar rotation matrix and \(c\) is a translation vector. A lower \(E_{\mathrm{form}}(t)\) indicates better preservation of the target morphology at that specific time step.

\textbf{Mean Formation Error.} 
To quantify the overall formation maintenance performance over an entire evaluation run, we report the temporal mean formation error, defined as
\begin{equation}
    \bar{E}_{\mathrm{form}}
    =
    \frac{1}{T_{\mathrm{ep}}}
    \sum_{t=1}^{T_{\mathrm{ep}}}
    E_{\mathrm{form}}(t),
    \label{eq:mean_formation_error}
\end{equation}
where \(T_{\mathrm{ep}}\) represents the total simulation steps of the episode. This metric provides an episode-level evaluation of the swarm's morphological fidelity.

\textbf{Navigation time.}
Navigation efficiency is measured by the time required for the swarm centroid to reach the goal region:
\begin{equation}
    T_{\mathrm{nav}}
    =
    \min
    \left\{
    t\in[1,T_{\max}]
    :
    \left\|
    p_c(t)-p_{\mathrm{goal}}
    \right\|_2
    <
    d_{\mathrm{th}}
    \right\},
    \qquad
    p_c(t)=\frac{1}{N}\sum_{i=1}^{N}p_i(t).
    \label{eq:navigation_time}
\end{equation}
Here, \(p_c(t)\) is the swarm centroid, \(p_{\mathrm{goal}}\) is the target position, \(d_{\mathrm{th}}\) is the arrival threshold and \(T_{\max}\) is the maximum episode length. If the swarm does not reach the goal within \(T_{\max}\), we set \(T_{\mathrm{nav}}=T_{\max}\). A smaller \(T_{\mathrm{nav}}\) indicates faster goal-directed navigation.

\subsubsection{Role-Adaptive Swarm Search and Rescue}

For \textit{Role-Adaptive Swarm Search and Rescue}, we evaluate whether the swarm can establish an effective base--target relay chain and how quickly this connection is formed. Relay-chain algebraic connectivity is used as the topological criterion from which task success rate and time to connectivity are derived.

\textbf{Relay-chain connectivity criterion.}
We model the communication structure at time \(t\) as a weighted undirected graph \(\mathcal{G}(t)\) whose nodes include the base, the target and all robots. Edge weights are determined by distance-dependent communication quality. We first check whether the base and target belong to the same connected component under the communication threshold \(d_{\mathrm{comm}}\). If they are disconnected, the relay chain is considered invalid and we set
\[
    \lambda_2(t)=0.
\]
If they are connected, we extract the connected subgraph containing both the base and the target, construct its weighted graph Laplacian \(L_{\mathrm{sub}}(t)\), and define
\begin{equation}
    \lambda_2(t)
    =
    \lambda_2
    \left(
    L_{\mathrm{sub}}(t)
    \right),
    \label{eq:relay_algebraic_connectivity}
\end{equation}
where \(\lambda_2(\cdot)\) denotes the second-smallest eigenvalue of the graph Laplacian, also known as the algebraic connectivity or Fiedler value~\cite{fiedler1973algebraic,deAbreu2007Algebraic}. The smallest eigenvalue of an undirected graph Laplacian is always zero, whereas \(\lambda_2>0\) indicates graph connectivity. In this task, a positive \(\lambda_2(t)\) means that a valid base--target relay pathway has been established, and a larger value indicates a more robust relay topology.

\textbf{Task success rate.}
For trial \(m\), task success is defined by whether a valid base--target relay chain is established at least once within the episode:
\begin{equation}
    S_m
    =
    \mathbb{I}
    \left(
    \max_{t\in[1,T_{\max}]}
    \lambda_2^{(m)}(t)
    >
    0
    \right),
    \label{eq:sar_trial_success}
\end{equation}
where \(\lambda_2^{(m)}(t)\) is the relay-chain algebraic connectivity in trial \(m\). Over \(N_{\mathrm{trial}}\) independent trials, the success rate is
\begin{equation}
    \eta_{\mathrm{succ}}
    =
    \frac{1}{N_{\mathrm{trial}}}
    \sum_{m=1}^{N_{\mathrm{trial}}}
    S_m.
    \label{eq:sar_success_rate}
\end{equation}
A higher \(\eta_{\mathrm{succ}}\) indicates stronger robustness to random initial conditions, target locations and decentralized role allocation.

\textbf{Time to connectivity.}
The time to connectivity measures how quickly the swarm establishes a valid relay chain:
\begin{equation}
T_{\mathrm{TTC}}
=
\left\{
\begin{aligned}
&\min\{t\in[1,T_{\max}]:\lambda_2(t)>0\},
&& \text{if } \max_t\lambda_2(t)>0,\\
&T_{\max},
&& \text{otherwise}.
\end{aligned}
\right.
\label{eq:time_to_connectivity}
\end{equation}
If no base--target relay chain is established during the episode, the metric is assigned the maximum time penalty \(T_{\max}\). A smaller \(T_{\mathrm{TTC}}\) indicates faster relay-chain formation and better response efficiency.

\clearpage

\subsection{Baseline and ablation methods}
\label{sec:baseline_ablation_methods}

To examine whether the performance of PhySwarm comes from adaptive physical-parameter learning rather than from hand-crafted heuristics or a favorable static parameter setting, we compared the PhySwarm with one rule-based baseline and two fixed-parameter ablation variants in each benchmark task. All methods used the same robot platform, arena, sensing range, actuator limits, low-level safety interface and task-specific sensory cues. The comparison therefore isolates the effect of adaptive modulation of the physical parameters \(P(t)=\{\omega(t),D(t),\lambda(t)\}\).

\textbf{PhySwarm.}
The Neural-Physics Controller (NPC) predicts the advection weights, diffusion coefficients and reaction rates online from local observations and temporal memory. These parameters are projected onto the physically feasible parameter manifold and executed through the macro--micro coupling.

\textbf{Finite-state-machine baseline.}
The finite-state-machine baseline represents a conventional hand-crafted reactive controller. It uses the same task-relevant cues as PhySwarm, such as food detection, nest direction, corridor width, formation deformation, target awareness and communication conditions. However, these cues are mapped to robot commands through manually designed priority rules rather than through learned ADR parameters. 

\textbf{Fixed-parameter ablations.}
The two ablation variants retain the PhySwarm execution structure but disable adaptive parameter learning. Instead, the field gains, diffusion coefficient and reaction rates are fixed throughout the episode. Ablation A represents a conservative or single-mechanism-dominant regime, whereas Ablation B represents the opposite degenerate regime, such as transition-dominant, propulsion-dominant or relay/diffusion-dominant behaviour. These variants test whether each task can be solved by a static ADR parameter configuration.

The rule-based baselines are summarized in Supplementary Table~\ref{tab:fsm_baseline_settings}, and the fixed-parameter ablation settings are summarized in Supplementary Table~\ref{tab:ablation_parameter_settings}. For ablation variants, the listed field gains are fixed implementation gains; when a normalized advection-weight simplex is used, these gains are normalized before forming \(\omega\). The exact numerical values of the fixed parameters used in each ablation are also provided in the implementation configuration files of our project repository and are kept unchanged across repeated trials.

\begin{sidewaystable}[p]
\centering
\caption{
\textbf{Finite-state-machine baselines used for comparison.}
Each baseline uses the same task cues available to PhySwarm but maps them to actions through hand-crafted priority rules rather than learned physical parameters.
}
\label{tab:fsm_baseline_settings}
\scriptsize
\renewcommand{\arraystretch}{1.25}
\begin{tabular}{
p{0.20\textheight}
p{0.38\textheight}
p{0.36\textheight}
}
\toprule
\textbf{Scenario} & \textbf{Decision logic} & \textbf{Control constants} \\
\midrule

\textit{Trail-Guided Swarm Foraging} &
Priority-based reactive control. For robot \(i\), wall avoidance is activated if \(d_{\mathrm{min,w},i}<d_{\mathrm{wall}}\); inter-robot avoidance is activated if \(d_{\mathrm{min,a},i}<d_{\mathrm{robot}}\). If the robot is carrying a resource, it moves towards the nest. Otherwise, it follows food cues if available, then trail or information cues if available, and performs sinusoidal exploratory motion otherwise. The steering angle is computed from the corresponding local heading error \(\theta_{\cdot,i}\). &
\(d_{\mathrm{wall}}=0.08\,\mathrm{m}\), \(d_{\mathrm{robot}}=0.12\,\mathrm{m}\);
\(V_{\mathrm{avoid}}=0.02\,\mathrm{m\,s^{-1}}\),
\(V_{\mathrm{return}}=0.12\,\mathrm{m\,s^{-1}}\),
\(V_{\mathrm{search}}=0.08\,\mathrm{m\,s^{-1}}\),
\(V_{\mathrm{food}}=0.10\,\mathrm{m\,s^{-1}}\);
\(\epsilon=10^{-3}\), \(A=0.5\), \(\gamma=0.1\), \(\phi_i=i\);
\(K_{\mathrm{wall}}=4.0\),
\(K_{\mathrm{robot}}=5.0\),
\(K_{\mathrm{nest}}=2.5\),
\(K_{\mathrm{food}}=3.0\),
\(K_{\mathrm{info}}=2.0\). \\

\midrule

\textit{Formation-Reconfigurable Swarm Navigation} &
Reactive vector-field controller. If the swarm is close to the goal, the command combines formation regulation, goal attraction and obstacle avoidance. If the local corridor width is small or the deformation measure is large, the robot prioritizes forward drive and strengthened avoidance. Otherwise, it follows formation regulation with obstacle avoidance. The body-frame velocity is composed from \(v_{\mathrm{form}}\), \(v_{\mathrm{avoid}}\), \(v_{\mathrm{attract}}\) and \(v_{\mathrm{drive}}\). &
Goal-neighbourhood threshold \(d_{\mathrm{g}}<0.3\,\mathrm{m}\);
corridor/deformation trigger \(w_i<0.85\) or \(q_{\mathrm{total}}>0.8\);
target formation radius \(R_0=0.2\,\mathrm{m}\);
\(K_{\mathrm{form}}=1.0\),
\(K_{\mathrm{rep}}=0.05\),
\(K_{\mathrm{goal}}=0.1\);
\(V_{\mathrm{cruise}}=0.08\,\mathrm{m\,s^{-1}}\). \\

\midrule

\textit{Role-Adaptive Swarm Search and Rescue} &
Hierarchical role-assignment and reactive control. Let \(\mathcal{K}_t\) be the set of robots aware of the target and \(M_i(t)\) the number of target-aware robots closer to the target than robot \(i\). A robot is assigned as a searcher if \(i\notin\mathcal{K}_t\), a responder if \(i\in\mathcal{K}_t\) and \(M_i(t)<3\), and a relay if \(i\in\mathcal{K}_t\) and \(M_i(t)\geq3\). Low-level priority is wall avoidance, inter-robot avoidance, search motion, target response and relay-line following. Relay robots use a larger safety margin to encourage chain stretching. &
Role-dependent safety margins:
\(d_{\mathrm{robot}}=0.50\,\mathrm{m}\) for relay robots and
\(d_{\mathrm{robot}}=0.12\,\mathrm{m}\) for searchers/responders;
\(d_{\mathrm{wall}}=0.08\,\mathrm{m}\);
\(V_{\mathrm{avoid}}=0.02\,\mathrm{m\,s^{-1}}\),
\(V_{\mathrm{search}}=0.08\,\mathrm{m\,s^{-1}}\),
\(V_{\mathrm{rush}}=0.12\,\mathrm{m\,s^{-1}}\),
\(V_{\mathrm{relay}}=0.06\,\mathrm{m\,s^{-1}}\);
\(K_{\mathrm{wall}}=4.0\),
\(K_{\mathrm{robot}}=5.0\),
\(K_{\mathrm{task}}=3.0\);
\(A=0.5\), \(\gamma=0.1\), \(\phi_i=i\). \\

\bottomrule
\end{tabular}
\end{sidewaystable}

\begin{sidewaystable}[p]
\centering
\caption{
\textbf{Fixed-parameter ablation settings.}
Ablation A and Ablation B keep the PhySwarm execution structure but disable adaptive NPC parameter learning. The listed field gains are fixed throughout the episode; when the implementation uses normalized advection weights, the gains are normalized before constructing \(\omega\). \(D\) denotes the diffusion gain used in Micro-EDM density-gradient compensation.
}
\label{tab:ablation_parameter_settings}
\scriptsize
\renewcommand{\arraystretch}{1.25}
\begin{tabular}{
p{0.15\textheight}
p{0.40\textheight}
p{0.36\textheight}
}
\toprule
\textbf{Scenario} & \textbf{Ablation A} & \textbf{Ablation B} \\
\midrule

\textit{Trail-Guided Swarm Foraging} &
\textbf{Weak-static setting.}
The physical parameters are fixed to a low-response regime with weak food attraction, weak nest attraction, weak exploration and weak information guidance:
\[
\omega_{\mathrm{food}}=0.08,\quad
\omega_{\mathrm{nest}}=0.05,\quad
\omega_{\mathrm{exp}}=0.05,\quad
\omega_{\mathrm{info}}=0.03,
\]
\[
D=0.05,\quad
\lambda_{\mathrm{pick}}=0.6,\quad
\lambda_{\mathrm{drop}}=0.6.
\]
&
\textbf{Transition-dominant homing-biased setting.}
The physical parameters are fixed to a regime with stronger transition rates and relatively stronger nest attraction but weaker food and exploration guidance:
\[
\omega_{\mathrm{food}}=0.04,\quad
\omega_{\mathrm{nest}}=0.07,\quad
\omega_{\mathrm{exp}}=0.03,\quad
\omega_{\mathrm{info}}=0.03,
\]
\[
D=0.03,\quad
\lambda_{\mathrm{pick}}=0.8,\quad
\lambda_{\mathrm{drop}}=0.8.
\]
\\

\midrule

\textit{Formation-Reconfigurable Swarm Navigation} &
\textbf{Shape-dominant setting.}
The physical parameters are fixed to prioritize morphology preservation over forward progression:
\[
\omega_{\mathrm{flow}}=0.05,\quad
\omega_{\mathrm{shape}}=0.08,\quad
D=0.05,\quad
\lambda=0.0.
\]
&
\textbf{Propulsion-dominant setting.}
The physical parameters are fixed to strengthen forward flow while weakening morphology constraints:
\[
\omega_{\mathrm{flow}}=0.09,\quad
\omega_{\mathrm{shape}}=0.04,\quad
D=0.05,\quad
\lambda=1.0.
\]
\\

\midrule

\textit{Role-Adaptive Swarm Search and Rescue} &
\textbf{Response-dominant setting.}
The physical parameters are fixed to bias robots towards aggressive target response with weak relay organization and limited diffusion:
\[
\omega_{\mathrm{target}}=0.12,\quad
\omega_{\mathrm{coup}}=0.01,\quad
\omega_{\mathrm{exp}}=0.00,
\]
\[
D=0.01,\quad
\lambda_{\mathrm{anchor}}=0.1,\quad
\lambda_{\mathrm{release}}=0.9.
\]
&
\textbf{Relay/diffusion-dominant setting.}
The physical parameters are fixed to favour relay-axis organization and spatial spreading while weakening target-directed response:
\[
\omega_{\mathrm{target}}=0.02,\quad
\omega_{\mathrm{coup}}=0.08,\quad
\omega_{\mathrm{exp}}=0.03,
\]
\[
D=0.15,\quad
\lambda_{\mathrm{anchor}}=0.9,\quad
\lambda_{\mathrm{release}}=0.1.
\]
\\

\bottomrule
\end{tabular}
\end{sidewaystable}

\clearpage

\section{Supplementary Note}
\label{supp:emergent_collective_behavior_cn}

\subsection{Emergent Behavior v.s. Collective Behavior}
\label{sec:emergent_vs_collective}

In swarm robotics, multi-agent systems and complex-systems research, \textit{collective behavior} and \textit{emergent behavior} are closely related, but they are not identical concepts. Collective behavior is a broader descriptive term that generally refers to group-level behaviors exhibited by multiple robots, without specifying how these behaviors are generated. Such behaviors may arise from centralized control, explicitly programmed rules, predefined coordination strategies or decentralized self-organization. By contrast, emergent behavior emphasizes the generative mechanism of the macroscopic pattern: the ordered behavior observed at the swarm level is not directly encoded in each individual robot, nor explicitly prescribed by a central controller, but arises over space and time from local perception, neighbourhood interactions, decentralized decision-making and robot--environment feedback. Therefore, this work does not focus on general multi-robot coordination or task execution, but on how multi-stage macroscopic behaviors in robot swarms are generated, transformed and stabilized through the coupling of local interactions, physical constraints and environmental feedback.

Based on this distinction, we define the target of this work as \textit{multi-stage emergent behaviors}. Such behaviors commonly appear in representative swarm-robot tasks. For example, in swarm foraging, the swarm may sequentially undergo dispersed exploration, resource discovery, homing, trail formation and renewed exploration. In collective navigation, the swarm may form a circular formation in open space, adaptively transform into a line-like queue in a narrow corridor and recover its original formation after traversal. In search and rescue, the swarm may first disperse for search, then aggregate around a detected target and further differentiate into functional subgroups such as searchers, responders and communication relays. The key challenge in these processes is not merely the completion of several task objectives, but the continuous transition among different macroscopic behavioral modes. Therefore, a single static formation target, a fixed control law or a concatenation of task-specific rules is insufficient to capture the underlying mechanism. This motivates the central scientific question of this work: can multi-stage macroscopic behaviors in robot swarms be represented as behavioral-state transitions on a learnable physical field, rather than as a simple combination of handcrafted task rules or black-box policies? Addressing this question requires a modeling framework that can simultaneously describe how local interactions induce global behavior, how environmental constraints reshape swarm dynamics, how one behavioral mode transforms into another and how different types of emergent behaviors can be expressed within a unified mathematical form.

To maintain conceptual clarity, we use the following working definition:

\begin{quote}
\textit{\textbf{Emergent behavior} refers to a macroscopic behavioral pattern, or a transition process among such patterns, generated by local robot perception, decentralized interactions and feedback between robots and the environment. It is not a behavior explicitly specified at the level of individual robots or directly prescribed by a central controller. Such behavior typically involves a non-trivial micro--macro relationship, self-organization and dynamic adaptability, and may manifest as continuous transitions among multiple behavioral modes.}
\end{quote}

Correspondingly, we define collective behavior as follows:

\begin{quote}
\textit{\textbf{Collective behavior} refers to any group-level behavior jointly exhibited by multiple robots, regardless of whether it is generated by centralized control, explicit programming, distributed coordination or self-organized emergence.}
\end{quote}

Under these definitions, all emergent behaviors studied in this work are collective behaviors, but not all collective behaviors fall within the scope of emergent-behavior modeling considered here. In other words, this work focuses on macroscopic swarm behaviors that are induced by local interactions and environmental feedback and that exhibit multi-stage transitions during task execution, rather than on all forms of multi-robot cooperative control or collective task execution.

\end{document}